\documentclass{article}
\pdfoutput=1
\usepackage[utf8]{inputenc}
\usepackage[T1]{fontenc}
\usepackage{amsmath,amsthm, amssymb, bbm,color,mathtools}
\usepackage[verbose=true,
letterpaper,
textheight=8.4in,
textwidth=6.1in,
margin=1in,
headheight=12pt,
headsep=25pt,
footskip=30pt]{geometry}

\usepackage[toc,page,header]{appendix}
\usepackage{titletoc}
\usepackage[usenames,dvipsnames]{xcolor} 
\usepackage{hyperref}
\usepackage{ss}

\usepackage[round]{natbib}

\titlecontents{section}[3em]{\vspace{-1pt}}%
    {\bfseries\contentslabel{2em}}%
    {}%
    {\titlerule*[0.5pc]{.}\contentspage}%
    
\titlecontents{subsection}[5em]{\vspace{-1pt}}%
    {\contentslabel{2em}}%
    {}%
    {\titlerule*[0.5pc]{.}\contentspage}%

\titlecontents{subsubsection}[5em]{\vspace{-1pt}}%
    {\footnotesize\contentslabel{3em}}%
    {}%
    {\titlerule*[0.5pc]{.}\contentspage}%

\numberwithin{equation}{section}

\usepackage{url}            
\usepackage{booktabs}       
\usepackage{amsfonts}       
\usepackage{nicefrac}       
\usepackage{microtype}      
\usepackage{csquotes}

\usepackage{subcaption}
\usepackage{graphicx}
\usepackage{enumitem}
\usepackage{amsmath} 

\DeclareMathOperator*{\argmin}{arg\,min}
\newcommand{\TF}{{\sf TF}} 
\newcommand{\softmax}{{\sf softmax}} 
 
\newcommand{\relu}{{\sf relu}} 
\newcommand{\dist}{{\sf dist}} 

\DeclareMathOperator{\E}{\mathbb{E}}

\newcommand{\N}{\mathcal{N}}
\renewcommand{\P}{\mathcal{P}}

\newcommand*{\twocase}[4]{\left\{\begin{array}{ll}
        #1 & \text{for } #2\\
        #3 & \text{for } #4
        \end{array}\right.}

\newcommand{\R}{\mathbb{R}}

\renewcommand{\S}{{\mathcal S}}

\newtheorem{theorem}{Theorem} 
\newtheorem{definition}{Definition}
\newtheorem{lemma}{Lemma}
\newtheorem{corollary}[theorem]{Corollary}
\newtheorem{proposition}{Proposition}
\newtheorem{remark}{Remark}

\newtheorem{assumption}{Assumption}

\newtheorem{example}{Example}

\usepackage{mdframed}
\usepackage{algorithmicx}

\usepackage{tikz}
\usetikzlibrary{positioning}

\usepackage{float}
\usepackage{wrapfig}
\usepackage[most]{tcolorbox}

\renewcommand{\paragraph}[1]{\ \newline\noindent\textbf{#1}\quad}

\DeclareMathOperator{\Dist}{Dist}

\newcommand{\att}{\mathrm{Attn}}

\newcommand{\wstar}{w_\star}

\newcommand{\tz}[1]{{z^{(#1)}}}
\newcommand{\tx}[1]{{x^{(#1)}}}

\newcommand{\ty}[1]{y^{(#1)}}
\newcommand{\tY}[1]{Y^{(#1)}}

\newcommand{\tS}{\tilde{S}}
\newcommand{\tT}{\tilde{T}}

\newcommand*\bmat[1]{\begin{bmatrix} #1 \end{bmatrix}}

\DeclareMathOperator{\tr}{Tr}

\newcommand*\lin[1]{\left\langle #1 \right\rangle}

\newcommand*\lrb[1]{\left[ #1 \right]}
\newcommand*\lrn[1]{\left\| #1 \right\|}
\newcommand*\lrp[1]{\left( #1 \right)}
\newcommand*\lrbb[1]{\left\{ #1 \right\}}

\renewcommand\th{{\tilde{h}}}

\newcommand\jA{{\mathcal{J}_{\th}}}
\newcommand\K{\mathcal{K}}

\newcommand\Kmat{\mathbb{K}}
\newcommand\HH{\mathbb{H}}

\newcommand\X{\mathcal{X}}

\newcommand{\diag}{\mathrm{diag}}

\newcommand\numberthis{\addtocounter{equation}{1}\tag{\theequation}}

\newcommand*\at[2]{\left.#1\right|_{#2}}
\newcommand{\US}{U_\Sigma}

\makeatletter
\newcommand{\ostar}{\mathbin{\mathpalette\make@circled\star}}
\newcommand{\make@circled}[2]{%
\ooalign{$\m@th#1\smallbigcirc{#1}$\cr\hidewidth$\m@th#1#2$\hidewidth\cr}%
}
\newcommand{\smallbigcirc}[1]{%
\vcenter{\hbox{\scalebox{0.77778}{$\m@th#1\bigcirc$}}}%
}
\makeatother

\mathtoolsset{showonlyrefs}

\title{Transformers Implement Functional Gradient Descent to Learn Non-Linear Functions In Context}

\begin{document}

\linespread{1.03}              

\author{\name{Xiang Cheng$^\ast$} \email{chengx@mit.edu}\\
\name{Yuxin Chen$^\dagger$}
\email{yxxchen@ucdavis.edu}\\
\name{Suvrit~Sra$^{\ast,\ddagger}$} \email{suvrit@mit.edu}\\[2mm]
\addr{$^\ast$Massachusetts Institute of Technology, USA}\\
\addr{$^\dagger$University of Calfornia, Davis, USA}\\
\addr{$^\ddagger$Technical University of Munich, Germany}
}

\maketitle

\vskip12pt
\begin{abstract}
\noindent Many neural network architectures are known to be Turing Complete, and can thus, in principle implement arbitrary algorithms. However, Transformers are unique in that they can implement gradient-based learning algorithms \emph{under simple parameter configurations}. 
This paper provides theoretical and empirical evidence that (non-linear) Transformers naturally learn to implement gradient descent \emph{in function space}, which in turn enable them to learn non-linear functions in context. Our results apply to a broad class of combinations of non-linear architectures and non-linear in-context learning tasks. Additionally, we show that the optimal choice of non-linear activation depends in a natural way on the class of functions that need to be learned.
\end{abstract}

\section{Introduction}

Transformers \citep{vaswani2017attention} have been observed to produce the correct output based on contextual demonstrations provided in the prompt alone, a phenomenon commonly known as in-context learning (ICL)~\citep{brown2020language}. Understanding ICL and the mechanism underlying may hold the key to explaining the success of the Transformer architecture, and has therefore attracted great attention. 

\vskip2pt
A promising conjecture is that Transformers learn in-context by implementing algorithms in their forward pass \citep{akyurek2022learning,von2022transformers,ahn2023transformers,zhang2023trained,mahankali2023one,von2023uncovering,lin2023transformers,bai2023transformers}. A subset of this work focuses on ICL for  \emph{linear functions} 
using \emph{linear Transformers} (i.e., the attention module contains no nonlinear activations). In this setting, there exists a \emph{simple parameter configuration} under which the Transformer implements gradient descent. Subsequently, \citep{ahn2023transformers,mahankali2023one} were able to verify the local and global optimality of similar parameter configurations, and \citet{zhang2023trained} showed convergence to this parameter configuration when training the Transformer with gradient descent. 

\vskip2pt
The above works provide a convincing explanation of how linear Transformers learn linear functions in context; their theoretical conclusions are also well supported by experiments~\citep{von2022transformers, ahn2023transformers}. An important aspect of this setting is that linear Transformers are \emph{very well suited} to learning linear functions---by simply setting the Query and Key matrices to the identity matrix, a single linear attention can implement one step of gradient descent for the least squares loss. 

\vskip2pt
But in real Transformers, emph{non-linear activations such as softmax are very important}; moreover, the training data are more likely generated by \emph{complicated non-linear functions}. It is unclear whether there exists a similarly elegant construction for non-linear Transformers that explains how they could learn non-linear functions in context. These observations raise two central motivating questions:
\begin{tcolorbox}[enhanced,title=,
                    frame hidden,
                    colback=Goldenrod!20,
                    breakable,
                    left=0pt,
                    right=0pt,
                    top=3pt,
                    bottom=3pt,
                    ]
\begin{enumerate}[label=(Q\arabic*),leftmargin=2cm]
\setlength{\itemsep}{0pt}
    \item \emph{What learning algorithms are implemented by Transformers with \textbf{non-linear activations}?}
    \item \emph{Can Transformers learn \textbf{non-linear functions} of data in context?}
\end{enumerate}
\end{tcolorbox}
\noindent This paper aims to answer both questions, so as to shed light on the inner workings of Transformers, and in turn advance our understanding on what makes Transformers powerful learners. 

\vskip2pt
More specifically, we simultaneously consider both \textbf{\emph{non-linear architectures}}---Attention modules with arbitrary non-linear activations $\th$ (e.g., softmax or ReLU) and \textbf{\emph{non-linear data}}---where labels are sampled from a non-linear process (e.g., a Gaussian Process, or certain more general processes, to be clarified later) conditioned on the covariates. Surprisingly, we show that the answers to questions~(Q1) and~(Q2) are \emph{deeply intertwined}: there exists \textbf{\emph{a simple parameter configuration that makes Transformers implement gradient descent in function space}}; moreover, we show that this functional gradient descent converges to the \textbf{\emph{Bayes optimal predictor if the non-linearity of the attention module matches the underlying data distribution.}} Beyond our construction, we also provide \textbf{\emph{theoretical and empirical evidence}} that Transformers do indeed learn to implement functional gradient descent via training. Our analysis applies to a broad range of functions (such as labels generated from two layer ReLU networks, see Ex.~\ref{ex:2_layer_relu}) and common architectures (such as ReLU and Softmax Transformers, see Ex.~\ref{ex:relu}, \ref{ex:softmax}).

\subsection{Main Contributions}
The main contributions of this work are as follows:
\begin{enumerate}
\setlength{\itemsep}{1pt}
\item In Proposition \ref{p:rkhs_descent_transformer_construction}, we show that when the non-linearity in the Attention module matches a kernel $\K$, then \textbf{\emph{Transformers can implement gradient descent in function space wrt the Reproducing Kernel Hilbert Space (RKHS) metric induced by $\K$}}. In Sections \ref{ss:construction_linear} and \ref{ss:construction_softmax_rbf}, we discuss the connection between the construction in Proposition \ref{p:rkhs_descent_transformer_construction} and several common Transformer variants. Our result generalizes the least-squares gradient descent construction from \cite{von2022transformers}.
    \item In Proposition \ref{p:matching_h_k_optimality}, we consider a general setting when the data labels $\ty{i}$ are generated from a Kernel Gaussian Process. We show that when the non-linear module $\th$ matches the generating kernel $\K$, \textbf{\emph{the functional gradient descent construction converges to the Bayes optimal predictor as the number of layers increases.}} In Section \ref{ss:experiments_for_optimality}, we \emph{verify experimentally} that the highest accuracy is indeed achieved when the \textbf{\emph{non-linear module matches the generating kernel}}.

    \item In Proposition \ref{c:multihead}, we present a generalization of Propositions \ref{p:rkhs_descent_transformer_construction} and \ref{p:matching_h_k_optimality} to \textbf{multi-head} attention. A multi-head Transformer with different activation $\th$ per-head can implement the Bayes-optimal functional gradient descent algorithm for any RKHS that is obtainable by composition of the kernels of each individual $\th$.

    \item We analyze the loss landscape of a Transformer on non-linear data. In Theorem \ref{t:informal_master_sparse}, we characterize certain stationary points of the in-context loss under a sparsity constraint on the value matrix. When $\th$ coincides with a kernel,  \textbf{\emph{this stationary point is exactly the functional gradient descent construction of Proposition \ref{p:rkhs_descent_transformer_construction}}}. We verify empirically that this stationary point is consistently learned during training. 

    \item In Theorem \ref{t:informal_master_full}, we characterize stationary points of the in-context loss without the sparsity constraint. Our proposed stationary point implements an algorithm that interleaves steps of covariate transformation with functional gradient descent. Once again, we verify empirically that the stationary point is consistently learned during training.

    \item Less importantly, but possibly of independent interest, our experiments in Section \ref{ss:experiments_for_optimality} identify a simple scenario where ReLU Transformers appears to out-perform softmax Transformers (and vice-versa).
\end{enumerate}

\noindent The table below summarizes the main theoretical results of this paper, along with their key assumptions. We emphasize that \textbf{Theorems \ref{t:informal_master_sparse} and \ref{t:informal_master_full} apply to the commonly used softmax and ReLU attentions.}
\begin{center}
\vspace{-6pt}
\colorbox{gray!20}{
\begin{tabular}{l|l l l l | l}
Results  & $\tx{i}$& $\ty{i}$ &  $\th$ & \begin{tabular}{@{}l@{}}Transformer\\ 
Parameters\end{tabular} & Basic Description\\
\hline
\hline
Prop. \ref{p:rkhs_descent_transformer_construction} & None & None & $\th$ a kernel & None & \begin{tabular}{@{}l@{}}Transformer can implement functional\\ GD in RKHS induced by kernel $\th$.\end{tabular}\\
\hline
Prop. \ref{p:matching_h_k_optimality} & None & \begin{tabular}{@{}l@{}}$\K$-GP~\eqref{d:k_gaussian_process}\end{tabular} & $\th$ matches $\K$   & None & \begin{tabular}{@{}l@{}}Transformer prediction can be Bayes\\ opt.~if $\th$ matches $\K$, with suff.~layers.\end{tabular}\\
\hline
Prop. \ref{c:multihead} & None & \begin{tabular}{@{}l@{}}$\K$-GP~\eqref{d:k_gaussian_process}\end{tabular} & \begin{tabular}{@{}l@{}}$\K$ composed \\from $\th$'s \end{tabular} & None & \begin{tabular}{@{}l@{}}Multi-head Transformer can implement \\functional GD on RKHS for composite \\kernels, which is Bayes optimal.\end{tabular}\\
\hline
Thm. \ref{t:informal_master_sparse} & Asm. \ref{ass:x_distribution} & Asm. \ref{ass:y_distribution} & Asm. \ref{ass:th} & 
\begin{tabular}{@{}l@{}} Asm. \ref{ass:full_attention}\\ $\&$ $A_\ell=0$ \end{tabular} & 
\begin{tabular}{@{}l@{}}Functional GD is stationary point of\\ IC loss under $A_\ell=0$ constraint. \end{tabular}
\\
\hline
Thm. \ref{t:informal_master_full} & Asm. \ref{ass:x_distribution} & Asm. \ref{ass:y_distribution} & Asm. \ref{ass:th} & 
Asm. \ref{ass:full_attention}
& \begin{tabular}{@{}l@{}}Characterizing stationary point\\ of in-context loss (unconstrained). \end{tabular}
\\
\end{tabular}
}
\end{center}
\subsection{Related Work}

\citet{garg2022can} show experimentally that Transformers can learn \emph{simple functions} in context, including linear functions, decision trees, and two layer neural networks. \citet{akyurek2022learning,dai2022can} propose that Transformers learn in-context by implementing learning algorithms. Building upon \citet{akyurek2022learning}, \citet{lin2023transformers} propose more efficient constructions for a broader range of learning algorithms. \citet{bai2023transformers} apply a similar technique to study the in-context reinforcement learning problem. Independent of the ICL motivation, numerous other authors have also studied the algorithmic power of transformers \cite{perez2021attention,wei2022statistically,giannou2023looped,olsson2022context}. 

\vskip2pt
Many of the above papers propose some form of \emph{construction} (i.e., a specific parameter configuration), under which the Transformer implements the desired algorithm. It is however often unclear if the Transformer actually learns these constructions during training. Motivated by the question of \emph{``what do Transformers actually learn,''} a line of recent work turned their attention to linear Transformers \cite{schlag2021linear,von2022transformers}. 

\vskip2pt
In~\citep{von2022transformers}, the authors devise a \emph{simple weight construction} for the linear Transformer, which can be shown to implement gradient descent (as well as a more sophisticated algorithm known as GD++). Subsequently, \cite{ahn2023transformers,zhang2023trained,mahankali2023one} show that for 1-layer Transformers, there exists a global minimum of the in-context loss that closely resembles the construction in \cite{von2022transformers}.  \cite{zhang2023trained} further show that training a 1-layer Transformer with gradient descent converges in polynomial time to the proposed global minimum. For multi-layer Transformers, \cite{ahn2023transformers} show the local optimality of preconditioned GD and preconditioned GD++, under different parameter sparsity assumptions. We note that Assumptions \ref{ass:QK_sparsity} and \ref{ass:full_attention} in this paper closely parallel the parameter sparsity assumptions in \cite{ahn2023transformers}. Furthermore, Theorems \ref{t:informal_master_sparse} and \ref{t:informal_master_full} can be viewed as generalizations of Theorems 3 and 4 of \cite{ahn2023transformers} to non-linear architectures and functions. Finally, Theorem 5 of \cite{ahn2023transformers} establishes the global optimality of gradient descent for \emph{one-layer ReLU-activated Transformers}, which was shown in \cite{wortsman2023replacing} to perform comparably to softmax transformers on certain tasks. More recently, \cite{von2023uncovering} studied the ability of linear Transformers to perform auto-regressive next-token prediction for sequential data. \cite{wu2023many} study the statistical complexity of ICL with a 1-layer linear Transformer for linear regression. \cite{huang2023context} uses a similar framework to study 1-layer softmax-activated Transformers, when the covariates are all orthonormal. 

\vskip2pt
Distinct from the above, another relevant line of work views the attention module as a kernel operation, and proposes alternatives to standard attention based on various kernels \citep{tsai2019transformer,choromanski2020rethinking,ali2021xcit,nguyen2022improving,nguyen2022fourierformer,chi2022kerple}. In \citep{wright2021transformers,chen2023primal}, authors consider the connection between Transformers and \emph{asymmetric kernels}.

\section{Setup: In-context learning with non-linear Transformers}
\label{s:define_icl}

\paragraph{Input Data for In-Context Learning}\\
We begin by defining the in-context learning problem. We are given $n$ demonstrations $\tz{i} := (\tx{i}, \ty{i})$, for $i=1...n$. $\tx{i} \in \R^d$ are covariates and $\ty{i}\in \R$ are scalar labels. We are also given a query $\tx{n+1} \in \R^d$, an the goal is to predict its label $\ty{n+1}$, which is unobserved. In general, $X := [\tx{1}\ldots \tx{n+1}] \in \R^{d\times (n+1)}$ have joint distribution $\P_X$. We assume that $Y = [\ty{1}...\ty{n+1}] \in \R^{1\times n+1}$ have joint distribution $\P_{Y|X}$ conditional on $X$. An important example of $\P_{Y|X}$ is when $\ty{i} = \phi(\tx{i})$ for some unknown function $\phi$. For instance, $\phi(x) = \lin{\theta, x}$ gives rise to linear regression. We will discuss specific choices of $\P_{X}$ and $\P_{Y|X}$ in Sections \ref{ss:optimality} and \ref{ss:distributional_assumptions}. Thus the input of the in-context learning problem is given by
\begin{align}
\label{d:Z_0}
Z_0 = \begin{bmatrix}
\tz{1} \ \tz{2} \ \cdots \ \tz{n}  \ \tz{n+1}
\end{bmatrix} = \begin{bmatrix}
\tx{1} & \tx{2} & \cdots & \tx{n} &\tx{n+1} \\ 
\ty{1} & \ty{2} & \cdots &\ty{n}& 0
\end{bmatrix} \in \R^{(d+1) \times (n+1)}.
\end{align}

\paragraph{Transsformers with general non-linear attention}\\
We define the \emph{generalized attention module} as 
\begin{align*}
    \numberthis \label{d:att_th}
    \att_{V,B,C}^{\th}(Z) := V Z M \th\lrp{B X, C X},
\end{align*}
$V\in \R^{(d+1)\times(d+1)},B\in \R^{d\times d},C\in \R^{d\times d}$ are the value, query, and key matrices respectively, and they parameterize the attention module. $M:= \begin{bmatrix}
    I_{n\times n}&0\\0&0
\end{bmatrix}$ is a mask matrix, and $\th: \R^{d\times(n+1)}\times \R^{d\times(n+1)}\to\R^{(n+1)\times(n+1)}$ denotes a matrix-valued function. The matrix $X$ is shorthand for $[\tx{1}\ldots \tx{n+1}]\in \R^{d\times(n+1)}$, the first $d$ rows of $Z$. 

\begin{remark}
The Attention definition in \eqref{d:att_th} differs from standard attention, in that $X$ should be replaced by $Z$. We discuss in Section \ref{ss:th_examples} how the common attention modules maps to \eqref{d:att_th} under a sparsity assumption on the Query, Key matrices (Assumption \ref{ass:QK_sparsity}).
\end{remark}

We construct a $k$-layer Transformer by stacking $k$ layers of the attention module (with residual). To be precise, let $Z_\ell$ denote the output of the $(\ell-1)^{th}$ layer of the Transformer. Then $Z_{\ell+1} := Z_\ell + \att^{\th}_{V_\ell, B_\ell, C_\ell} \lrp{Z_\ell}$, or equivalently:
\begin{align*}
    \numberthis \label{e:dynamics_Z}
    Z_{\ell+1} = Z_\ell + V_\ell Z_\ell M \th\lrp{B_\ell X_\ell, C_\ell X_\ell }
\end{align*}
where $V_\ell,B_\ell,C_\ell$ are the value, query and key matrices of the attention module at layer $\ell$. $\th: \R^{d\times(n+1)} \times \R^{d\times(n+1)}\to\R^{(n+1)\times(n+1)}$ denotes a non-linear activation function. 

Consider a $k$ layer Transformer. For the rest of the paper, we let $V:=\lrbb{V_\ell}_{\ell=0\ldots k}$, $B:=\lrbb{B_\ell}_{\ell=0\ldots k}$, $C:=\lrbb{C_\ell}_{\ell=0\ldots k}$ denote collections of the attention parameters across layers. For $\ell = 0\ldots k+1$, define 
\begin{align*}
    \numberthis \label{e:transformer_prediction}
    \TF_\ell(\textcolor{blue}{x}; (V,B,C)\vert \tz{1}\ldots \tz{n}):=\lrb{Z_{\ell}}_{(d+1),(n+1)},
\end{align*}
where $Z_i$ evolves as \eqref{e:dynamics_Z}, initialized at
{$Z_0 = 
{\lrb{\begin{smallmatrix}
\tx{1} & \tx{2} & \cdots & \tx{n} &\color{blue}{x} \\ 
\ty{1} & \ty{2} & \cdots &\ty{n}& 0
\end{smallmatrix}}}$}. We interpret $\TF_\ell(x; (V,B,C)\vert \tz{1}\ldots \tz{n})$ as \emph{``The predictor for $-\ty{n+1}$ at layer $\ell$, given $\tx{n+1} = x$, conditioned on demonstrations $\tz{1}\ldots \tz{n}$, parameterized by weight matrices $V,B,C$''}. This definition is consistent with the setup in \cite{von2022transformers,ahn2023transformers}.

\paragraph{The In-Context Loss}\\
Given the input $Z_0$ and the Transformer parameterized by $V,B,C$, we define the in-context loss as
\begin{align} \label{d:ICL_loss}
f\left(V,B,C\right) 
=& \E_{Z_0, \ty{n+1}} \Bigl[ \left( \TF_{k+1}(\tx{n+1}; (V,B,C)\vert \tz{1}\ldots \tz{n}) + \ty{n+1}  \right)^2\Bigr] \\
:=& \E_{Z_0, \ty{n+1}} \Bigl[ \left( \lrb{Z_{k+1}}_{(d+1),(n+1)} + \ty{n+1}  \right)^2\Bigr].
\end{align}

\subsection{Examples of Attention Modules}
\label{ss:th_examples}
To motivate definition~\eqref{d:att_th}, we show in Examples \ref{ex:linear}, \ref{ex:relu}, \ref{ex:softmax} below how the most common variants of the attention module can be realized via specific choices of non-linearity $\th$, assuming the following sparsity constraints: 
\begin{assumption}[$QK$ last column and row sparsity]
    \label{ass:QK_sparsity}
    For $Q,K\in \R^{(d+1)\times (d+1)}$, exist $B,C \in \R^{d\times d}$ such that $Q = {\small\begin{bmatrix}
    B & 0 \\ 0 & 0
\end{bmatrix}}$ and $K = {\small\begin{bmatrix}
    C & 0 \\ 0 & 0
\end{bmatrix}}.$
\end{assumption}
In words, Assumption \ref{ass:QK_sparsity} restricts the last row/column of $Q,K$ to be $0$; this restriction was considered in \cite{von2022transformers,ahn2023transformers}. and is naturally satisfied by the global minimum of 1-layer Linear Transformers in \cite{ahn2023transformers,mahankali2023one,zhang2023trained}.

\begin{tcolorbox}[enhanced,title=,
                    frame hidden,
                    colback=gray!5,
                    breakable,
                    left=1pt,
                    right=1pt,
                    top=1pt,
                    bottom=1pt,
                ]
\begin{example}[Linear Transformer] 
    \label{ex:linear}
The linear attention module of~\citep{von2022transformers}, with parameters $V,Q,K$, is given by 
\begin{align*}
    \att^{linear}_{V,Q,K}(Z) := V Z M Z^\top Q^\top K Z.
\end{align*}
Assume $Q,K,B,C$ satisfy Assumption \ref{ass:QK_sparsity}. By choosing $\th(U,W):= U^\top W$, $\att^{\th}_{VBC}$ from \eqref{d:att_th} equals $\att^{linear}_{VQK}$. 
\end{example}
\end{tcolorbox}

\begin{tcolorbox}[enhanced,title=,
                    frame hidden,
                    colback=gray!5,
                    breakable,
                    left=1pt,
                    right=1pt,
                    top=1pt,
                    bottom=1pt,
                ]
\begin{example}[ReLU Transformer]
    \label{ex:relu}
    The ReLU attention module, with parameters $V,Q,K$, is given by
    \begin{align*}
        \att^{relu}_{V,Q,K}(Z) := V Z M\relu\lrp{Z^\top Q^\top K Z}.
    \end{align*}
    In the above, $\relu$ denotes element-wise ReLU function. Assume $Q,K,B,C$ satisfy Assumption \ref{ass:QK_sparsity}. By choosing $\th\lrp{U,W} = \relu\lrp{U^\top W}$, $\att^{\th}_{VBC}$ from \eqref{d:att_th} equals $\att^{relu}_{VQK}$. 
\end{example}
\end{tcolorbox}

\begin{tcolorbox}[enhanced,title=,
                    frame hidden,
                    colback=gray!5,
                    breakable,
                    left=1pt,
                    right=1pt,
                    top=1pt,
                    bottom=1pt,
                ]
\begin{example}[Softmax Transformer]
    \label{ex:softmax}
    The softmax attention module, with parameters $V,Q,K$, is given by
    \begin{align*}
        \att^{softmax}_{V,Q,K}(Z) := V Z \softmax\lrp{Z^\top Q^\top K Z}.
    \end{align*}
    In the above, $\softmax: \R^{(n+1)\to (n+1)}$ is the masked softmax function:
    \begin{align*}
        \lrb{\softmax(W)}_{ij} = \twocase{\frac{\exp\lrp{W_{ij}}}{\sum_{k=1}^n \exp\lrp{W_{kj}}}}{i\neq n+1}{0}{i=n+1}.
    \end{align*}
    Let $Q,K,B,C$ satisfy Assumption \ref{ass:QK_sparsity}. Let us define
    \begin{align*}
        \lrb{\th\lrp{U,W}}_{ij} := \twocase{\frac{\exp\lrp{\lrb{U^\top W}_{ij} }}{\sum_{k=1}^n \exp\lrp{ \lrb{U^\top W}_{kj}}}}{i\neq n+1}{0}{i=n+1}.
        \numberthis \label{e:softmax_th}
    \end{align*}
    Then $\att^{\th}_{VBC}$ from \eqref{d:att_th} equals $\att^{softmax}_{VQK}$. Note that the mask matrix $M$ from \eqref{d:att_th} is unnecessary here as $M\th(\cdot) = \th(\cdot)$.
\end{example}
\end{tcolorbox}

\section{Transformers \emph{can} implement gradient descent in function space.}
In this section, we show that under a choice of $V,B,C$ and $\tilde{h}$, the forward pass of the Transformer defined in \eqref{e:dynamics_Z} can implement \emph{Kernel Regression} for a kernel $\K$. We present the necessary background on RKHS in Section \ref{s:rkhs_basics}.

We begin by defining \emph{``gradient descent in function space.''} Let $\HH$ denote a Hilbert space of functions mapping from $\R^d\to \R$, equipped with the metric $\lrn{\cdot}_{\HH}$. Let $L(f): \HH \to \R$ denote some loss. The gradient descent of $L(f)$ with respect to $\lrn{\cdot}_{\HH}$ is defined as the sequence
\begin{align*}
    f_{\ell+1} = f_\ell - r_\ell \nabla L (f_\ell),
    \numberthis \label{d:functional_gradient_descent}
\end{align*}
where $\nabla L(f) := \at{\argmin_{\lrn{g}_{\HH}=1} \frac{d}{dt} L(f + tg)}{t=0}$, and $r_\ell$ is a sequence of stepsizes.

\subsection{Transformers \emph{can} implement gradient descent in function space.}
\label{s:can_implement_function_descent}
The first main result of this section is Proposition \ref{p:rkhs_descent_transformer_construction} below, which shows that a Transformer can implement the functional descent sequence \eqref{d:functional_gradient_descent}. We highlight that Proposition \ref{p:rkhs_descent_transformer_construction} \textbf{works for any kernel $\K$} -- as long as the choice of $\th$ coincides with $\K$. In Sections \ref{ss:construction_linear} and \ref{ss:construction_softmax_rbf}, we motivate Proposition \ref{p:rkhs_descent_transformer_construction} using specific examples. 

\begin{proposition}
    \label{p:rkhs_descent_transformer_construction}
    Let $\K$ be an arbitrary kernel. Let $\HH$ denote the Reproducing Kernel Hilbert space induced by $\K$. Let $\tz{i} = (\tx{i},\ty{i})$ for $i=1\ldots n$ be an arbitrary set of in-context examples. Denote the empirical loss functional by $L(f):= \sum_{i=1}^n \lrp{f(\tx{i}) - \ty{i}}^2$. Let $f_0=0$ and let $f_\ell$ denote the gradient descent sequence of $L$ wrt $\lrn{\cdot}_{\HH}$, as defined in \eqref{d:functional_gradient_descent}. Then there exist scalars stepsizes $r_0'\ldots r_k'$ such that the following holds:
    
    Let $\th$ be the function defined as $\lrb{\tilde{h}(U,W)}_{i,j} := \K\lrp{U^{(i)}, W^{(j)}}$, where $U^{(i)}$ and $W^{(i)}$ denote the $i^{th}$ column of $U$ and $W$ respectively. Let 
    $V_\ell = \begin{bmatrix}
    0 & 0 \\ 
    0 & -r_\ell'
    \end{bmatrix}$, $B_\ell = I_{d\times d}$, $C_\ell = I_{d\times d}$. Then for any $x:=\tx{n+1}$, the Transformer's prediction for $\ty{n+1}$ at each layer $\ell$ matches the prediction of the functional gradient sequence \eqref{d:functional_gradient_descent} at step $\ell$, i.e. for all $\ell=0\ldots k$,
    \begin{align*}
        \TF_\ell(x; (V,B,C)\vert \tz{1}\ldots \tz{n}) = - f_\ell(x).
        \numberthis \label{e:tf_fi_equivalence}
    \end{align*}
\end{proposition}
In Theorem \ref{t:informal_master_sparse} below, we show that the above choices of $V_\ell, B_\ell, C_\ell$ form a stationary point of the Transformer training objective. In Section \ref{ss:experiments_for_sparse}, we emprically verify that these parameter choices are consistently learned in experiments. In Sections \ref{ss:construction_linear} and \ref{ss:construction_softmax_rbf}, we show how Proposition \ref{p:rkhs_descent_transformer_construction} applies to two common settings.

\begin{proof}[Proof of Proposition \ref{p:rkhs_descent_transformer_construction}] 
We will first write down the explicit expression for \eqref{d:functional_gradient_descent}. By Lemma \ref{l:rkhs_descent}, for any $f\in \HH$, $\nabla L(f) = - c \sum_{i=1}^n \lrp{\ty{i} - f(\tx{i})} \K(\cdot, \tx{i})$, thus \eqref{d:functional_gradient_descent} is equivalent to
\begin{align*}
    f_{\ell+1}(\cdot)  = f_\ell(\cdot) + r_\ell' \sum_{i=1}^n \lrp{\ty{i} - f_\ell(\tx{i})} \K(\cdot, \tx{i})
    \numberthis \label{e:t:oaimda:1}.
\end{align*}
See proof of Lemma \ref{l:rkhs_descent} for the explicit relation between $r_\ell'$ and $r_\ell$.

Let $X_\ell$ and $Y_\ell$ denote the first $d$ rows and the last row of $Z_\ell$ respectively, for any layer $\ell=0\ldots k+1$ and for any $i=0\ldots n$, 
\begin{align*}
    \tY{i}_\ell = \ty{i} + \TF_\ell(\tx{i}; (V,B,C), \tz{1}\ldots \tz{n}).
    \numberthis \label{e:t:oaimda:0}
\end{align*}
In words: "$\ty{i}-\tY{i}_\ell$ is equal to the predicted label for $\tx{n+1}$, if $\tx{i} = \tx{n+1}$." \eqref{e:t:oaimda:0} follows immediately from \eqref{e:dynamics_Z}, by setting $\tx{n+1} = \tx{i}$, and verifying that $\lrb{Z_\ell}_{(d+1),i}$ and $\lrb{Z_\ell}_{(d+1),(n+1)}$ have identical updates across layers $\ell =0\ldots k$.

We will now prove the lemma statement by induction. For the input, $\lrb{Z_0}_{(d+1),(n+1)}:=0 = f_0(x)$ by definition in \eqref{d:Z_0}, so that \eqref{e:tf_fi_equivalence} holds. Now assume that $\TF_\ell(x; (V,B,C)\vert \tz{1}\ldots \tz{n}) = -f_\ell(x)$ up to some layer $\ell$. 

By definition of the dynamics on $Z_i$ in \eqref{e:dynamics_Z}, and plugging in our choice of $V, B, C$, we verify that
\begin{align*}
    & \TF_{\ell+1}(x; (V,B,C)\vert \tz{1}\ldots \tz{n}) \\
    =& \TF_{\ell}(x; (V,B,C)\vert \tz{1}\ldots \tz{n}) - r_\ell \sum_{i=1}^n \tY{i}_\ell \lrb{\th\lrp{X_0, X_0}}_{i,(n+1)}
    \numberthis \label{e:alskda:0}\\
    =& \TF_{\ell}(x; (V,B,C)\vert \tz{1}\ldots \tz{n}) - r_\ell \sum_{i=1}^n \tY{i}_\ell \K(\tx{i},x)
    \numberthis \label{e:alskda:1}\\
    =& -f_{\ell} (x) - r_\ell \sum_{i=1}^n \lrp{\ty{i} - f_{\ell} (x)} \K(\tx{i},x)\\
    =& -f_{\ell+1}(x).
\end{align*}
In the above, the first line is by plugging in our choice of $V,B,C$ into \eqref{e:dynamics_Z}. The second line is by our assumption on $\th$ in the lemma statement. The third line is by inductive hypothesis, along with \eqref{e:t:oaimda:0}. The fourth line is by \eqref{e:t:oaimda:1}. This concludes the proof.
\end{proof}
\noindent
As an immediate consequence of step \eqref{e:alskda:1} of the proof above, the functional gradient descent sequence \eqref{d:functional_gradient_descent} is equivalent to
\begingroup\abovedisplayskip=0.4em\belowdisplayskip=0.4em
\begin{align*}
    f_{\ell+1}(\cdot)  = f_\ell(\cdot) + r_\ell' \sum_{i=1}^n \lrp{\ty{i} - f_\ell(\tx{i})} \K(\cdot, \tx{i})
    \numberthis \label{e:t:oaimda:2}
\end{align*}
\endgroup
for some stepsizes $r_\ell'$. 

\subsubsection{Case Study: Linear Kernel}
\label{ss:construction_linear}
We first consider the simplest setting of the \textbf{Euclidean inner product Kernel}, i.e. $\K^{linear}(u,w) := \lin{u,w}$. In this setting, the choice of query, key, value matrices in Proposition \ref{p:rkhs_descent_transformer_construction} essentially match constructions in \cite{von2022transformers,ahn2023transformers}. It is also worth noting that \emph{functional gradient descent in the RKHS induced by $\K^{linear}$} in fact follows the same trajectory as \emph{(Euclidean) gradient descent of the linear regression parameter vector}: Let $\theta \in \R^d$ be the linear regression parameter to learn. Let $R(\theta):= \displaystyle\frac{1}{2} \sum_{i=1}^n \lrp{\lin{\tx{i}, \theta} - \ty{i}}^2$ denote the empirical least squares loss. Let $\theta_\ell$ denote the $\ell^{th}$ iterate of gradient descent with stepsize $r_\ell'$, thus
\begin{align*}
    \theta_{\ell+1} = \theta_\ell - r_\ell' \nabla R(\theta_\ell).
\end{align*}
Let $f_\ell(x) := \lin{\theta_\ell, x}$. Notice that $\nabla R(\theta_\ell) = - \sum_{i=1}^n \lrp{\ty{i} - f_k(\tx{i})} \tx{i}$, so that
\begin{align*}
    f_{\ell+1}(x) 
    =& f_{\ell}(x) + r_\ell' \sum_{i=1}^n \lrp{\ty{i} - f_\ell(\tx{i})} \lin{\tx{i},x}
\end{align*}
which is exactly the same as \eqref{d:functional_gradient_descent} (or more specifically, \eqref{e:t:oaimda:2}), by noting that $\lin{\tx{i},x} = \K^{linear}(\tx{i}, x)$.

\subsubsection{Case Study: Exponential Kernel, and connection to Softmax activation}
\label{ss:construction_softmax_rbf}
We will now show that the construction in Proposition \ref{p:rkhs_descent_transformer_construction}, when $\K$ is the exponential kernel, bears remarkable similarity to the softmax Transformer with identity weights.

Let $\K$ denote the exponential kernel with bandwidth $\sigma$, i.e. $\K(x,x') := \exp\lrp{\frac{1}{\sigma^2}\lin{x,x'}}$.
Choose $\lrb{\th^{exp}\lrp{U,W}}_{ij}:= \K(U_i, W_j)$ (where $U_i$ is the $i^{th}$ column of $U$), and choose Transformer parameters
\begin{align*}
    V_\ell = \begin{bmatrix}
    0 & 0 \\ 
    0 & - r_\ell
    \end{bmatrix}, 
    B_\ell = \frac{1}{\sigma} I_{d\times d}, 
    C_\ell = \frac{1}{\sigma}I_{d\times d}.
\end{align*}
Recall that $f_\ell(\tx{n+1})$ denotes the Transformer's prediction for $\ty{n+1}$ at layer $\ell$. The $\th^{exp}$ Transformer's prediction at each layer follows the functional gradient descent sequence from \eqref{e:t:oaimda:2}. 

On the other hand, recall from Example \ref{ex:softmax} the standard $\th^{softmax}$ Transformer with $\th^{softmax}$ defined in \eqref{e:softmax_th}. By choosing parameters
\begin{align*}
    V_\ell = \begin{bmatrix}
    0 & 0 \\ 
    0 & - r_\ell
    \end{bmatrix}, 
    Q_\ell = \begin{bmatrix}
    \frac{1}{\sigma} I_{d\times d} & 0 \\ 
    0 & 0
    \end{bmatrix}, K_\ell = \begin{bmatrix}
    \frac{1}{\sigma} I_{d\times d} & 0 \\ 
    0 & 0
    \end{bmatrix},
\end{align*}
The softmax-activated Transformer implements the update 
\begin{align*}
\numberthis \label{e:t:adnsnadkjnk}
& f_{\ell+1}(\cdot)  = f_\ell(\cdot)  + r_\ell' \color{blue}\tau(\cdot)\color{black} \sum_{i=1}^n \lrp{\ty{i} - f_\ell(\tx{i})} \K(\cdot, \tx{i}),
\end{align*}
where $\tau(\cdot) = 1/\sum_{j=1}^n \K\lrp{\cdot,\tx{j}}$ is the normalization factor in softmax. Comparing \eqref{e:t:adnsnadkjnk} to \eqref{e:t:oaimda:2}, we see that the algorithms implemented by $\th^{exp}$ and $\th^{softmax}$ Transformers are very similar, with the only difference being the normalization factor $\tau$.

\begin{remark}
    When $\|\tx{i}\|_2 = 1$ for all $i=1\ldots n+1$, the exponential kernel is up to scaling equal to the RBF kernel.
\end{remark}

\subsection{Optimality of $\th$ for matching $\K$.}
\label{ss:optimality}
We will show in Proposition \ref{p:matching_h_k_optimality} below, that the functional gradient descent algorithm \eqref{d:functional_gradient_descent}, which is implemented by the Transformer in Proposition \ref{p:rkhs_descent_transformer_construction}, can in fact lead to a \emph{nearly statistically optimal} prediction when the non-linear activation $\th$ matches the data distribution. We begin by defining a general class of data distributions. Let $\K:\R^d\times \R^d \to \R$ denote a symmetric function.  We define a conditional distribution for $Y|X$ as follows:

\begin{definition}[$\K$ Gaussian Process]
    \label{d:k_gaussian_process}
    Given symmetric $\K:\R^d\times \R^d \to \R$, we define the $\K$ Gaussian Process as the conditional distribution
    \begin{align*}
        Y|X \sim \N(0,\Kmat_+(X)),
    \end{align*}
    where $Y= [\ty{1}\ldots \ty{n+1}]$, $X = [\tx{1}\ldots \tx{n+1}]$, and $\Kmat_{ij}(X) := \K\lrp{\tx{i},\tx{j}}$. Let $U D U^\top$ be the Eigenvalue decomposition of $\Kmat(X)$. $\Kmat_+(X) := U |D| U^\top$, where $|D|_{ii} := |D_{ii}|$ is entry-wise equal to the absolute value of $D$. 
\end{definition}
Definition \ref{d:k_gaussian_process} generalizes the notion of a Gaussian process, to when the "metric" is given by the function $\K$. In Example \ref{ex:gaussian_process_kernel_examples} from Section \ref{ss:distributional_assumptions}, we discuss a few concrete examples of $\K$ Gaussian Processes. 
Note that Definition \ref{d:k_gaussian_process} \textbf{\emph{does not assume that $\K$}} is a kernel (specifically, $\K$ may not be PSD). However, if $\K$ is a kernel, then $\Kmat$ is always positive semidefinite, so that $\Kmat_+ = \Kmat$. 

In the following result, we see that when the data labels are generated by a $\K$-Gaussian Process for some kernel $\K$, then the Transformer prediction \eqref{e:transformer_prediction}, for the construction from Proposition \ref{p:rkhs_descent_transformer_construction}, is statistically optimal if $\th$ matches $\K$:

\begin{proposition}
    \label{p:matching_h_k_optimality}
    Let $X = [\tx{1}\ldots \tx{n+1}]$, $Y = [\ty{1}\ldots \ty{n+1}]$. Let $\K: \R^d\times \R^d\to\R$ \emph{be a kernel}. Assume that $Y|X$ is drawn from the $\K$ Gaussian Process. Let the attention activation $[\th(U,W)]_{ij} := \K\lrp{U_i,W_j}$, and consider the functional gradient descent construction in Proposition \ref{p:rkhs_descent_transformer_construction}. Then as the number of layers $\ell \to \infty$, the Transformer's prediction for $\ty{n+1}$ at layer $\ell$ \eqref{e:transformer_prediction} approaches the \textbf{Bayes (optimal) estimator} that minimizes the in-context loss \eqref{d:ICL_loss}.
\end{proposition}
\begin{proof}[Proof of Proposition \ref{p:matching_h_k_optimality}]
    Let $\hat{Y}\in \R^{n}$ denote the vector of $\ty{1}\ldots \ty{n}$. Let $\hat{\Kmat}$ denote the top-left $n\times n$ block of $\Kmat$. Let $\nu\in \R^{n}$ denote the vector given by $\nu_i := \Kmat_{i,n+1}$. i.e. $\Kmat =: \bmat{\hat{\Kmat} & \nu \\ \nu^\top & \Kmat_{(n+1),(n+1)}}$. By the formula for conditional Gaussian, we know that $\ty{n+1}$, conditioned on $\ty{1}\ldots \ty{n}$, has Gaussian distribution with mean $\nu^\top \hat{\Kmat}^{-1} \hat{Y}$. The Bayes estimator of $\ty{n+1}$ is thus exactly this mean, which is equivalent to
    \begin{align*}
        \nu^\top \hat{\Kmat}^{-1} \hat{Y}
        =& \sum_{j,k=1}^n \K\lrp{\tx{n+1}, \tx{j}}\lrb{\hat{\Kmat}^{-1}}_{jk} \ty{k}.
        \numberthis \label{e:f_bayes_kernel}
    \end{align*}
    Consider the construction in Proposition \ref{p:rkhs_descent_transformer_construction}, i.e. $B_\ell = C_\ell = I$, $V_\ell = \bmat{0&0\\0&-r_\ell}$. For simplicity, further assume that $r_\ell = \delta$ for all $\ell$, where $\delta$ is some positive constant satisfying $\delta < \|\hat{\Kmat}\|_{2}$. For $\ell=0\ldots k$, let $\ty{i}_\ell := \lrb{Z_\ell}_{(d+1),i}$, and let $\hat{Y}_\ell := \lrb{\ty{1}_\ell\ldots \ty{n+1}_\ell}$. From \eqref{e:dynamics_Z}, we verify that under the above choice of Transformer weights and ,
    \begin{align*}
        & \ty{i}_{\ell+1} = \ty{i}_\ell - \delta \sum_{j=1}^n \K\lrp{\tx{i},\tx{j}} \ty{j}_{\ell}
        \qquad \Leftrightarrow \qquad 
        \hat{Y}_{\ell+1} = \lrp{I - \delta \hat{\Kmat}} \hat{Y}_{\ell} = \lrp{I - \delta \hat{\Kmat}}^\ell \hat{Y}.
    \end{align*}
    Again by \eqref{e:dynamics_Z}, we verify that $\ty{n+1}_{\ell+1} 
    = \ty{n+1}_{\ell} - \delta \sum_{j=1}^n \K\lrp{\tx{n+1},\tx{j}} \ty{j}_{\ell}$. Rearranging terms gives $\ty{n+1}_{\ell+1} = \ty{n+1}_{\ell} - \delta \nu^\top \hat{Y}_{\ell} = - \nu^\top \sum_{k=0}^\ell \hat{Y}_{k} = -\nu^\top \sum_{k=0}^\ell \lrp{I - \delta \hat{\Kmat}}^k \hat{Y}_{k}$.
    By Taylor expansion, $\hat{\Kmat}^{-1} = \delta \sum_{\ell=0}^\infty \lrp{I-\delta \hat{\Kmat}}^\ell$. Thus as $\ell \to \infty$, $\ty{n+1}_{\ell+1} \to \nu^\top \hat{\Kmat}^{-1} \hat{Y}$, which is the optimal estimator of $\ty{n+1}$. 
\end{proof}

We note a few caveats of Proposition \ref{p:matching_h_k_optimality}: 
\begin{enumerate}
    \item The result guarantees optimality of $\th = \K$ in the limit of $\ell \to \infty$. For finite $\ell$, there may exist a better choice of $\th$ that implements a more iteration-efficient algorithm. For example, see discussion of Figure \ref{f:h_match_k_against_layer_b}.
    \item The construction in Proposition \ref{p:rkhs_descent_transformer_construction} sets the top-left $d\times d$ block of $V_\ell$ to $0$ (i.e. $A_\ell=0$). In practice, as see in Theorem \ref{t:informal_master_full} and Figure \ref{f:theorem_full}, $A_\ell$ is often a non-zero multiple of $I_{d\times d}$. For this reason, a choice of $\th$ that differs from $\K$ may nonetheless recover the Bayes estimator.
\end{enumerate}

\subsection{Experiments for Proposition \ref{p:matching_h_k_optimality}}
\label{ss:experiments_for_optimality}
To experimentally verify Proposition \ref{p:matching_h_k_optimality}, we compare the performance of different choices of $\th$ against different choices of generating kernel $\K$. We present our findings in Figures \ref{f:h_match_k} and \ref{f:h_match_k_against_layer}.

We consider three types of $\K$ Gaussian Processes:
\begin{align*}
    \K^{linear}(u,w) := \lin{u,w} \qquad \K^{relu}(u,w) := \relu(\lin{u,w}) \qquad \K^{exp}(u,w) := \exp\lrp{\lin{x,y}}.
    \numberthis \label{e:3_kernel_choices}
\end{align*}

For each choice of $\K$ above, we train Transformers with four different types of non-linear attention module $\th$ below: 
\begin{alignat}{1}
    & [\th^{linear}(U,W)]_{ij}:= [U^\top W]_{ij}, \quad [\th^{relu}(U,W)]_{ij}:= \relu([U^\top W]_{ij}), \quad [\th^{exp}(U,W)]_{ij}:= \exp([U^\top W]_{ij}),\\
    &[\th^{softmax}\lrp{U,W}]_{ij} := \twocase{\frac{\exp\lrp{\lrb{U^\top W}_{ij} }}{\sum_{k=1}^n \exp\lrp{ \lrb{U^\top W}_{kj}}}}{i\neq n+1}{0}{i=n+1}.
    \numberthis \label{e:4_kernel_choices}
\end{alignat}   

\textbf{Note that Proposition \ref{p:rkhs_descent_transformer_construction} does not apply for $\th^{\mathrm{relu}}$ and $\th^{\mathrm{softmax}}$ as they are not kernels. We nonetheless include them in the experiments as they are the architectures most commonly used in practice.}  

The covariates $\tx{i}$ are drawn iid from the unit sphere, and the labels $\ty{i}$ are drawn from one of the three $\K$ Gaussian Processes. In all plots, the loss values are taken after convergence of training loss. Full experiment details are found in Appendix \ref{ss:common_experiment_details}.

In all our Figures, we will show the loss of the \textbf{Bayes Estimator} $f^{bayes}$ as a baseline. This represents the information-theoretically optimal loss. Recall from \eqref{d:k_gaussian_process} that $Y|X \sim \N(0,\Kmat_+(X))$. Let $\Kmat$ and $\Kmat_+$ be as defined in Definition \ref{d:k_gaussian_process}. Let $\lrbb{\hat{\Kmat}, \nu, \mu}\in \lrbb{\R^{d\times d}, \R^d, \R}$ be \underline{defined} as $\bmat{\hat{\Kmat} & \nu \\ \nu^\top & \mu}:= \Kmat_+$. Let $\hat{Y}\in \R^{n}$ denote the vector of $\ty{1}\ldots \ty{n}$. Then the Bayes Estimator is defined as
\begin{align*}
    f^{bayes}(\tx{n+1}) := \nu^\top \hat{\Kmat}^{-1} \hat{Y}.
    \numberthis \label{e:f_bayes}
\end{align*}
When $\K$ is a PSD kernel, $\Kmat_+:= \Kmat$, and \eqref{e:f_bayes} is identical to \eqref{e:f_bayes_kernel}. More generally, when $\K$ is not PSD, \eqref{e:f_bayes} and \eqref{e:f_bayes_kernel} are \textbf{not equal}. Nonetheless, $f^{bayes}$ from \eqref{e:f_bayes} is a well-defined estimator.

Figure \ref{f:h_match_k} plots the in-context loss of a 3-layer Transformer against \textbf{number of demonstrations} $n \in \lrbb{2,4,6,8,10,12}$, for different combinations of label-generating kernel $\K$ and attention module $\th$ (see \eqref{e:4_kernel_choices}). Figure \ref{f:h_match_k_against_layer} plots the in-context loss against \textbf{number of layers} $k \in \lrbb{1,2,3,4,5,6,7,8}$, for $n\in\lrbb{14,6}$. We show the losses of different combinations of $\K$ and $\th$. We summarize key observations of interest below:

\begin{enumerate}
    \item \textbf{\underline{Ignoring $\th^{softmax}$}, the best prediction error is achieved when the attention activation $\th$ matches distribution $\K$.} From Figures \ref{f:h_match_k_a}, \ref{f:h_match_k_b}, \ref{f:h_match_k_against_layer_a}, \ref{f:h_match_k_against_layer_c}, \ref{f:h_match_k_against_layer_d}: the best accuracy is obtained when the attention activation $\th$ matches $\K$, as suggested by Proposition \ref{p:matching_h_k_optimality}.

    \item \textbf{The \textcolor{black}{$\th^{softmax}$} attention is most accurate for \textcolor{black}{$\K^{exp}$} labels when number of layers is small and $n$ is large.}
    From Figure \ref{f:h_match_k_c}: with $3$ layers, \textcolor{black}{$\th^{softmax}$} is more accurate than \textcolor{black}{$\th^{exp}$} on \textcolor{black}{$\K^{exp}$} data for $n\in\lrbb{6,8,10,12,14}$. From Figure \ref{f:h_match_k_against_layer_b}, when $n=14$, the gap between \textcolor{black}{$\th^{softmax}$} and \textcolor{black}{$\th^{exp}$} (for \textcolor{black}{$\K^{exp}$} data) becomes very small for number of layers $\geq 6$. From Figure \ref{f:h_match_k_against_layer_d}, when $n=6$, \textcolor{black}{$\th^{exp}$} has the highest accuracy for \textcolor{black}{$\K^{exp}$} data when number of layers $\geq 5$. 
    \begin{enumerate}
        \item We conjecture that $\th^{softmax}$ implements an algorithm that is \emph{more iteration efficient} but \emph{less statistically efficient} than functional gradient descent. As each layer implements a step of some algorithm, $\th^{softmax}$ performs well with few layers (steps), and performs relatively poorly when number of samples is small.
        \item The relative performance of $\th^{softmax}$ and $\th^{exp}$ is consistent with Proposition \ref{p:matching_h_k_optimality}, which predicts that the $\th^{\exp}$ Transformer approaches Bayes-optimal prediction loss \textbf{as number of layers increases.} We also note that $\th^{softmax}$ is closely related to $\K^{exp}$, as discussed at the end of Example \ref{ss:construction_softmax_rbf}.
    \end{enumerate}
    \item For each Transformer, the \textbf{parameters learned are as predicted in Theorem \ref{t:informal_master_full}}. See experiments in Section \ref{ss:experiments_for_full} for details. 
    \end{enumerate}

    Finally, we also note that the gap between $\th^{relu}$ and the Bayes estimator in Figure \ref{f:h_match_k_against_layer_a} is quite significant, and does not seem to decrease with number of layers. This is likely because the Bayes estimator $f^{bayes}$ in \eqref{e:f_bayes} uses $\hat{\Kmat}$ which in turn requires flipping eigenvalues on $\Kmat$. Such an operation may not be easily implementable by single head Transformers.

\begin{figure}[H]
\centering
\begin{subfigure}{0.32\textwidth}
\centering
\includegraphics[width=\textwidth]{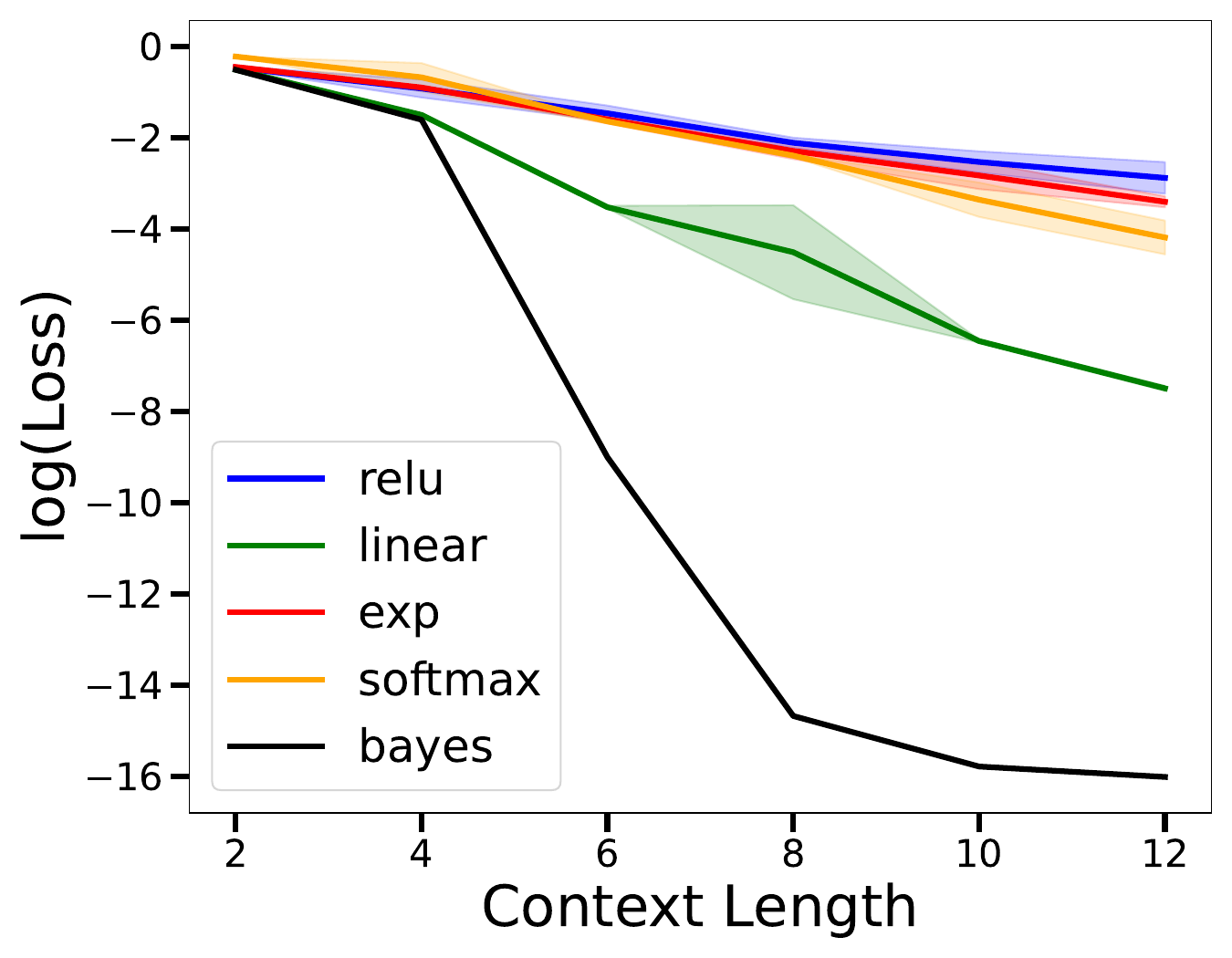} 
\caption{$\K^{linear}$}
\label{f:h_match_k_a}
\end{subfigure}\hfill
\begin{subfigure}{0.32\textwidth}
\centering
\includegraphics[width=\textwidth]{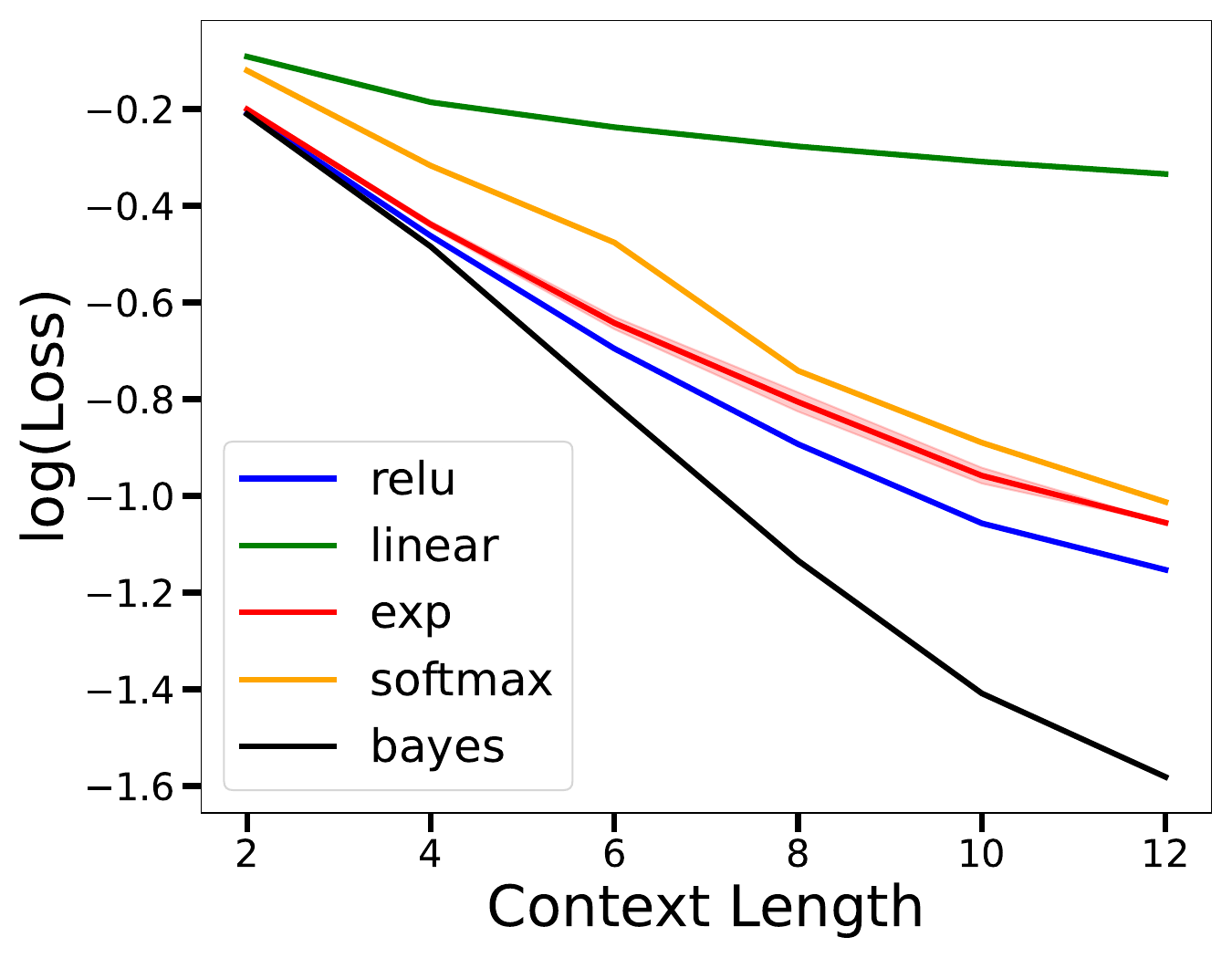}  
\caption{$\K^{relu}$}
\label{f:h_match_k_b}
\end{subfigure}
\begin{subfigure}{0.32\textwidth}
\centering
\includegraphics[width=\textwidth]{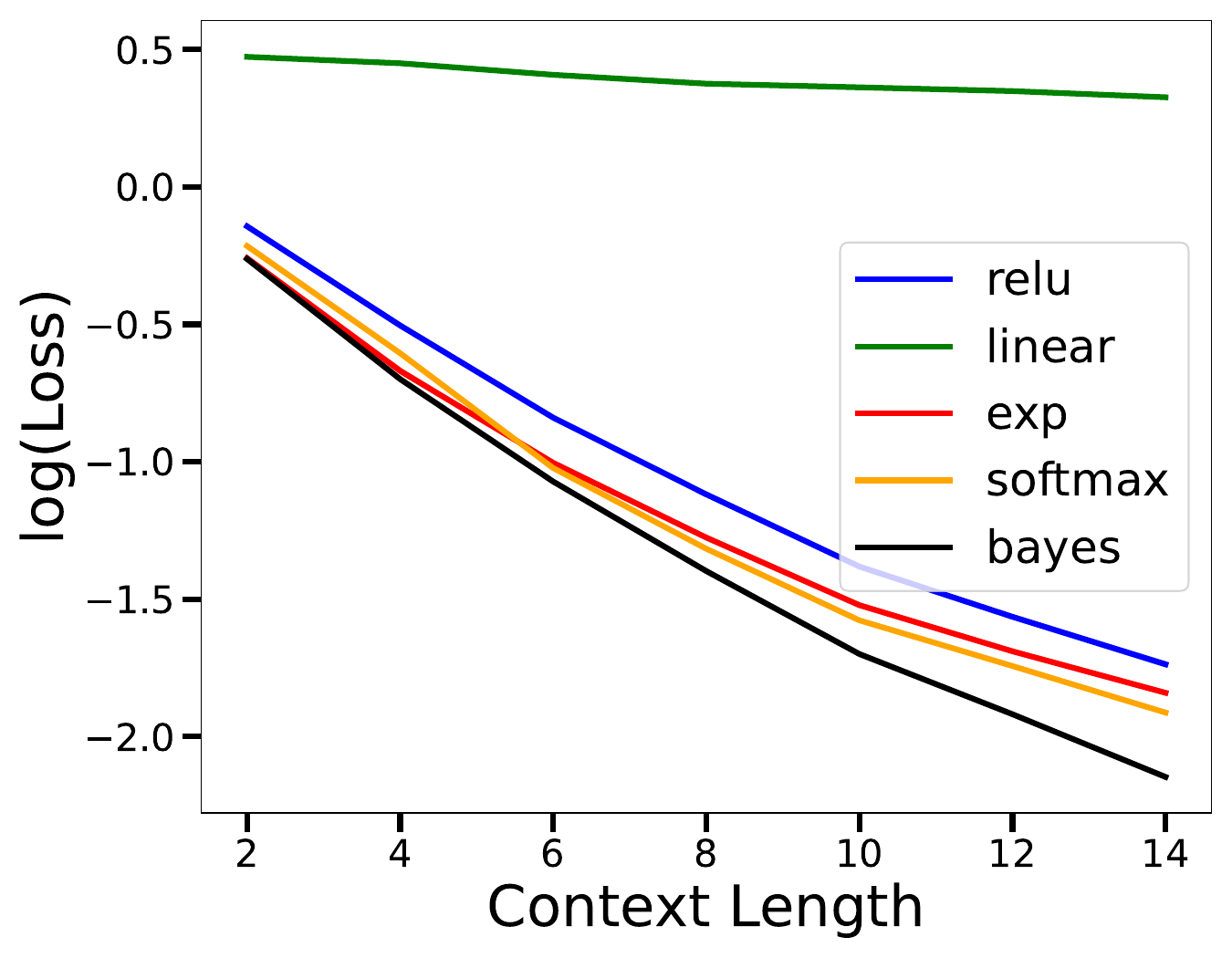}  
\caption{$\K^{exp}$}
\label{f:h_match_k_c}
\end{subfigure}
\caption{Plot of log(test ICL loss) against number of in-context demonstrations. The labels are generated using a $\K$ Gaussian Process (Definition \ref{d:k_gaussian_process}) Each sub-figure corresponds to one of three choices of $\K$, defined in \eqref{e:3_kernel_choices}. Each sub-figure contains 4 plots corresponding to 4 choices of $\th$, as defined in \eqref{e:4_kernel_choices}. Black line denotes Bayes Loss.}
\label{f:h_match_k}
\end{figure}

\captionsetup[subfigure]{oneside,margin={1cm,0cm}}
\begin{figure}[H]
\centering
\begin{subfigure}{0.4\textwidth}
\centering
\includegraphics[width=\textwidth]{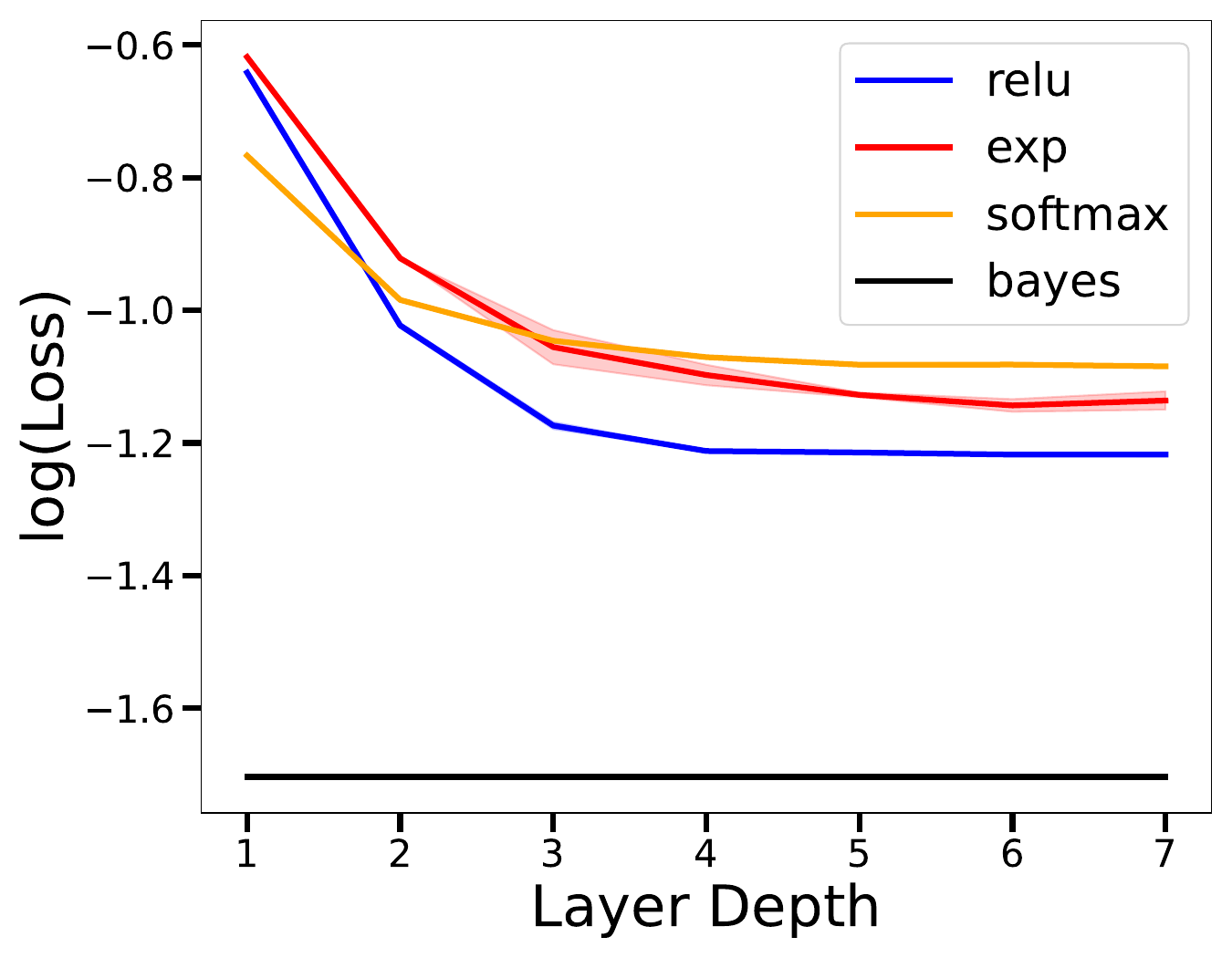} 
\caption{$\K^{relu}$, $n=14$}
\label{f:h_match_k_against_layer_a}
\end{subfigure}
\begin{subfigure}{0.4\textwidth}
\centering
\includegraphics[width=\textwidth]{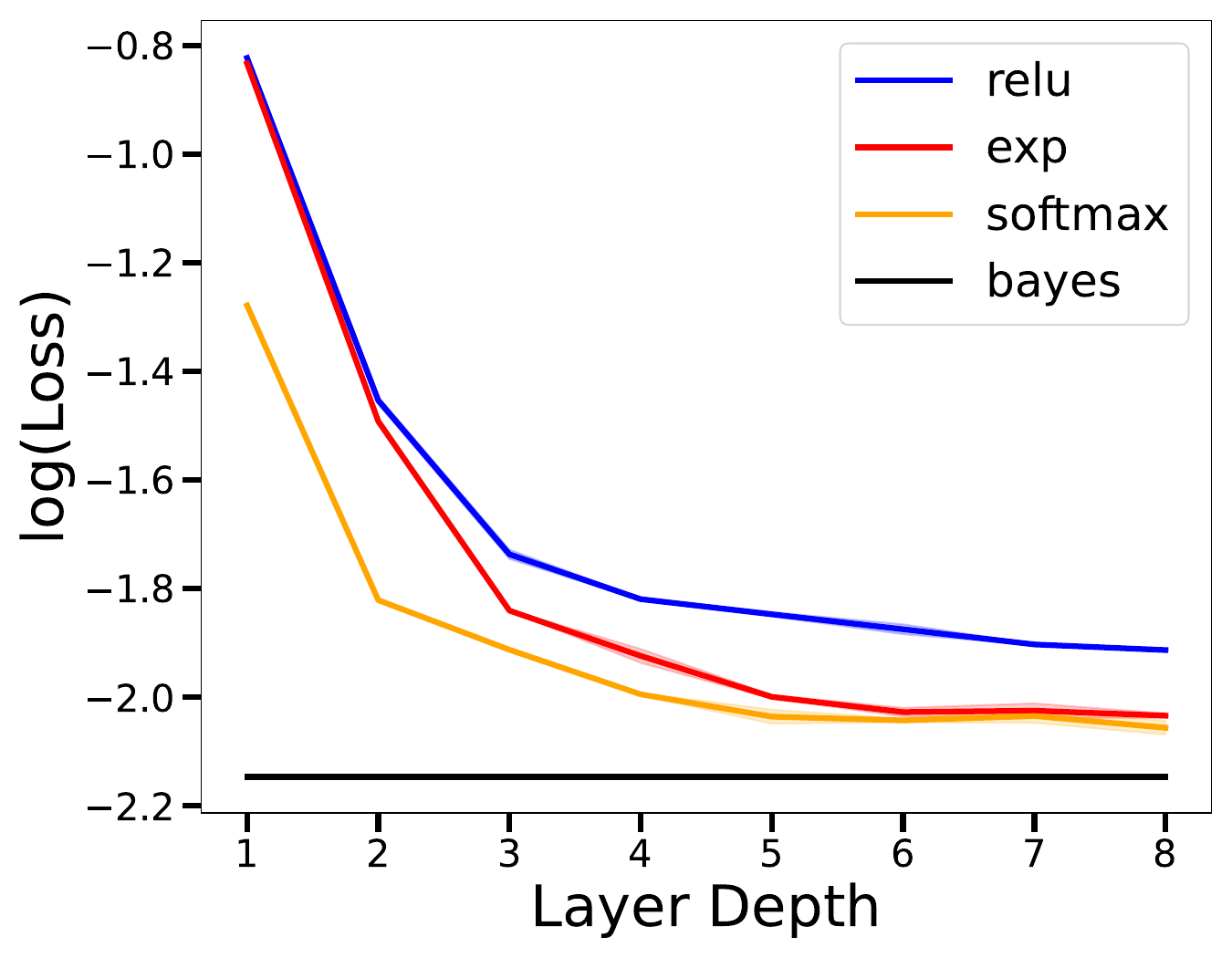}  
\caption{$\K^{exp}$, $n=14$}
\label{f:h_match_k_against_layer_b}
\end{subfigure}\\
\begin{subfigure}{0.4\textwidth}
\centering
\includegraphics[width=\textwidth]{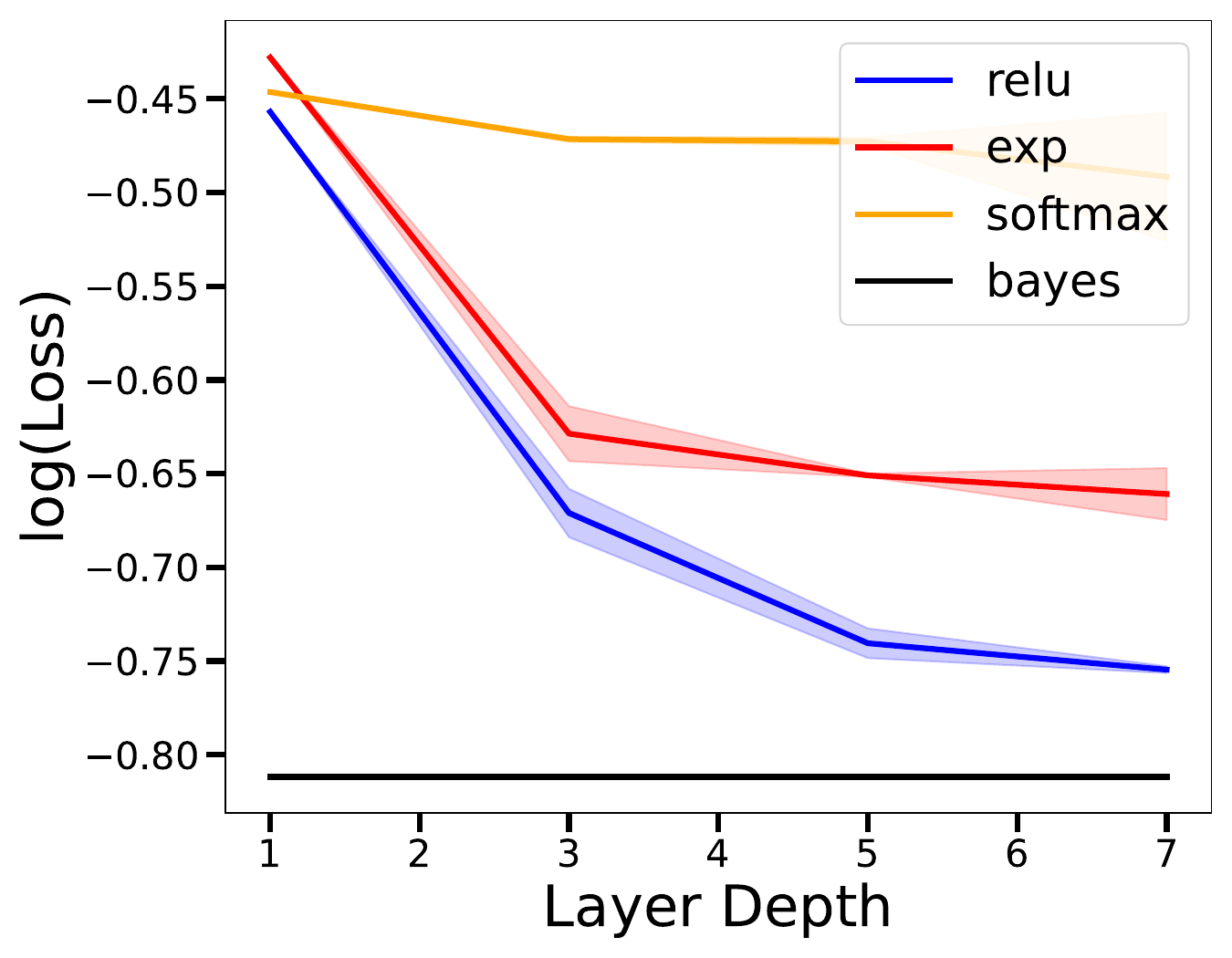} 
\caption{$\K^{relu}$, $n=6$}
\label{f:h_match_k_against_layer_c}
\end{subfigure}
\begin{subfigure}{0.4\textwidth}
\centering
\includegraphics[width=\textwidth]{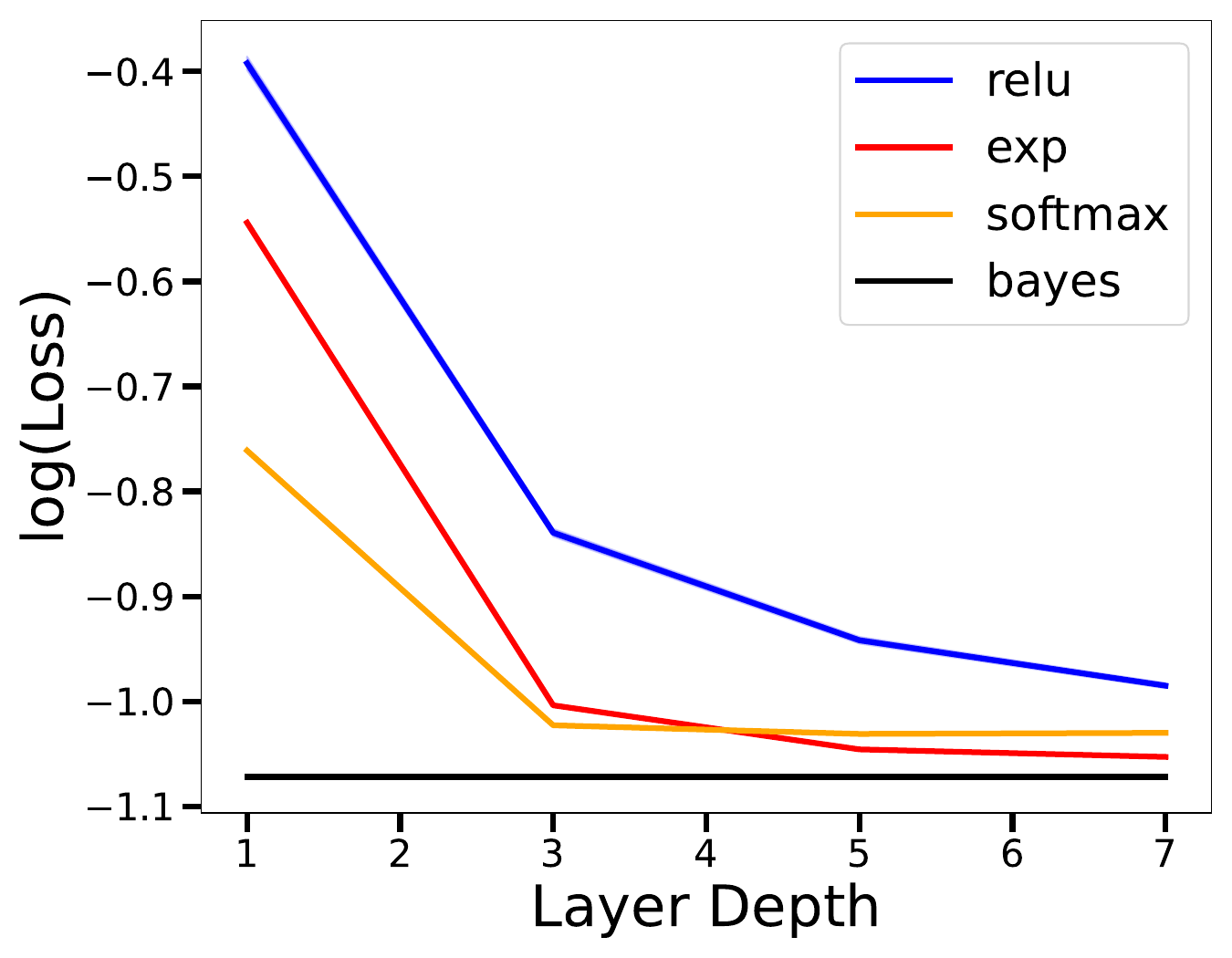}  
\caption{$\K^{exp}$, $n=6$}
\label{f:h_match_k_against_layer_d}
\end{subfigure}

\caption{Plot of log(test ICL loss) against number of layers. The labels are generated using a $\K$ Gaussian Process (Definition \ref{d:k_gaussian_process}), for $\K^{relu}$ and $\K^{exp}$ as defined in \eqref{e:3_kernel_choices}. Each sub-figure contains 3 plots corresponding to three choices of $\th$, as defined in \eqref{e:4_kernel_choices}. The two plots in the top row have $n=14$ demonstrations. The two plots in the bottom row have $n=6$ demonstrations. Black line denotes Bayes Loss.}
\label{f:h_match_k_against_layer}
\end{figure}

\subsection{Composing multiple attention heads with different activations}
\label{ss:multihead}
In this section, we show that \textbf{multi-head} Transformers with \textbf{different activations per-head} can attain much greater representation power. A powerful aspect of RKHS theory is the ability to form complex kernels by composing simple ones via addition and multiplication. Using this idea, we show, both theoretically (Proposition \ref{c:multihead}) and empirically (Figure \ref{f:multihead}), that a \textbf{multi-head Transformer can attain optimal prediction loss for a large class of $\K$ Gaussian Processes which are obtained from kernel composition}.

Formally, we will consider Transformers with \emph{multi-head} attention, defined by the following forward pass:
\begin{align*}
    \numberthis \label{e:dynamics_Z_multihead}
    Z_{\ell+1} = Z_\ell + \sum_{s=1}^H V^s_\ell Z_\ell M \th^s \lrp{B^s_\ell X_\ell, C^s_\ell X_\ell }.
\end{align*}
$H$ denotes the number of heads in a layer, and $\lrbb{V^s_\ell, B^s_\ell,C^s_\ell}_{s=1...H}$ denote the $\lrbb{\text{value, key, query}}$ matrices at layer $\ell$ for head $s$/ $\th^{s}$ denotes the activation for head $s$, \textbf{which could be different for each head}. The difference between \eqref{e:dynamics_Z_multihead} and \eqref{e:dynamics_Z}, is the additional summation over multiple heads $\sum_{s=1}^H$. Identical to to \eqref{e:transformer_prediction}, we let $\TF_\ell$ denote the Transformer's prediction for $-\ty{n+1}$ at layer $\ell$, given $\tx{n+1} = x$, conditioned on $\tz{1}...\tz{n}$ as:
\begin{align*}
    \numberthis \label{e:transformer_prediction_multihead}
    \TF_\ell(\textcolor{blue}{x}; (V,B,C)\vert \tz{1}\ldots \tz{n}):=\lrb{Z_{\ell}}_{(d+1),(n+1)},
\end{align*}
where $Z_i$ evolves as \eqref{e:dynamics_Z_multihead}, initialized at
{$Z_0 = 
{\lrb{\begin{smallmatrix}
\tx{1} & \tx{2} & \cdots & \tx{n} &\color{blue}{x} \\ 
\ty{1} & \ty{2} & \cdots &\ty{n}& 0
\end{smallmatrix}}}$}. We now present Proposition \ref{c:multihead} which shows that a single multi-head Transformer can perform (optimal) \textbf{functional gradient descent} with respect to a large class of RKHS metrics. Its proof is very similar to Propositions \ref{p:rkhs_descent_transformer_construction} and \ref{p:matching_h_k_optimality} (see Appendix \ref{ss:proof:c:multihead}).

\begin{proposition}
    \label{c:multihead}
    Let $\lrbb{\tz{i}}_{i=1...n}$ denote the in-context examples and let $L(f)$ be the empirical loss functional as defined in Proposition \ref{p:rkhs_descent_transformer_construction}. For $s=1...H$, let $\K^{s}$ denote a PSD kernel function. Let $\K^{\diamond}$ be a composite kernel, defined as $\K^{\diamond}(u,v) := \sum_{s=1}^H \K^s\lrp{G^s u, G^sv}$, where $G^s\in \R^{d\times d}$ are subject to the constraint that $\K^{\diamond}$ must be PSD (but are otherwise arbitrary). Let $f_\ell$ denote the functional gradient descent \eqref{d:functional_gradient_descent} of $L(f)$, wrt the RKHS metric induced by $\K^{\diamond}$.
    \begin{enumerate}[label=(\Alph*)]
        \item \textbf{{[}Generalization of Proposition \ref{p:rkhs_descent_transformer_construction}{]}} Consider the multi-head Transformer with $H$ heads, where the $s^{th}$ head has activation defined as $\lrb{\th^s\lrp{U,V}}_{ij}:= \K^s\lrp{U^{(i)}, W^{(j)}}$. Let the Transformer's parameters be $V^s_\ell = \begin{bmatrix}
        0 & 0 \\ 
        0 & -r^s_\ell
        \end{bmatrix}$, $B^s_\ell = G^s$, $C^s_\ell = G^s$. Then there exist scalars $\{r^s_\ell\}_{s=1...H, \ell=0...k}$ such that the following holds: For any $x:=\tx{n+1}$, the Transformer's prediction for $\ty{n+1}$ at each layer $\ell$ matches the prediction of the functional gradient sequence $f_\ell$ \eqref{e:transformer_prediction_multihead}, i.e. for all $\ell=0\ldots k$,
        \begin{align*}
            \TF_\ell(x; (V,B,C)\vert \tz{1}\ldots \tz{n}) = - f_\ell(x).
            \numberthis \label{e:tf_fi_equivalence_multihead}
        \end{align*}
        \item \textbf{{[}Generalization of Proposition \ref{p:matching_h_k_optimality}{]}} If we additionally assume that $Y|X$ is drawn from the $\K^{\diamond}$ Gaussian Process, then as the number of layers $\ell\to\infty$, the Transformer's prediction for $\ty{n+1}$ at layer $\ell$ \eqref{e:tf_fi_equivalence_multihead} approaches the \textbf{Bayes (optimal) estimator} that minimizes the in-context loss \eqref{d:ICL_loss}.
    \end{enumerate}
\end{proposition}

Remarkably, \textbf{a single} multi-head Transformer can give the near-optimal predictions over a large class of data distributions, \textbf{even without \emph{a priori} knowledge of the data distribution}. 

Figure \ref{f:multihead} provides empirical verification of Proposition \ref{c:multihead}: We plot the loss against number of layers for three kinds of Transformers: 1-head with $\th^{linear}$ activation, 1-head with $\th^{exp}$ activation, 2-head with $\th^{linear}$ activation on the first head and $\th^{linear}$ on the second head. Data labels are drawn from a $\K^{\diamond}$ Gaussian Process, where $\K^{\diamond}(u,v) := \alpha \K^{linear}(G_1 u, G_1 v) + (1-\alpha) \K^{exp}(G_2 u, G_2 v)$. We observe the following
\begin{enumerate}
    \item In Figure \ref{f:multihead_a} and \ref{f:multihead_a}, we see that the 2-head Transformer \textbf{can perform optimally on both $\K^{linear}$ and $\K^{\exp}$ data}. Specifically: In Figure \ref{f:multihead_a}, $\K^{\diamond}=\K^{linear}$ ($\alpha=1$, $G_1=G_2=I$). The 2-head Transformer performs as well as the $\th^{linear}$ Transformer. In Figure \ref{f:multihead_b}, $\K^{\diamond}=\K^{exp}$ ($\alpha=0$, $G_1=G_2=I$). The 2-head Transformer performs as well as the $\th^{exp}$ Transformer.
    \item In Figure \ref{f:multihead_c}, $\K^{\diamond}(u,v) := \frac{1}{2} (u_1 v_1 + u_2 v_2) + \frac{1}{2} \exp\lrp{\frac{1}{2} (u_3 v_3 + u_4 v_4 + u_5 v_5)}$, corresponding to choosing $\alpha = 1/2$, $G_1 = \diag(\lrb{1,1,0,0,0})$ and $G_2 = \diag(\lrb{0,0,1,1,1})$. For this choice of $\K^{\diamond}$, the 2-head Transformer \textbf{outperforms both single-head Transformers}.
\end{enumerate}
Note: the Transformer parameters \textbf{re-trained for each dataset}, so attention weights for \ref{f:multihead_a}, \ref{f:multihead_b} and \ref{f:multihead_c} are \textbf{different}.


\begin{figure}[H]
\centering
\begin{subfigure}{0.32\textwidth}
\centering
\includegraphics[width=\textwidth]{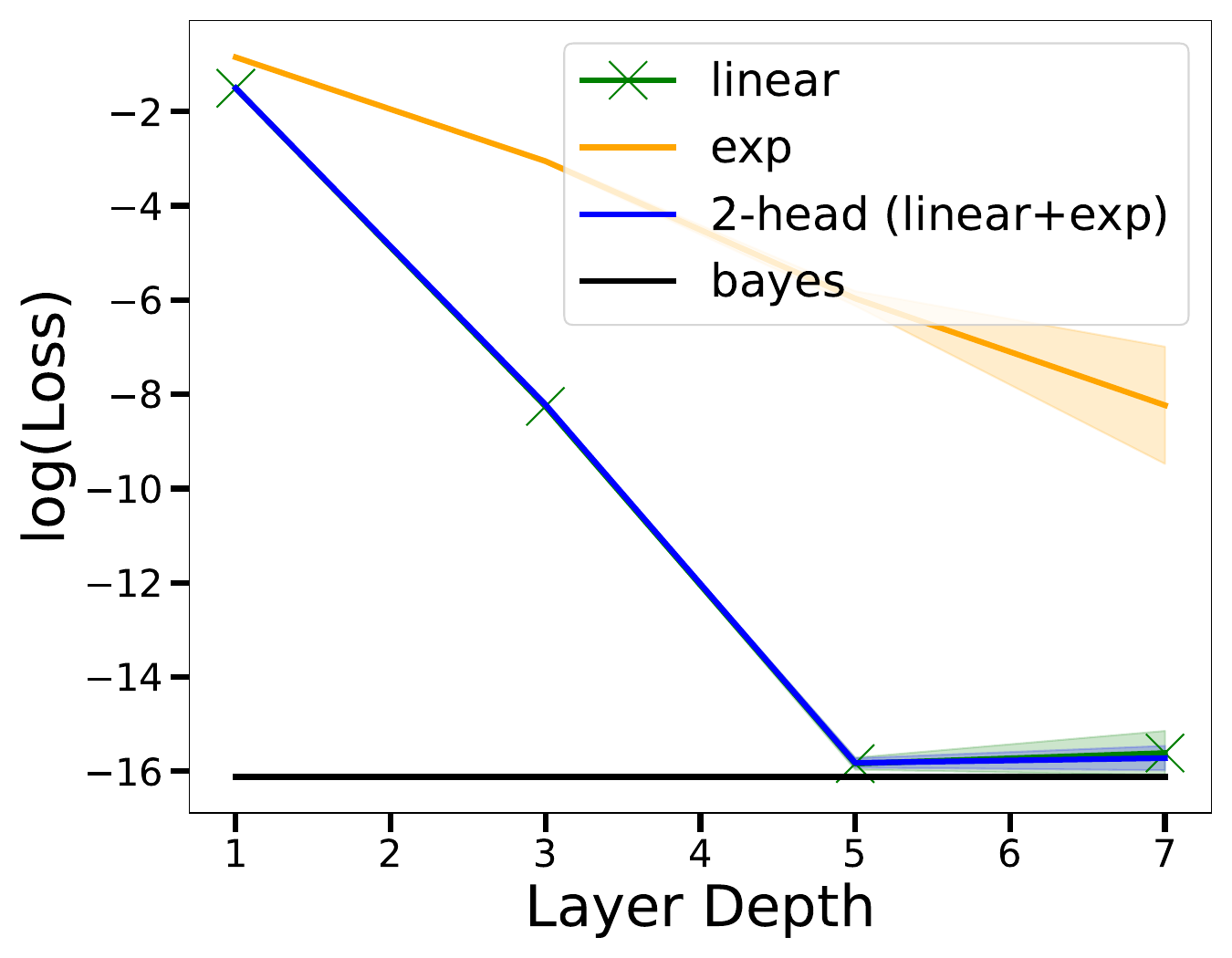} 
\caption{$\K^{linear}$}
\label{f:multihead_a}
\end{subfigure}\hfill
\begin{subfigure}{0.32\textwidth}
\centering
\includegraphics[width=\textwidth]{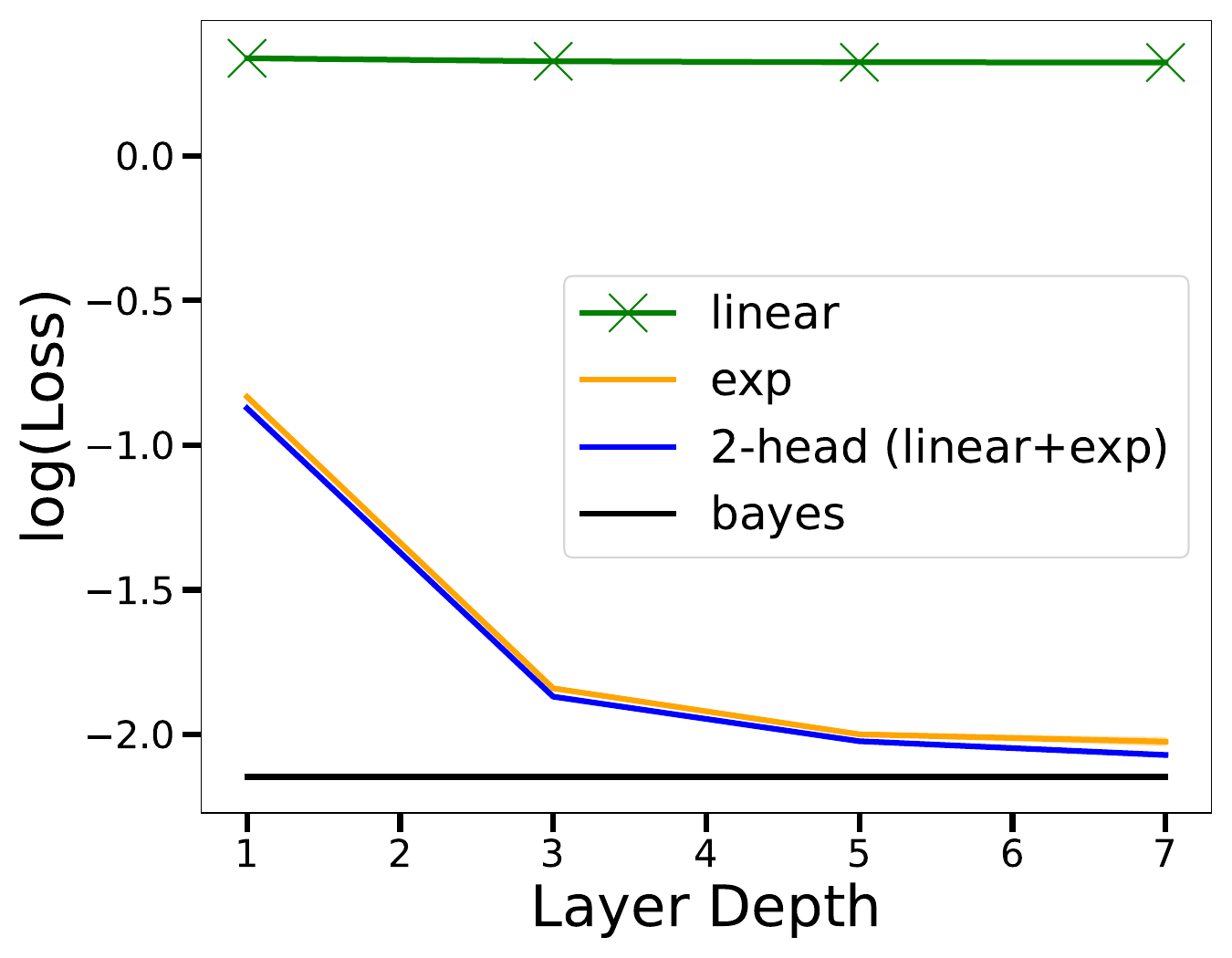}  
\caption{$\K^{exp}$}
\label{f:multihead_b}
\end{subfigure}
\begin{subfigure}{0.32\textwidth}
\centering
\includegraphics[width=\textwidth]{updated_figures/comb_datadist_loss_vs_layer_comb.pdf}  
\caption{$\K^{\diamond}$}
\label{f:multihead_c}
\end{subfigure}
\caption{Plot of log(test ICL loss) against number of layers. Each sub-figure samples data from a different distribution ($\K^{\diamond}(u,v) := \alpha \K^{linear}(G_1 u, G_1 v) + (1-\alpha) \K^{exp}(G_2 u, G_2 v)$). We compare the performance of three kinds of Transformers. The labels are generated using a $\K^{\diamond}$ Gaussian Process. Context length $n=14$.}
\label{f:multihead}
\end{figure}

\section{Optimization Landscape Results}
\label{s:landscape}
In the previous section, we saw that Transformers \emph{can} implement functional gradient descent in its forward pass, and that this implementation can be nearly statistical optimal. However, \emph{does the Transformer learn to implement functional gradient descent when training converges?} To answer this question, we analyze the optimization landscape of the in-context loss, for the Transformer defined in \eqref{e:dynamics_Z}. 

\textbf{\underline{In Theorem \ref{t:informal_master_sparse}}}, we show that the functional gradient descent construction of Proposition \ref{p:rkhs_descent_transformer_construction} is a stationary point of the in-context loss when we constrain the top left block of the Value matrix to $0$. 

\textbf{\underline{In Theorem \ref{t:informal_master_full}}}, we characterize stationary points of the in-context loss for general Value matrices. The stationary point implements a sophisticated algorithm that interleaves functional gradient descent steps with transformations of the covariates. 

\textbf{{We provide experimental verification}} of both Theorem \ref{t:informal_master_sparse} and \ref{t:informal_master_full} in Sections \ref{ss:experiments_for_sparse} and \ref{ss:experiments_for_full}. We present key assumptions in Sections \ref{ss:distributional_assumptions} and \ref{ss:architectural_assumptions}. We note that both Theorem \ref{t:informal_master_sparse} and \ref{t:informal_master_full} apply to softmax and ReLU Transformers. 

\subsection{Distributional Assumptions}
\label{ss:distributional_assumptions}
We will first state two assumptions on the distribution of covariates $X$ and labels $Y|X$. We motivate these assumptions with Examples \ref{ex:distribution_x_iid_rotational_invariant}-\ref{ex:2_layer_relu}.
\label{ss:distributional_assumptions}. 

Recall the setup from Section \ref{s:define_icl}. The input is $Z_0\in \R^{(d+1) \times (n+1)}$. Let $X = \begin{bmatrix}
\tx{1} & \tx{2} & \cdots & \tx{n} &\tx{n+1}
\end{bmatrix} \in \R^{d\times (n+1)}$ denote the first $d$ rows of $Z_0$. Let $Y=\begin{bmatrix}
\ty{1} & \ty{2} & \cdots & \ty{n} &\ty{n+1}
\end{bmatrix}\in \R^{1\times (n+1)}$ denote the row vector of labels $\ty{i}$'s. Note that the last row of $Z_0$ has $\ty{n+1}$ replaced by $0$, and thus differs from $Y$. We will make an assumption each on the distributions of $X$ and $Y$ respectively:
\begin{assumption}[$X$ distribution assumption]
    \label{ass:x_distribution}
    Let $\P_X$ denote the distribution of $X$, i.e. $\P_X$ is the joint distribution over $\tx{1}\ldots \tx{n+1}$. Furthermore, assume that there is a symmetric invertible matrix $\Sigma \in \R^{d\times d}$ such that for any orthogonal matrix $U$, $\Sigma^{1/2} U \Sigma^{-1/2} X \overset{d}{=} X$.
\end{assumption}

In Examples \ref{ex:distribution_x_iid_rotational_invariant} and \ref{ex:distribution_x_mixture} below, we provide two common distributions for $\tx{i}$ which satisfy Assumption \ref{ass:x_distribution}.

\begin{tcolorbox}[enhanced,title=,
                    frame hidden,
                    colback=gray!5,
                    breakable,
                    left=1pt,
                    right=1pt,
                    top=1pt,
                    bottom=1pt,
                ]
\begin{example}[$\tx{i}$ drawn from rotationally invariant distributions]
    \label{ex:distribution_x_iid_rotational_invariant}
    Assumption \ref{ass:x_distribution} is satisfied when $\tx{i} \overset{iid}{\sim} \N(0,\Sigma)$, or when $\tx{i} = \Sigma^{1/2} \xi^{(i)}$, for $\xi$ drawn uniformly from the unit sphere. This distribution of $\tx{i}$ has been considered in \cite{garg2022can, akyurek2022learning,von2022transformers,ahn2023transformers,zhang2023trained,mahankali2023one}.
\end{example}
\end{tcolorbox}

\begin{tcolorbox}[enhanced,title=,
                    frame hidden,
                    colback=gray!5,
                    breakable,
                    left=1pt,
                    right=1pt,
                    top=1pt,
                    bottom=1pt,
                ]
\begin{example}[$\tx{i}$ drawn from Gaussian Mixture Models]
\label{ex:distribution_x_mixture}
    More generally, Assumption \ref{ass:x_distribution} can be satisfied even when $\tx{i}$ are not iid. Let $\mu \sim \N(0,I)$, and let $\tx{i} = \mu + \xi^{(i)}$, where $\xi^{(i)}\overset{iid}{\sim} \N(0,I)$. This example can be further generalized to contain two or more cluster means $\mu_1, \mu_2$ sampled independently (i.e. mixture of Gaussians).
\end{example}
\end{tcolorbox}

\begin{assumption}[$Y|X$ distribution assumption]
    \label{ass:y_distribution}
    Conditional on $X=\lrb{\tx{1}\ldots \tx{n+1}}$, $Y = \lrb{\ty{1}\ldots \ty{n+1}} \in \R^{(n+1)}$ has covariance matrix $\E_{Y|X}\lrb{Y^\top Y}=: \Kmat(X)$, where $\Kmat(X): \R^{d\times (n+1)} \to \R^{(n+1)\times(n+1)}$. Assume that for all orthogonal matrix $U\in \R^{d\times d}$, $\Kmat(\Sigma^{1/2} U \Sigma^{-1/2}X) = \Kmat(X)$, where $\Sigma$ is the same matrix from Assumption \ref{ass:x_distribution}.
\end{assumption}

In Examples \ref{ex:linear_y_example}, \ref{ex:gaussian_process_kernel_examples}, \ref{ex:2_layer_relu} below, we will discuss a few common label distributions which satisfy Assumptions \ref{ass:y_distribution}. \textbf{Example \ref{ex:gaussian_process_kernel_examples} is of particular interest}, as it is quite general, and is the setting for all the experiments presented in Figures \ref{f:h_match_k}, \ref{f:h_match_k_against_layer}, \ref{f:theorem_sparse}, \ref{f:theorem_full}. Note that Example \ref{ex:linear_y_example} is a special case of Example \ref{ex:gaussian_process_kernel_examples}.

\begin{tcolorbox}[enhanced,title=,
                    frame hidden,
                    colback=gray!5,
                    breakable,
                    left=1pt,
                    right=1pt,
                    top=1pt,
                    bottom=1pt,
                ]
\begin{example}[$\ty{i}$ are linear functions of $\tx{i}$]
    \label{ex:linear_y_example}
    One example of Assumption \ref{ass:y_distribution} is when $\theta \sim \N(0,I)$, $\ty{i} = \lin{\theta, \xi^{(i)}}$, and $\tx{i} = \Sigma^{1/2} \xi^{(i)}$. We can verify that the covariance matrix $\Kmat(X_0) := \E\lrb{Y^\top Y} = X^\top \Sigma^{-1/2}\E\lrb{\theta \theta^T} \Sigma^{-1/2} X= X^\top \Sigma^{-1} X = \Kmat(\Sigma^{1/2} U \Sigma^{-1/2} X)$. This setting was considered in \cite{ahn2023transformers,mahankali2023one}.
\end{example}
\end{tcolorbox}

\begin{tcolorbox}[enhanced,title=,
                    frame hidden,
                    colback=gray!5,
                    breakable,
                    left=1pt,
                    right=1pt,
                    top=1pt,
                    bottom=1pt,
                ]
\begin{example}[Rotationally Symmetric $\K$ Gaussian Process]
    \label{ex:gaussian_process_kernel_examples}
    Recall the $\K$ Gaussian Process from Definition \ref{d:k_gaussian_process}. Under this definition, recall that $Y^\top|X \sim \N(0, \Kmat_+(X))$, where $\lrb{\Kmat(X)}_{ij} := \K\lrp{\Sigma^{-1/2} \tx{i}, \Sigma^{-1/2} \tx{j}}$, and $\Kmat_+(X)$ takes an absolute value on the eigenvalues of $\Kmat(X)$. We verify that Assumption \ref{ass:y_distribution} holds when $\K$ satisfies, for all orthogonal matrix $U$, if, for all orthogonal matrix $U$, $\K(v,w) = \K(U v, U w).$ This is satisfied by the following common choices of $\K$:
    \begin{align*}
        & \K^{linear}(u,w) := \lin{u,w}\qquad \K^{relu}(u,w) := \relu(\lin{u,w})\qquad  \K^{exp}_\sigma(u,w) := \exp\lrp{\lin{x,y}/\sigma^2}. 
        \numberthis \label{e:examples_of_generating_kernel}
    \end{align*}
    Notice that Example \ref{ex:linear_y_example} is a special case of Example \ref{ex:gaussian_process_kernel_examples} with $\Kmat$ generated by $\K^{linear}$.    
\end{example}
\end{tcolorbox}

\begin{tcolorbox}[enhanced,title=,
                    frame hidden,
                    colback=gray!5,
                    breakable,
                    left=1pt,
                    right=1pt,
                    top=1pt,
                    bottom=1pt,
                ]
\begin{example}[Two-layer ReLU network.]
    \label{ex:2_layer_relu}
    Finally, we provide an example where Assumption \ref{ass:y_distribution} holds for a \emph{random kernel}. Consider the random two-layer ReLU classification function described in \cite{garg2022can}:
    \begin{align*}
        \ty{i} = \lin{\theta_2, \relu\lrp{\theta_1 \tx{i}}},
    \end{align*}
    where $\theta_1\in \R^{d\times m}, \theta_2\in \R^{m}$. $\theta_1,\theta_2$ can be seen as representing a 2-layer ReLU network, with $m$ being the dimension of the hidden layer. Let $\theta_1, \theta_2$ be sampled coordinate-wise from $\N(0,1)$, independently. We verify that
    \begin{align*}
        \lrb{\Kmat(X)}_{ij}:= \E\lrb{\ty{i}\ty{j}} 
        =& \E_{\theta_1,\theta_2}\lrb{\relu\lrp{\theta_1 \tx{i}}^\top \theta_2 \theta_2^\top \relu\lrp{\theta_1 \tx{j}}} = \E_{\theta_1}\lrb{\relu\lrp{\theta_1 \tx{i}}^\top  \relu\lrp{\theta_1 \tx{j}}}.
    \end{align*}
    Similarly, using the fact that $\theta_1 U \overset{d}{=}\theta_1$, we verify that \\
    $\lrb{\Kmat(U X)}_{ij} - \E_{\theta_1}\lrb{\relu\lrp{\theta_1 U \tx{i}}^\top  \relu\lrp{\theta_1 U \tx{j}}} = \E_{\theta_1}\lrb{\relu\lrp{\theta_1 \tx{i}}^\top  \relu\lrp{\theta_1 \tx{j}}}$.

\end{example}
\end{tcolorbox}





\subsection{Architectural Assumptions}
\label{ss:architectural_assumptions}
For the rest of this section, we will assume that $\th\lrp{U,V}$ satisfies the following invariance:
\begin{assumption}
    \label{ass:th}
    For any $W, V \in \R^{d\times(n+1)}$ and for any matrix $S \in \R^{d\times d}$ with inverse $S^{-1}$, the function $\th(\cdot, \cdot)$ satisfies $\th(W, V) = \th(S^\top W, S^{-1} V)$.
\end{assumption}
We verify that the three examples of $\th$ from Examples \{\ref{ex:linear}, \ref{ex:relu}, \ref{ex:softmax}\} which implement \{Linear, ReLU, Softmax\}-activated Transformers, all satisfy Assumption \ref{ass:th}. We also assume that $V_\ell$ has the following sparsity pattern for $\ell =0\ldots k$:

\begin{assumption}
\label{ass:full_attention}
For $\ell =0\ldots k$, the value matrices $V_\ell$ which parameterize the Transformer layers in \eqref{e:dynamics_Z} satisfy the following structure:
$
V_\ell = \begin{bmatrix}
A_\ell & 0 \\ 
0 & r_\ell
\end{bmatrix}\quad \text{for some $A_i\in \R^{d\times d}$, $r_i \in \R$.}
$
\end{assumption}
The same sparsity pattern was considered in \cite{ahn2023transformers} in studying multi-layer linear Transformers.

\subsection{Theorem \ref{t:informal_master_sparse}: Functional gradient descent is a stationary point of (constrained) in-context loss.}
\label{ss:informal_theorem_sparse}
We first study the stationary points of the optimization problem, under the constraint that $A_\ell=0$ in Assumption \ref{ass:full_attention}. This setting is interesting because of its connection to the \emph{functional gradient descent} construction in Proposition \ref{p:rkhs_descent_transformer_construction}.
\begin{theorem}[Informal Statement of Theorem \ref{t:master_sparse}]
    \label{t:informal_master_sparse}
    Let $\th$ satisfy Assumption \ref{ass:th}, Let $\lrp{\tx{i},\ty{i}}_{i=1\ldots n+1}$ have distributions satisfying Assumptions \ref{ass:x_distribution} and \ref{ass:y_distribution}. Consider the optimization problem $\min_{V,B,C} f(V,B,C)$, for the in-context loss $f$ defined in \eqref{d:ICL_loss}, under the constraint that $V=\lrbb{V_\ell}_{\ell=0\ldots k}$ satisfies Assumption \ref{ass:full_attention}. \textcolor{blue}{\textbf{Additionally constrain $A_\ell=0$ for $\ell=0\ldots k$.}} Then there exist stationary points of the constrained optimization problem where, for all $\ell=0\ldots k$,
    \begin{align*}
        B_\ell = b_\ell \Sigma^{-1/2} \qquad C_\ell = c_\ell \Sigma^{-1/2},
        \numberthis \label{e:t:informal_sparse}
    \end{align*}
    where $b_\ell, c_\ell \in \R$.
\end{theorem}
The formal version of Theorem \ref{t:informal_master_sparse} is stated as Theorem \ref{t:master_sparse} in Appendix \ref{s:main_theorem_sparse}; its proof is in Appendix \ref{ss:proof:t:master_sparse}.This proposed stationary point implements the functional gradient descent construction of Proposition \ref{p:rkhs_descent_transformer_construction} -- when $\Sigma = I$, we verify that \eqref{e:t:informal_sparse} is, \textbf{up to scaling, identical to the construction in Proposition \ref{p:rkhs_descent_transformer_construction}.}

More generally, when $\Sigma$ is not identity, but $[\th\lrp{U,W}]_{ij}=\K(U^{(i)},W^{(j)})$ for some kernel $\K$ (see Examples \ref{ex:linear}, \ref{ex:relu}), \eqref{e:t:informal_sparse} implement functional descent with respect to the RKHS induced by $\tilde{\K}\lrp{u,w} := \K\lrp{\Sigma^{-1/2} u, \Sigma^{-1/2} w}$. One can view the kernel $\tilde{\K}$ as a rescaled version of $\K$.

Finally, in the case when $\th$ \textbf{does not coincide with a kernel}, \eqref{e:t:informal_sparse} implements the following algorithm:
\begin{align*}
\numberthis \label{e:sparse_general_alg}
    & f_{\ell+1}(\tx{n+1})  
    = f_\ell(\tx{n+1}) + r_\ell' \sum_{i=1}^n \lrp{\ty{i} - f_\ell(\tx{i})} \lrb{\th\lrp{\Sigma^{-1/2} X_0, \Sigma^{-1/2} X_0}}_{i,(n+1)}
\end{align*}
where $f_{\ell+1}(\tx{n+1})$ is ``the Transformer's prediction for $\ty{n+1}$ at layer $\ell$", for $\ell = 0...k+1$. It is instructive to compare \eqref{e:sparse_general_alg} with \eqref{e:t:oaimda:2}. 
\subsection{Experiment for Theorem \ref{t:informal_master_sparse}}
\label{ss:experiments_for_sparse}
In Figure \ref{f:theorem_sparse} below, we present empirical verification of Theorem \ref{t:informal_master_sparse}. In addition to the setup in Appendix \ref{ss:common_experiment_details}, we additionally constrain $A_\ell=0$ for each layer $\ell$. The number of demonstrations $n=30$.

To verify that the parameters are indeed converging to the predicted stationary point in Theorem \ref{t:informal_master_sparse}, we plot $\dist\lrp{\Sigma^{1/2} B_i^\top C_i \Sigma^{1/2},I}$, for $i=0,1,2$. The \emph{normalized Frobenius norm distance}: $\dist(M,I) := \min_{\alpha}  \frac{\lrn{M - \alpha  \cdot I}}{\lrn{M}_F}$, (equivalent to choosing $\alpha := \frac{1}{d} \sum_{i=1}^d M[i,i]$). This is essentially the projection distance of $\nicefrac{M}{\lrn{M}}_F$ onto the space of scaled identity matrices. 

{We only verify $B^\top C$ because the network is overparameterized, and for any $\Lambda\in \R^{d\times d}$, $(B_\ell, C_\ell)$ gives identical prediction as $(\Lambda^\top B_\ell, \Lambda^{-1} C_\ell)$}. (See also Remark \ref{r:overparameterized} after Theorem \ref{t:master_sparse}). It appears that in most cases, the matrices are converging to identity, which is the stationary point in Theorem \ref{t:informal_master_full}. This demonstrates that Theorem \ref{t:master_sparse} holds across a broad combination of $\K$ and $\th$.

We note that in the case of Figure \ref{f:theorem_sparse_c} and \ref{f:theorem_sparse_h}, a few of the parameter matrices appear to asymptote at around 0.2 distance to identity. It is unclear if this is due to optimization difficulties, or due to convergence to stationary points different from that proposed in Theorem \ref{t:informal_master_sparse}.

\begin{figure}[H]
\centering
\begin{subfigure}{0.32\textwidth}
\centering
\includegraphics[width=\textwidth]{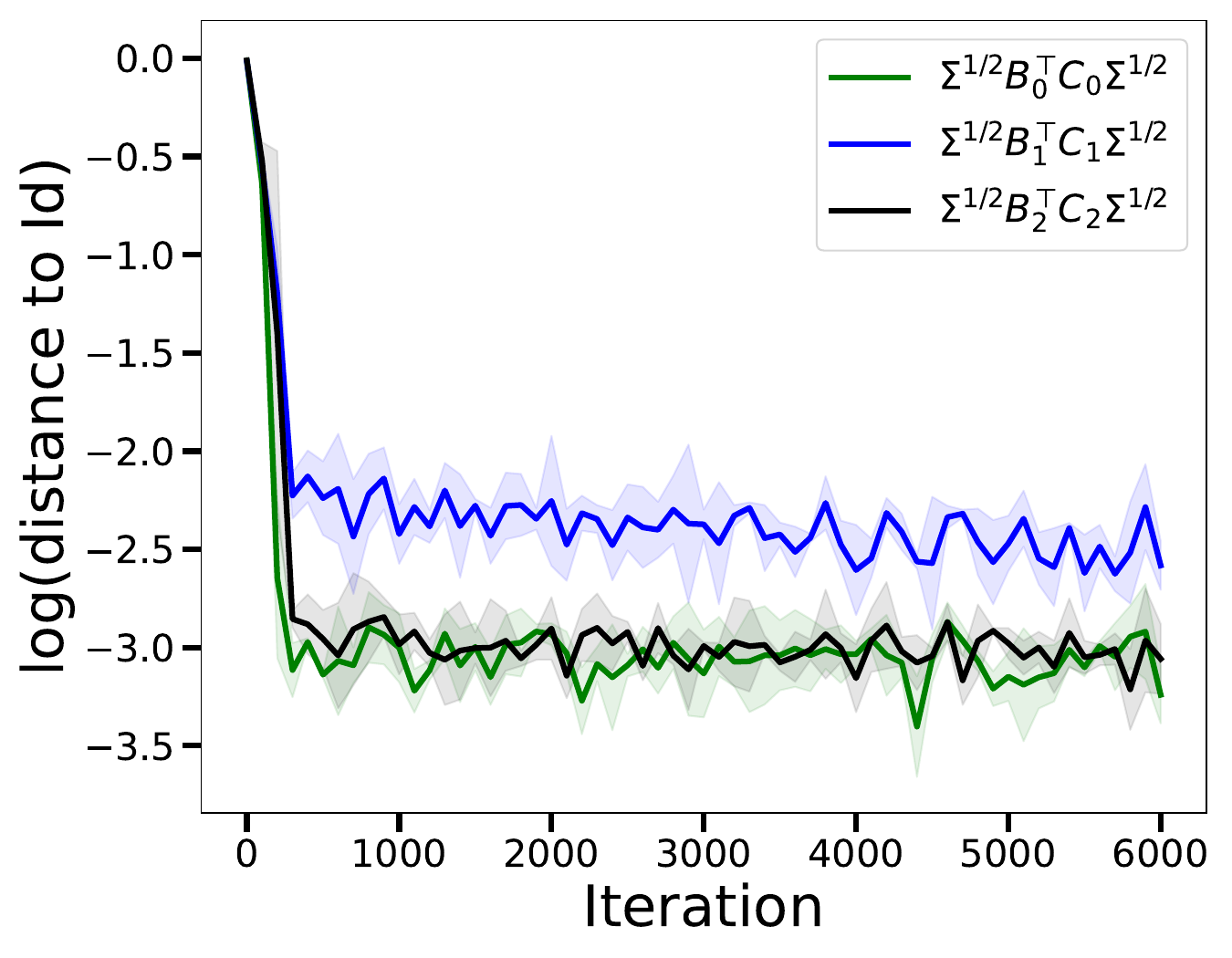} 
\caption{(Linear, ReLU)}
\end{subfigure}\hfill
\begin{subfigure}{0.32\textwidth}
\centering
\includegraphics[width=\textwidth]{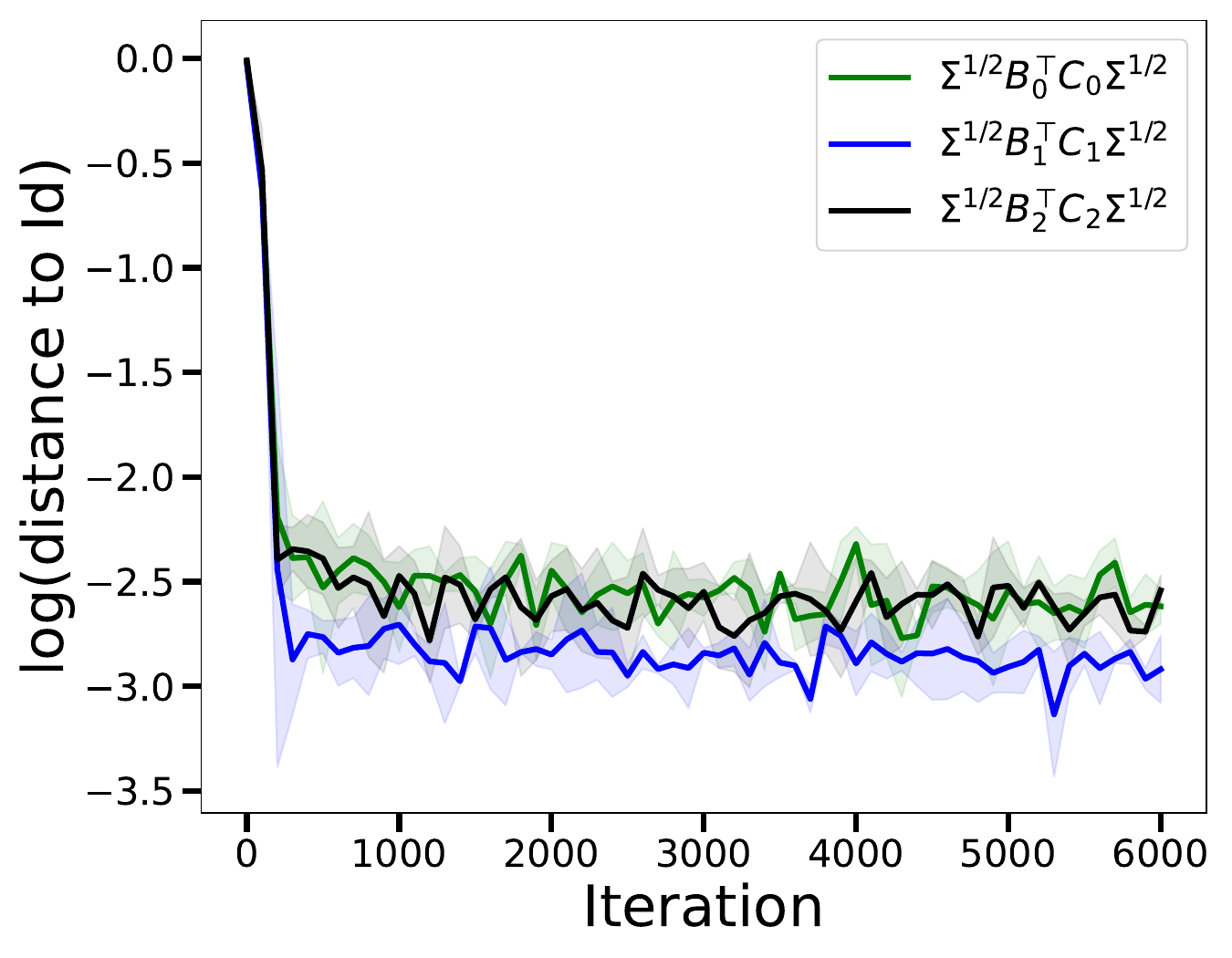}  
\caption{(ReLU, ReLU)}
\end{subfigure}
\begin{subfigure}{0.32\textwidth}
\centering
\includegraphics[width=\textwidth]{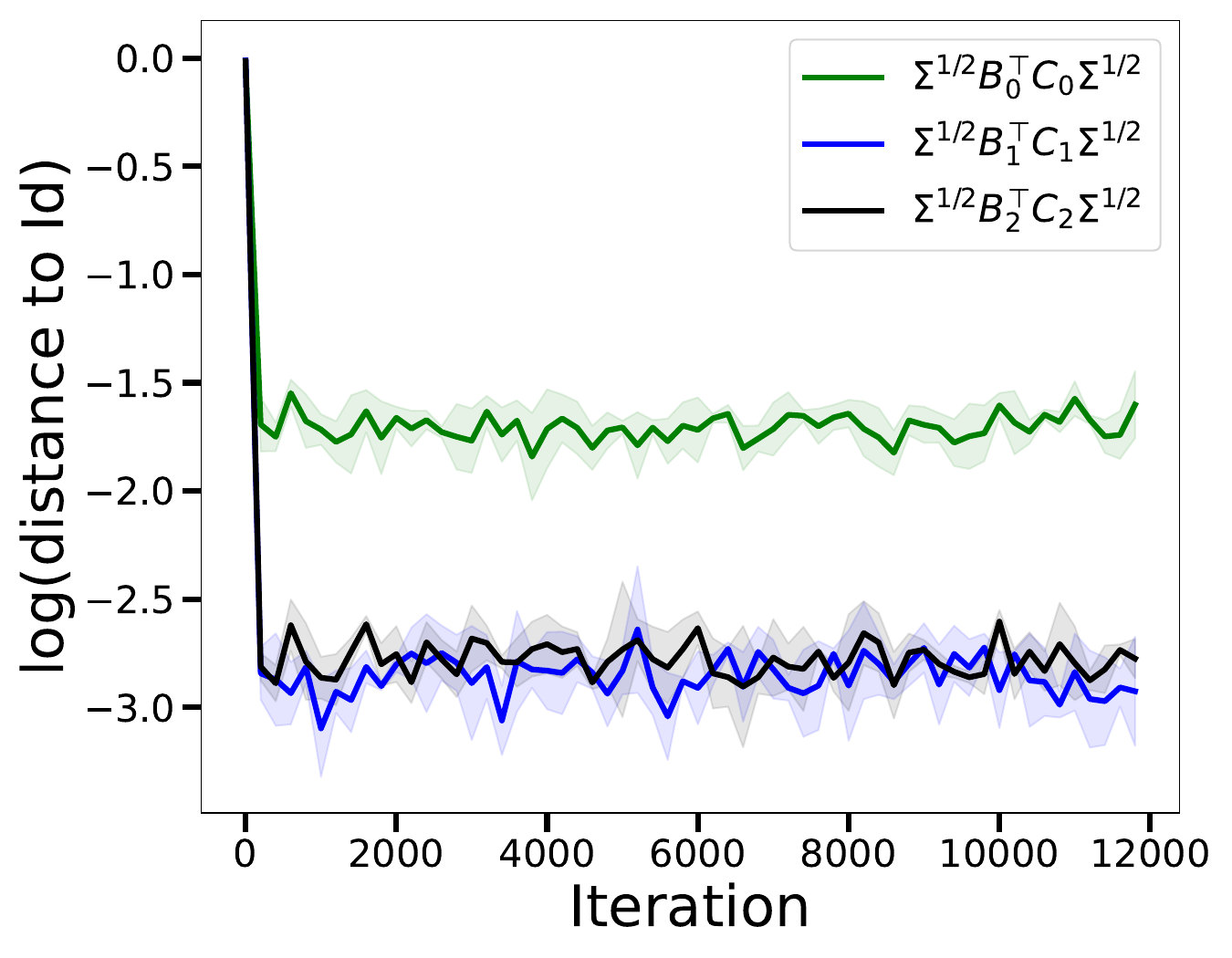}  
\caption{(Exp, ReLU)}
\label{f:theorem_sparse_c}
\end{subfigure}\\
\begin{subfigure}{0.32\textwidth}
\centering
\includegraphics[width=\textwidth]{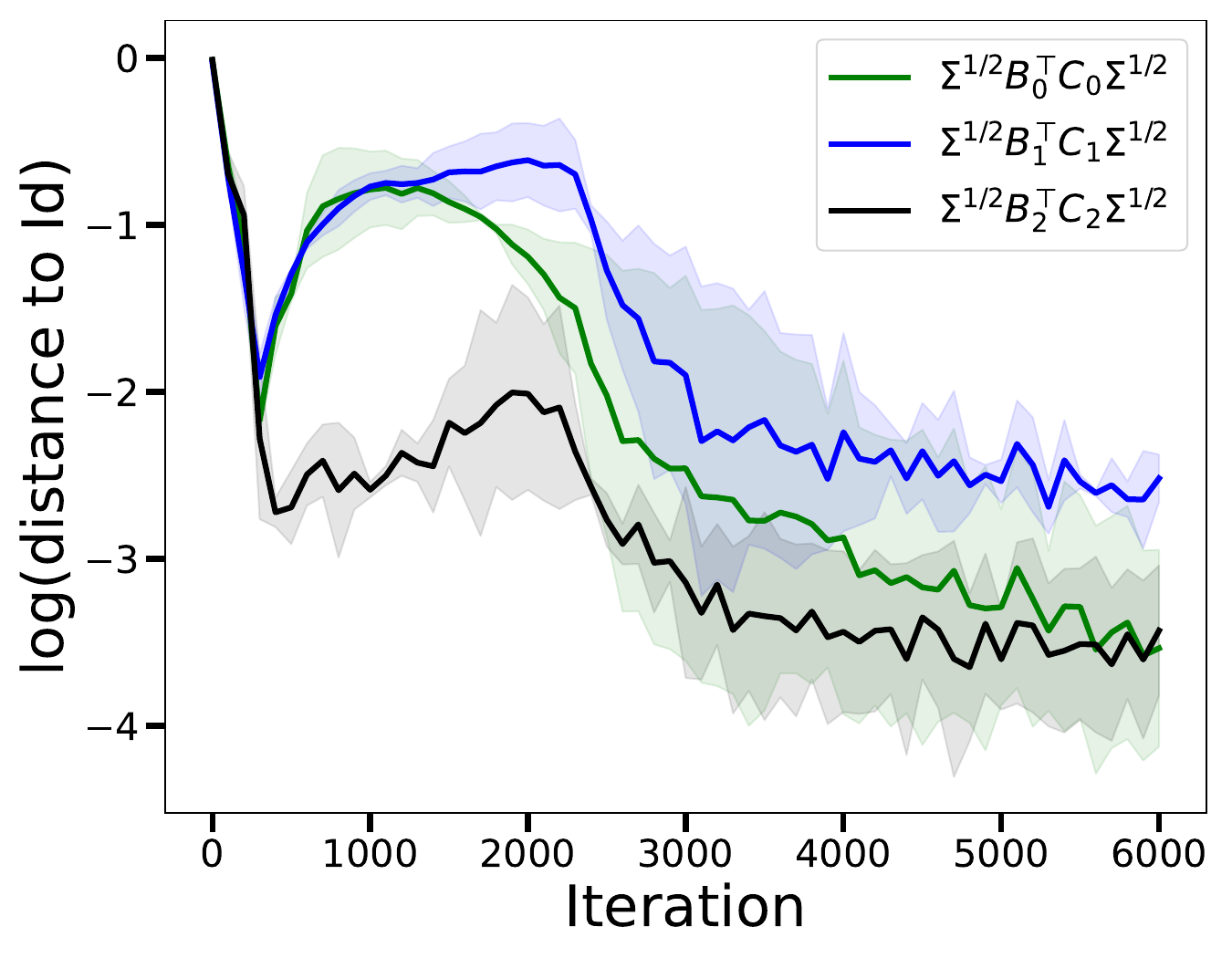} 
\caption{(Linear, Softmax)}
\end{subfigure}\hfill
\begin{subfigure}{0.32\textwidth}
\centering
\includegraphics[width=\textwidth]{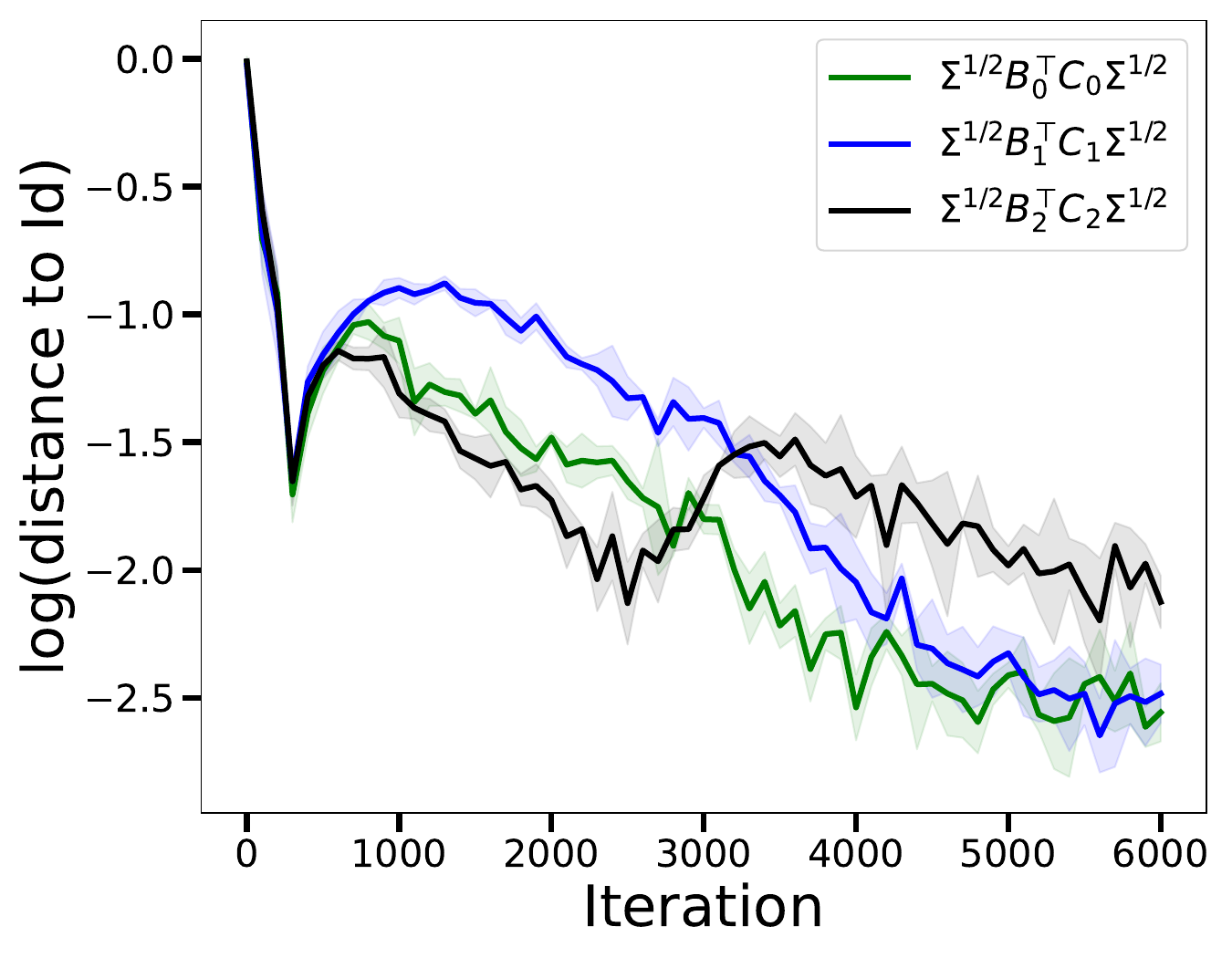}  
\caption{(ReLU, Softmax)}
\end{subfigure}
\begin{subfigure}{0.32\textwidth}
\centering
\includegraphics[width=\textwidth]{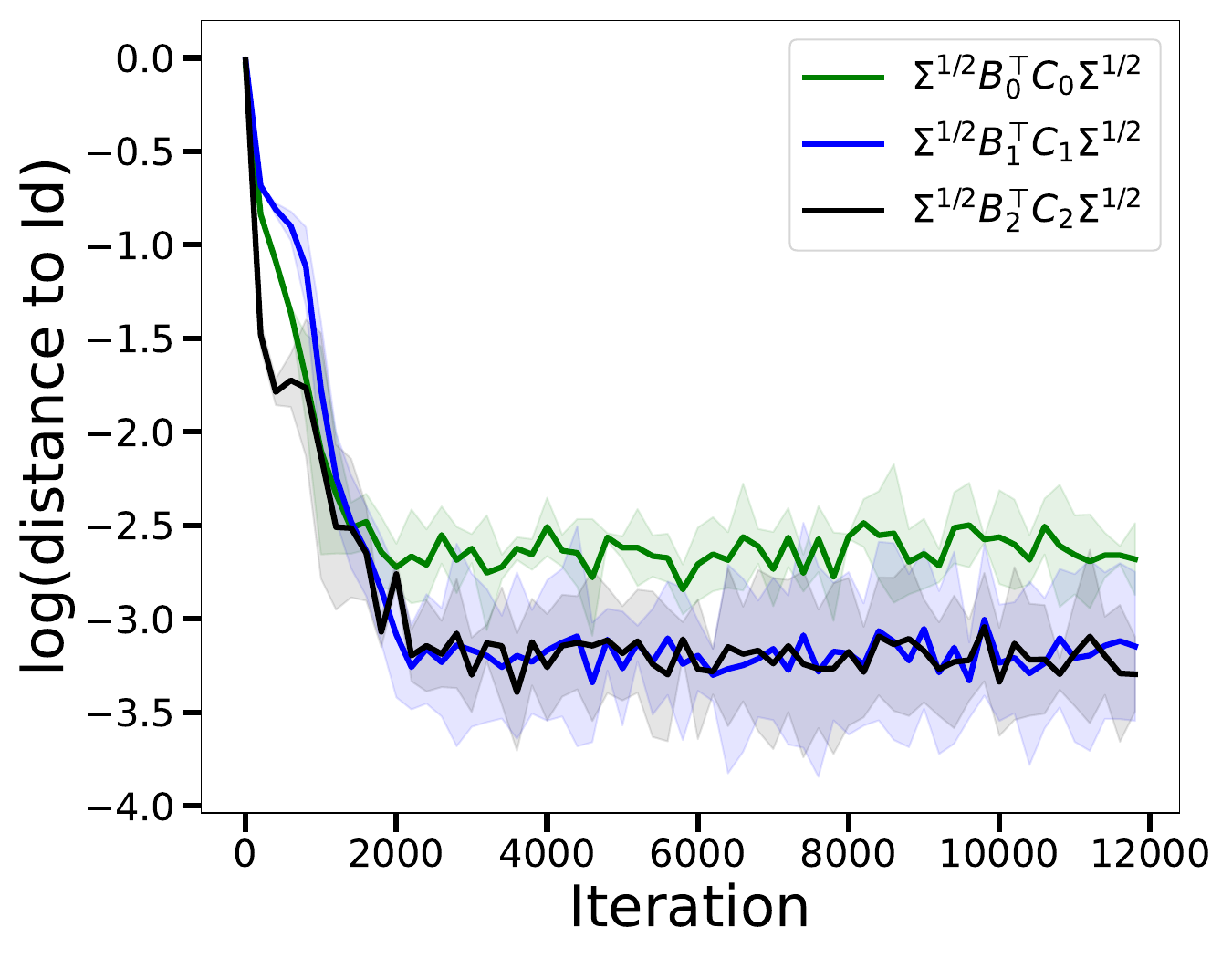}  
\caption{(Exp, Softmax)}
\end{subfigure}\\
\begin{subfigure}{0.32\textwidth}
\centering
\includegraphics[width=\textwidth]{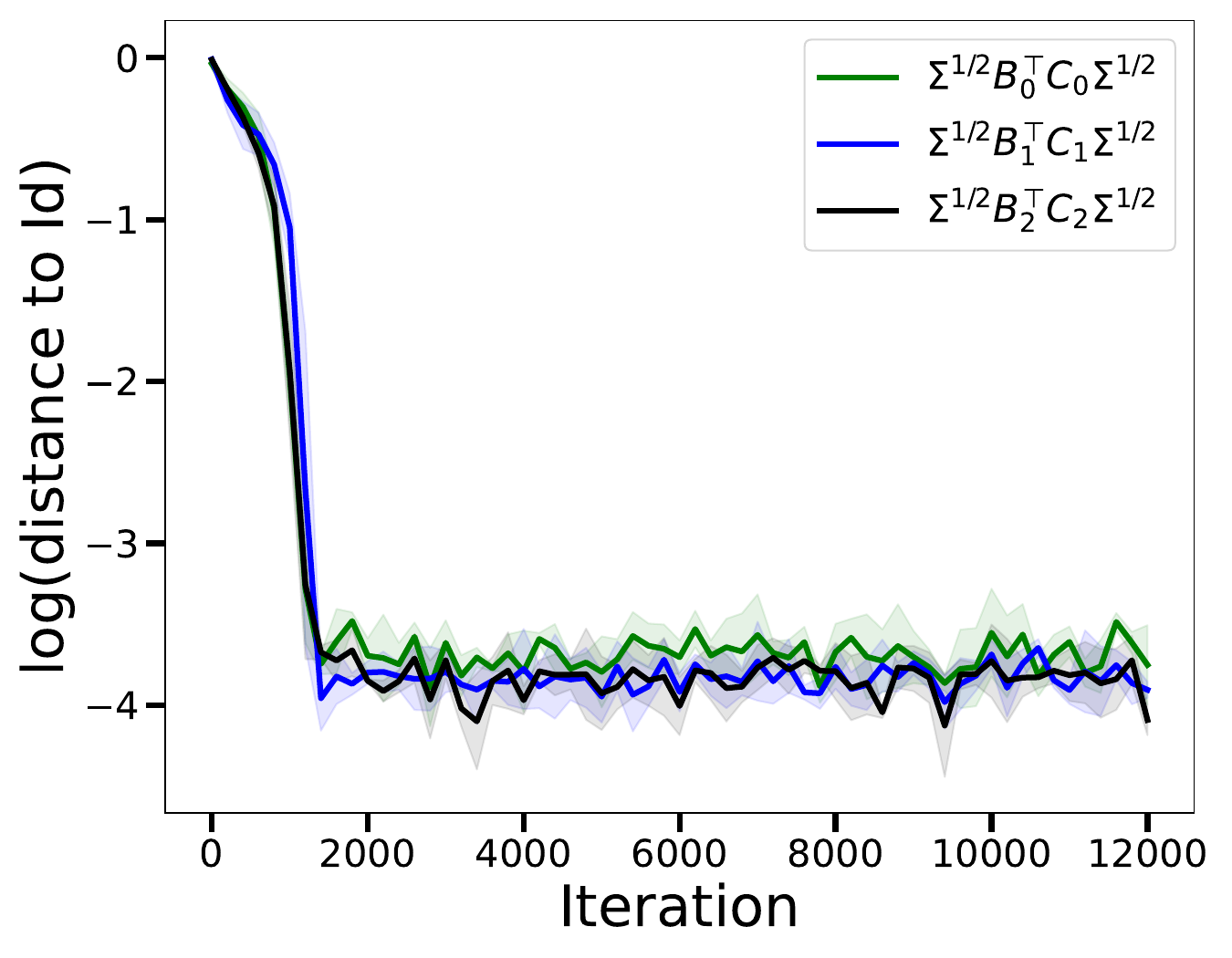} 
\caption{(Linear, Linear)}
\end{subfigure}\hfill
\begin{subfigure}{0.32\textwidth}
\centering
\includegraphics[width=\textwidth]{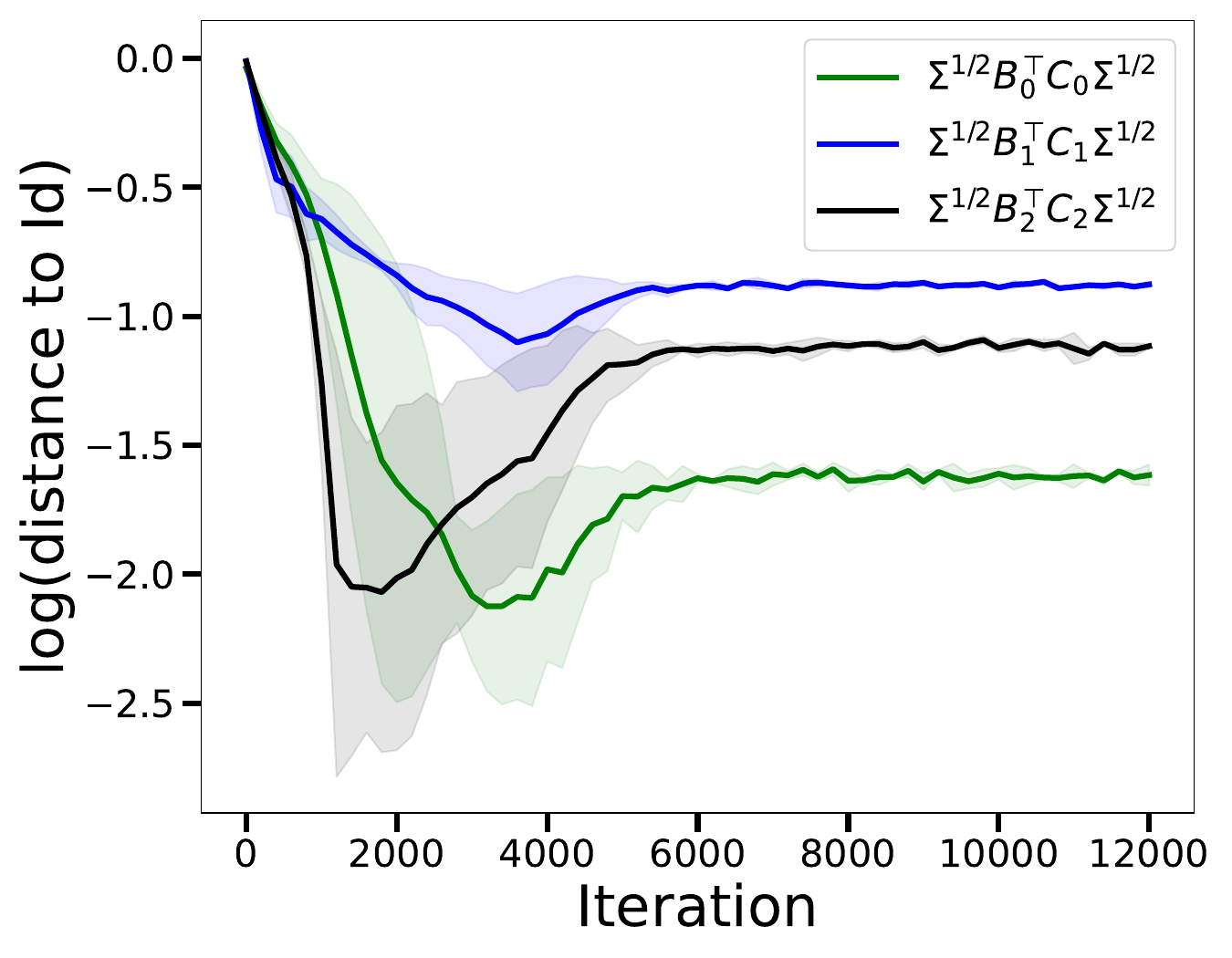}  
\caption{(ReLU, Linear)}
\label{f:theorem_sparse_h}
\end{subfigure}
\begin{subfigure}{0.32\textwidth}
\centering
\includegraphics[width=\textwidth]{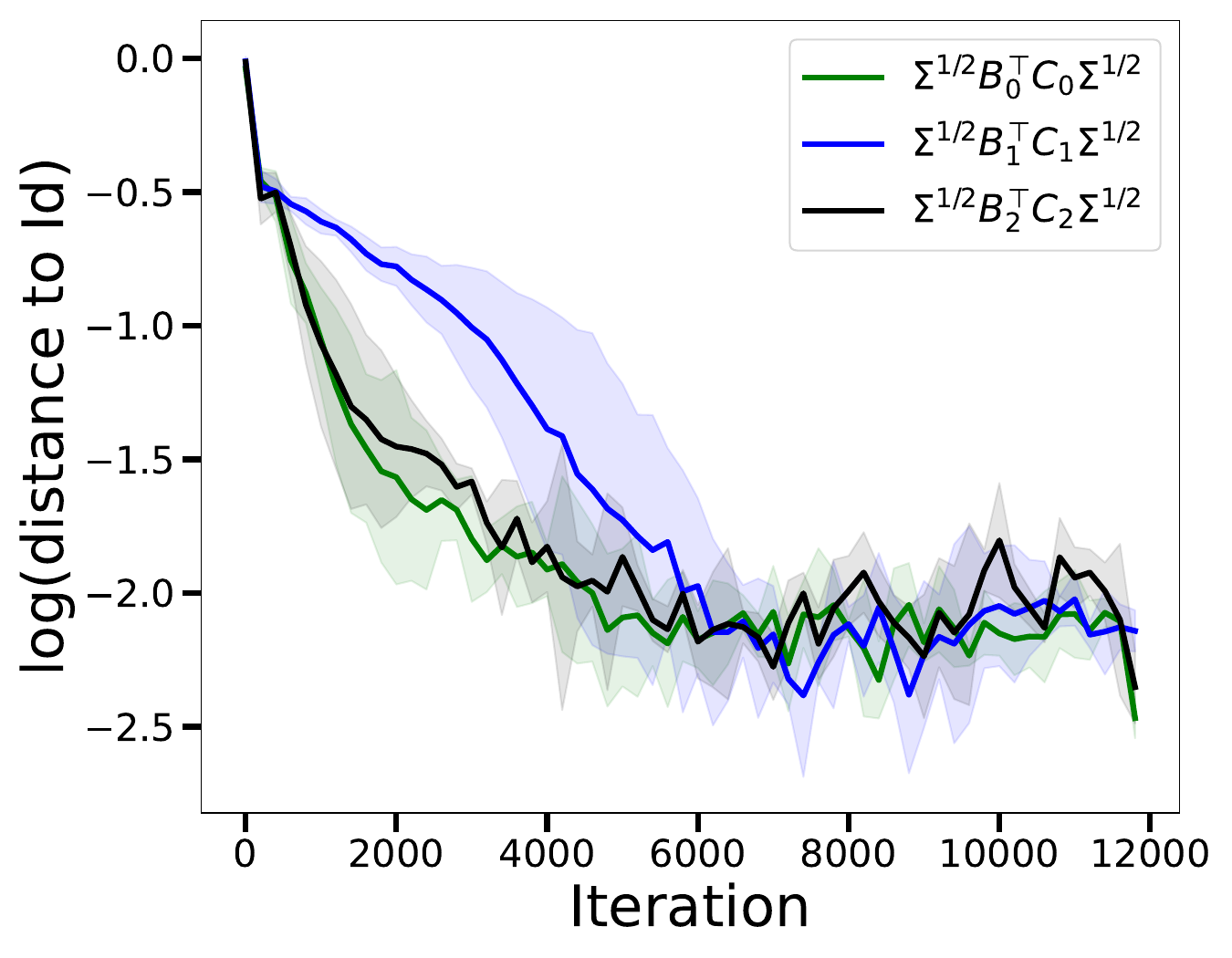}  
\caption{(Exp, Linear)}
\end{subfigure}

\caption{Plots of $\log(\dist(M,I))$ for $M=\Sigma^{1/2}\lrbb{B_0^\top C_0, B_1^\top C_1, B_2^\top C_2}\Sigma^{1/2}$ against number of training iterations. Each plot coincides with a different experiment setup, where we vary the generating distribution and the architecture. The subplot title is ($\K$, $\tilde{h}$), where $\K$ defines a Gaussian Process for labels, as described in Definition \ref{d:k_gaussian_process}, and $\th$ is the non-linear map in the Transformer's attention module. In all cases, the corresponding matrix appears to be converging to identity, which is the stationary point from Theorem \ref{t:informal_master_sparse}.}
\label{f:theorem_sparse}
\end{figure}

\subsection{Theorem \ref{t:informal_master_full}: Characterizing the stationary points of unconstrained in-context loss.}
\label{ss:informal_theorem_full}
We now study stationary points of the optimization problem under Assumption \ref{ass:full_attention} (with arbitrary $A_\ell$'s). This setting is considerably more general than the setting of Theorem \ref{t:informal_master_sparse}.
\begin{theorem}[Informal Statement of Theorem \ref{t:master_full}]
    \label{t:informal_master_full}
    Let $\th$ satisfy Assumption \ref{ass:th}, Let $\lrp{\tx{i},\ty{i}}_{i=1\ldots n+1}$ have distributions satisfying Assumptions \ref{ass:x_distribution} and \ref{ass:y_distribution}. Consider the optimization problem $\min_{V,B,C} f(V,B,C)$, for the in-context loss $f$ defined in \eqref{d:ICL_loss}, under the constraint that $V=\lrbb{V_\ell}_{\ell=0\ldots k}$ satisfies Assumption \ref{ass:full_attention}. Then there exist stationary points of the constrained optimization problem which equal the functional gradient descent construction from Proposition \ref{p:rkhs_descent_transformer_construction}. In particular, for all $\ell=0\ldots k$,
    \begin{align*}
        A_\ell = a_\ell I, \qquad B_\ell = b_\ell \Sigma^{-1/2}, \qquad C_\ell = c_\ell \Sigma^{-1/2},
        \numberthis \label{e:t:informal_full}
    \end{align*}
    where $a_\ell, b_\ell, c_\ell \in \R$.
\end{theorem}
The above setup is identical to Theorem \ref{e:t:informal_sparse}, \textbf{\textcolor{blue}{except we do not constrain $A_\ell=0$.}}

The formal version of Theorem \ref{t:informal_master_full} is stated as Theorem \ref{t:master_full} in Appendix \ref{s:main_theorem_full}; the proof can be found in Appendix \ref{ss:proof:t:master_full}. Unlike in Theorem \ref{t:informal_master_sparse}, the matrix $A_\ell$ is not $0$; thus the covariates $\tx{\ell}$ are transformed each layer. If $\th$ matches a kernel $\K$, we can verify (using same steps as Proposition \ref{p:rkhs_descent_transformer_construction}) that the Transformer dynamics \ref{e:dynamics_Z} implements an algorithm that interleaves functional gradient descent steps with transformations of the covariates. We refer the readers to \eqref{e:dynamic_XY_full_proof} in the proof of Theorem \ref{t:master_full} for an explicit description of the algorithm implemented by the choice of parameters in \eqref{e:t:informal_full}.

We leave interpretation of this algorithm as future work. We note here that, when $\th$ is linear (see Example \ref{ex:linear}), Theorem \ref{t:informal_master_full} implies Theorem 4 of \cite{ahn2023transformers}, which is similar to the GD++ construction of \cite{von2022transformers}. In this case, each covariate transformation step can be shown to improve the condition number of the empirical least squares loss. 


\subsection{Experiments for Theorem \ref{t:informal_master_full}}
\label{ss:experiments_for_full}
In Figure \ref{f:theorem_full}, we present experimental verification of Theorem \ref{t:informal_master_full}. The experiment setup is as described in Appendix \ref{ss:common_experiment_details}. The number of demonstrations $n=30$. The metric for measuring distance to identity is same as described in Section \ref{ss:experiments_for_sparse}, i.e. $\dist(M,I) := \min_{\alpha}  \frac{\lrn{M - \alpha  \cdot I}}{\lrn{M}_F}$. The experiment setup is very similar to that in Section \ref{ss:experiments_for_sparse}, with the following differences:
\begin{enumerate}
    \item The Transformer is parameterized by $(r_\ell, A_\ell, B_\ell, C_\ell)_{\ell=0,1,2}$. The value matrix $V_\ell$ is parameterized by $A_\ell$ as described in Assumption \ref{ass:full_attention}.
    \item We plot $\Dist\lrp{M, Id}$, for $M \in \lrbb{A_0, A_1} \cup \lrbb{\Sigma^{1/2} B_i^\top C_i \Sigma^{1/2}}_{i=0,1,2}$. In particular, we additionally plot $A_0, A_1$ as they are not constrained to be $0$ in Theorem \ref{t:informal_master_full}.
\end{enumerate}

Similar to Appendix \ref{ss:experiments_for_sparse}, it appears that in most cases the matrices are converging to identity, which is the stationary point in Theorem \ref{t:informal_master_full}. We note that in the case of Figure \ref{f:theorem_full_b}, a few of the parameter matrices appear to asymptote at around 0.2 distance to identity. It is unclear if this is due to optimization difficulties, or due to convergence to stationary points different from that proposed in Theorem \ref{t:informal_master_full}.

\begin{figure}[h]
\centering
\begin{subfigure}{0.32\textwidth}
\centering
\includegraphics[width=\textwidth]{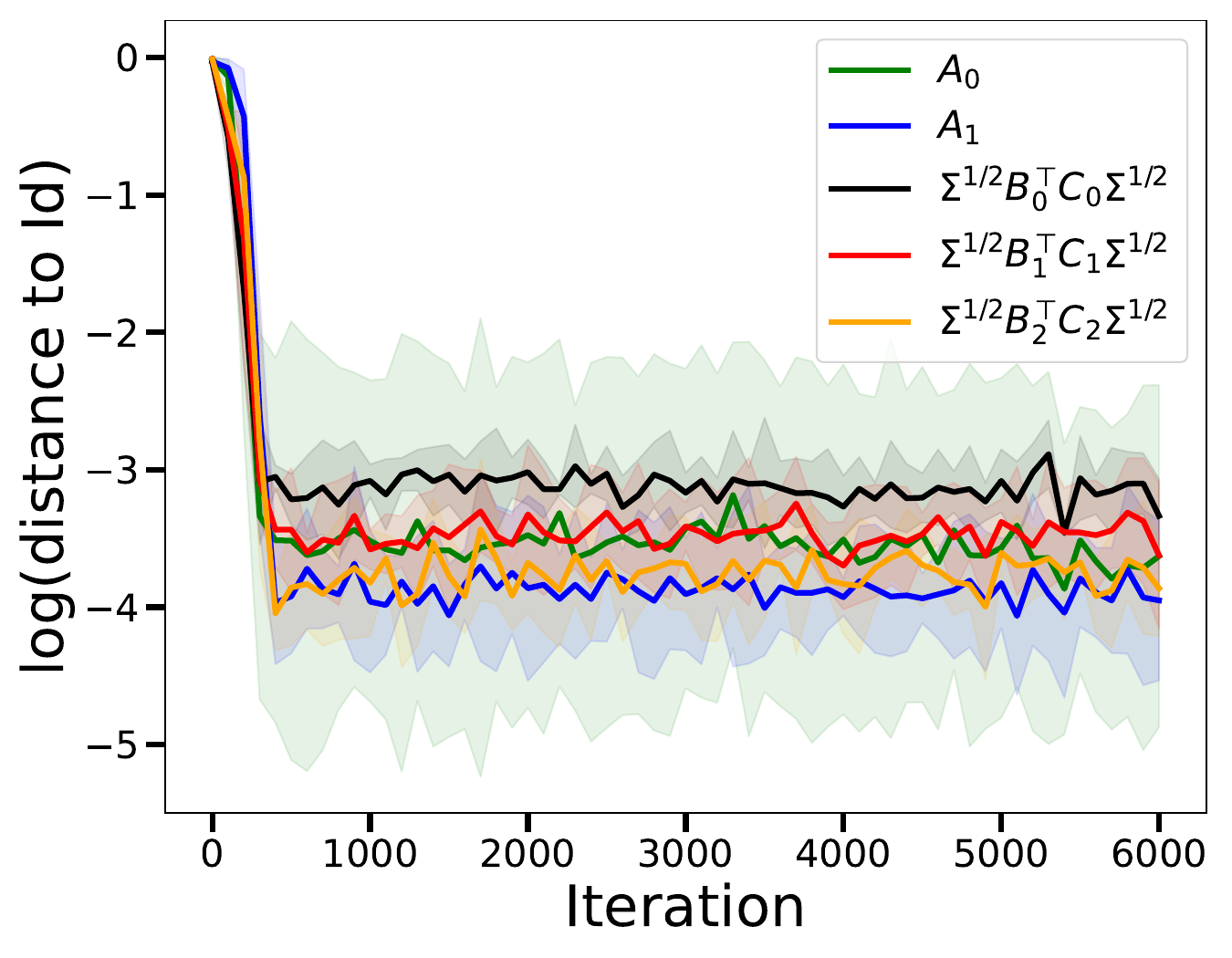} 
\caption{(Linear, ReLU)}
\end{subfigure}\hfill
\begin{subfigure}{0.32\textwidth}
\centering
\includegraphics[width=\textwidth]{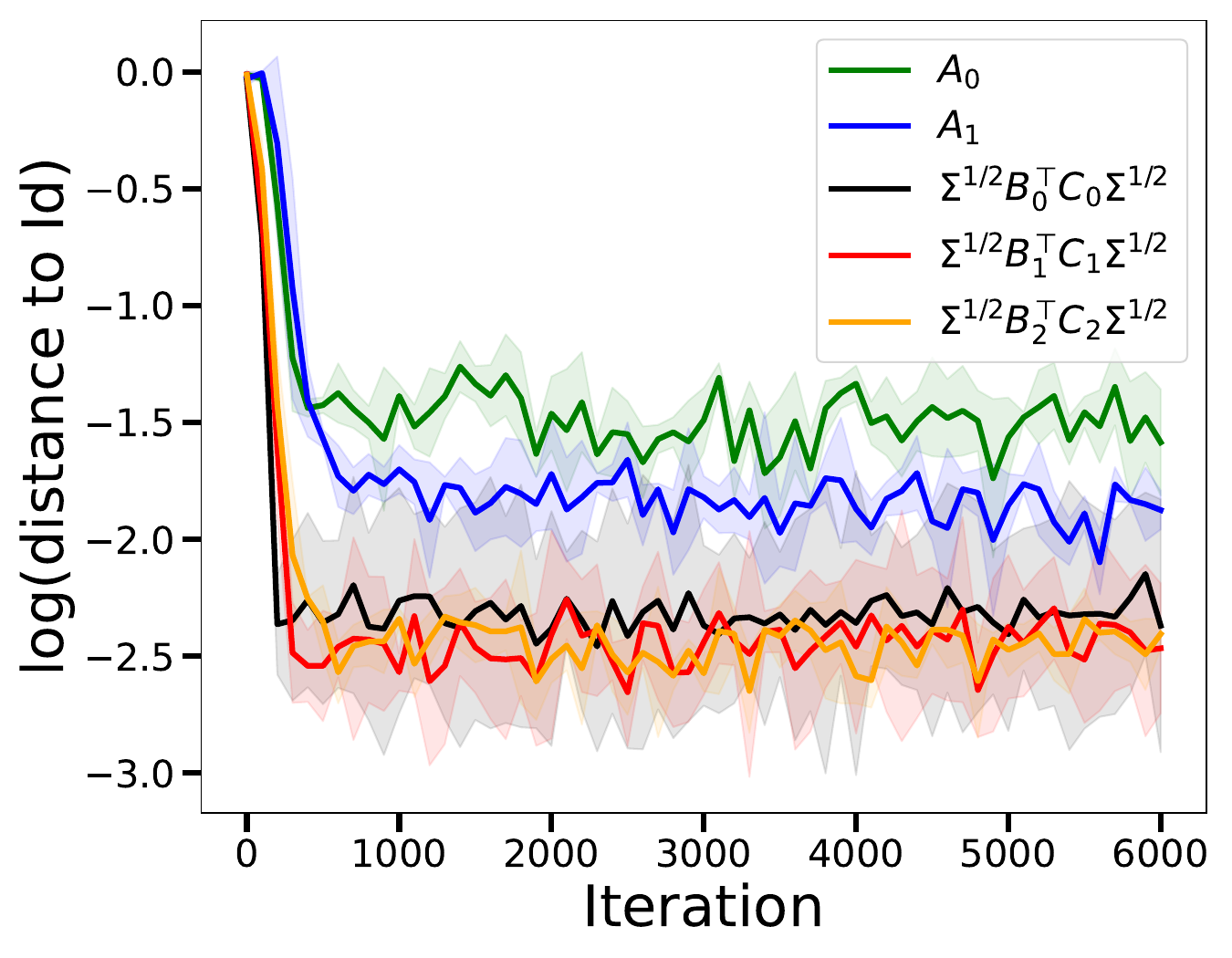}  
\caption{(ReLU, ReLU)}
\label{f:theorem_full_b}
\end{subfigure}
\begin{subfigure}{0.32\textwidth}
\centering
\includegraphics[width=\textwidth]{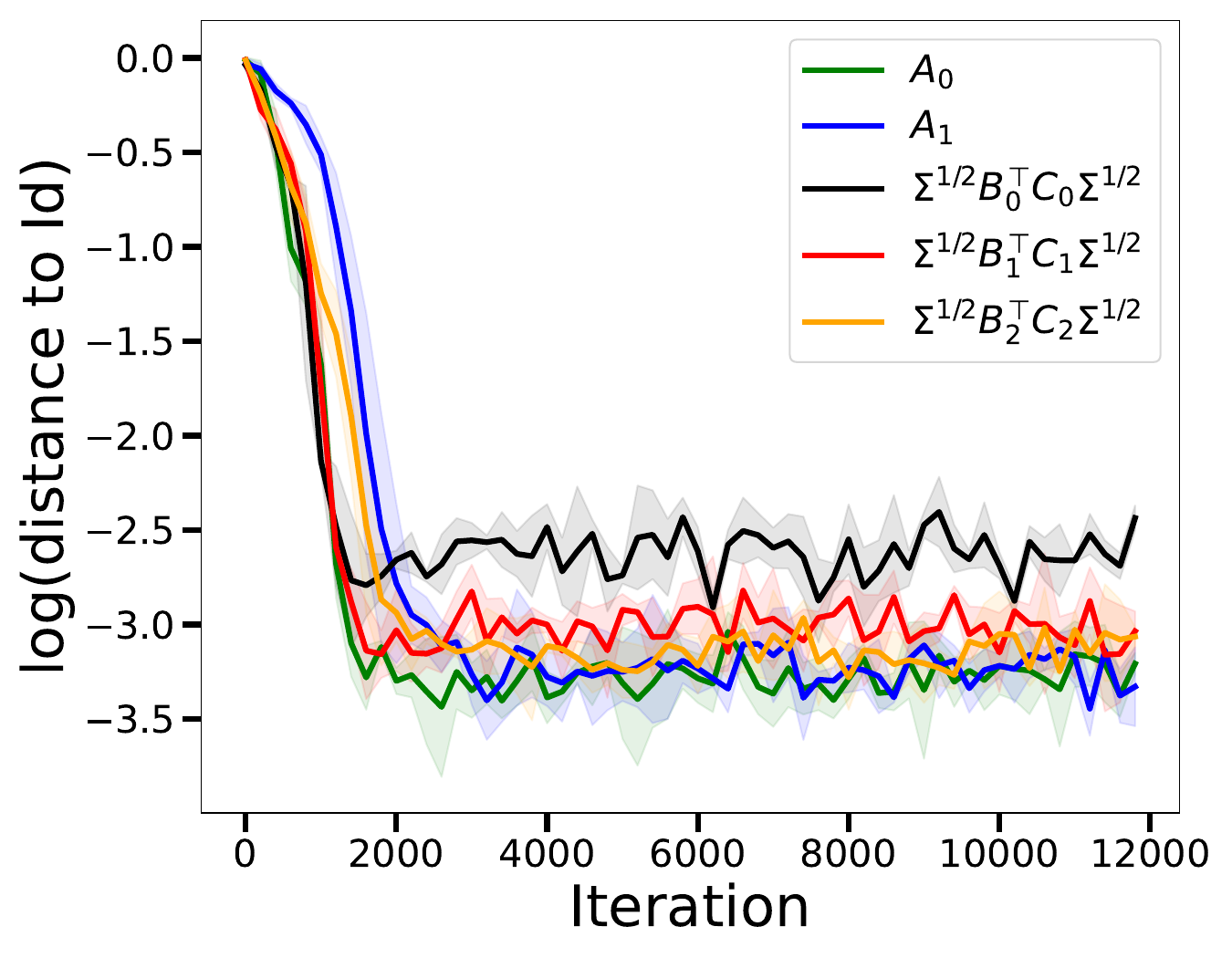}  
\caption{(Exp, ReLU)}
\end{subfigure}\\
\begin{subfigure}{0.32\textwidth}
\centering
\includegraphics[width=\textwidth]{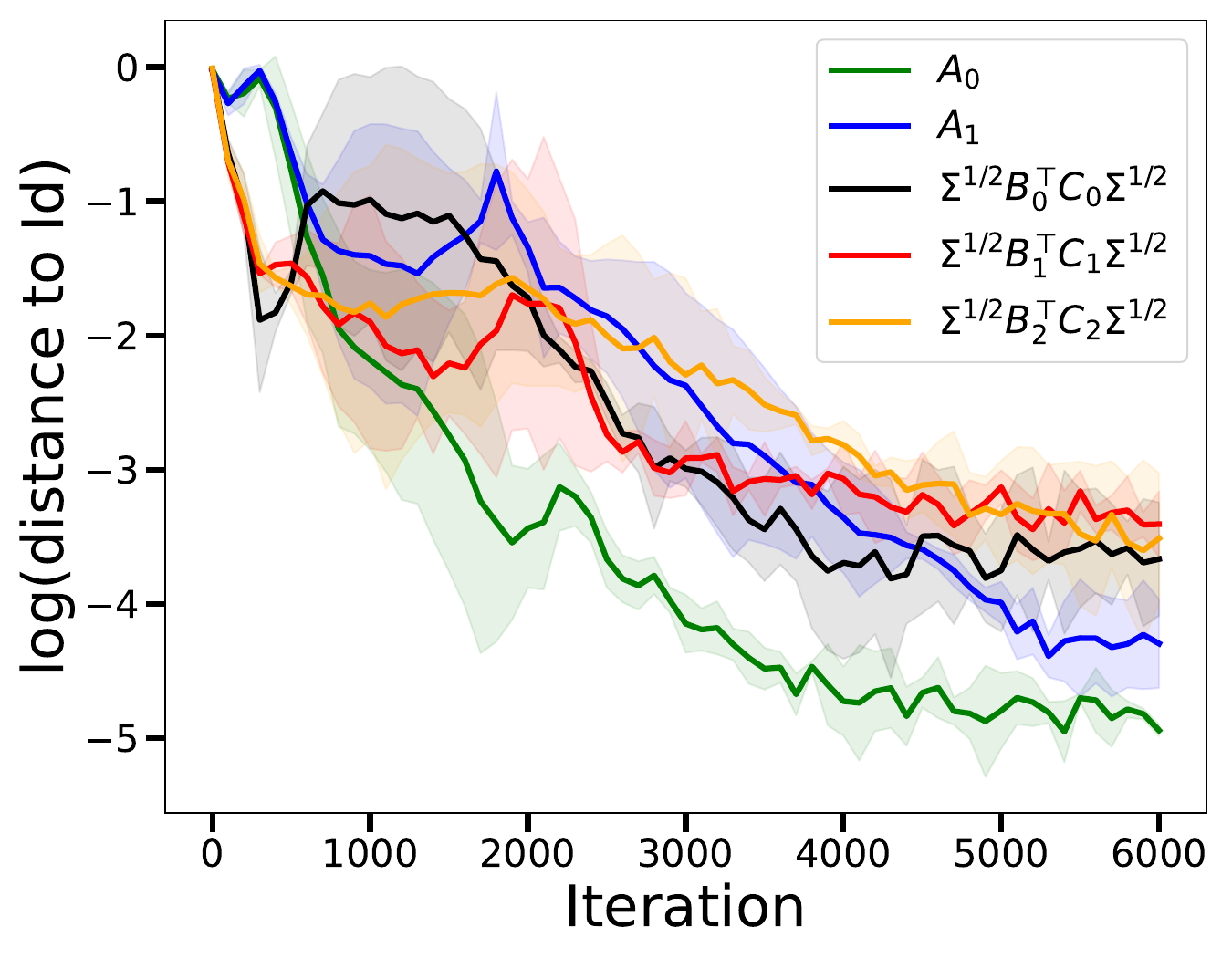} 
\caption{(Linear, Softmax)}
\end{subfigure}\hfill
\begin{subfigure}{0.32\textwidth}
\centering
\includegraphics[width=\textwidth]{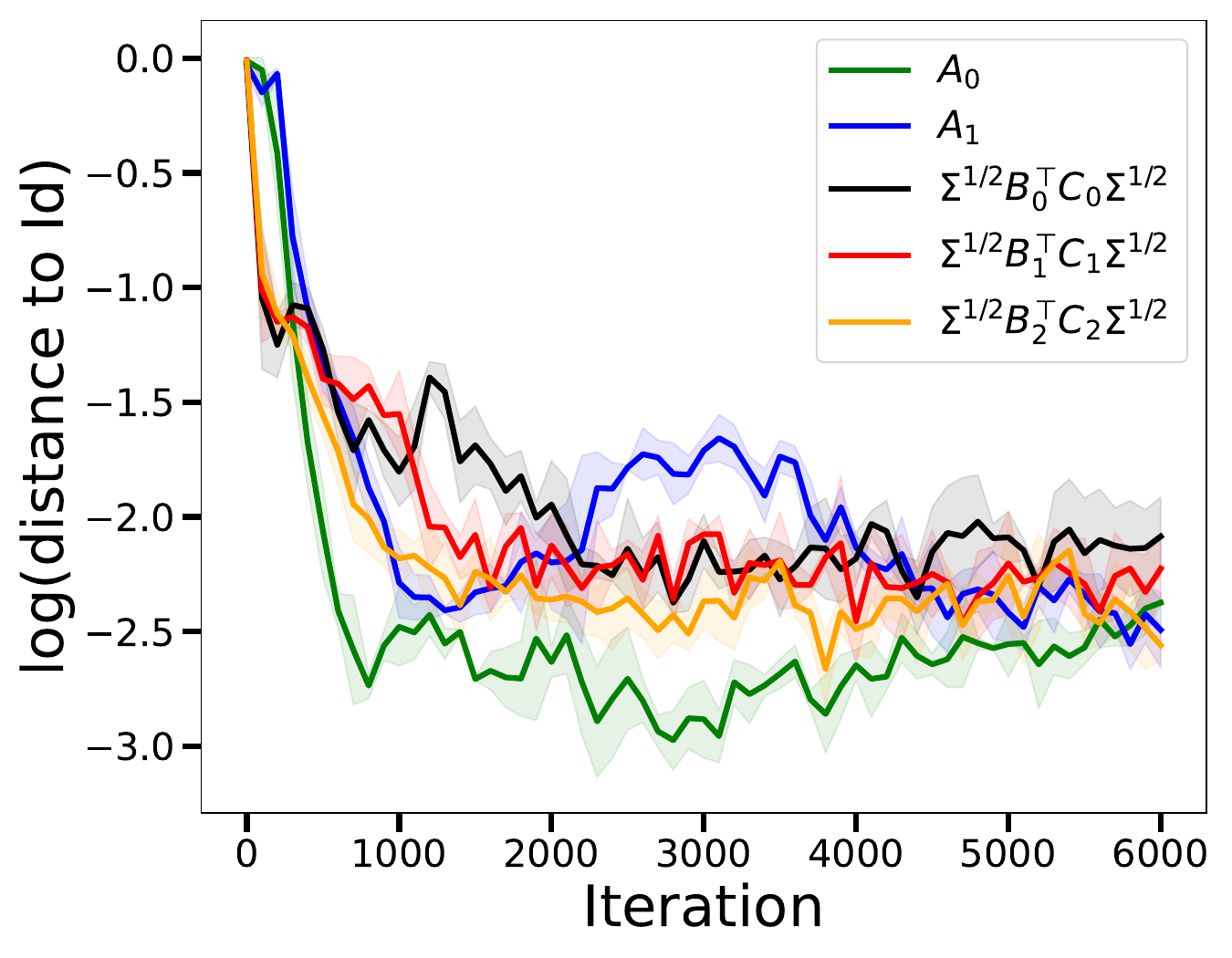}  
\caption{(Exp, Softmax)}
\end{subfigure}
\begin{subfigure}{0.32\textwidth}
\centering
\includegraphics[width=\textwidth]{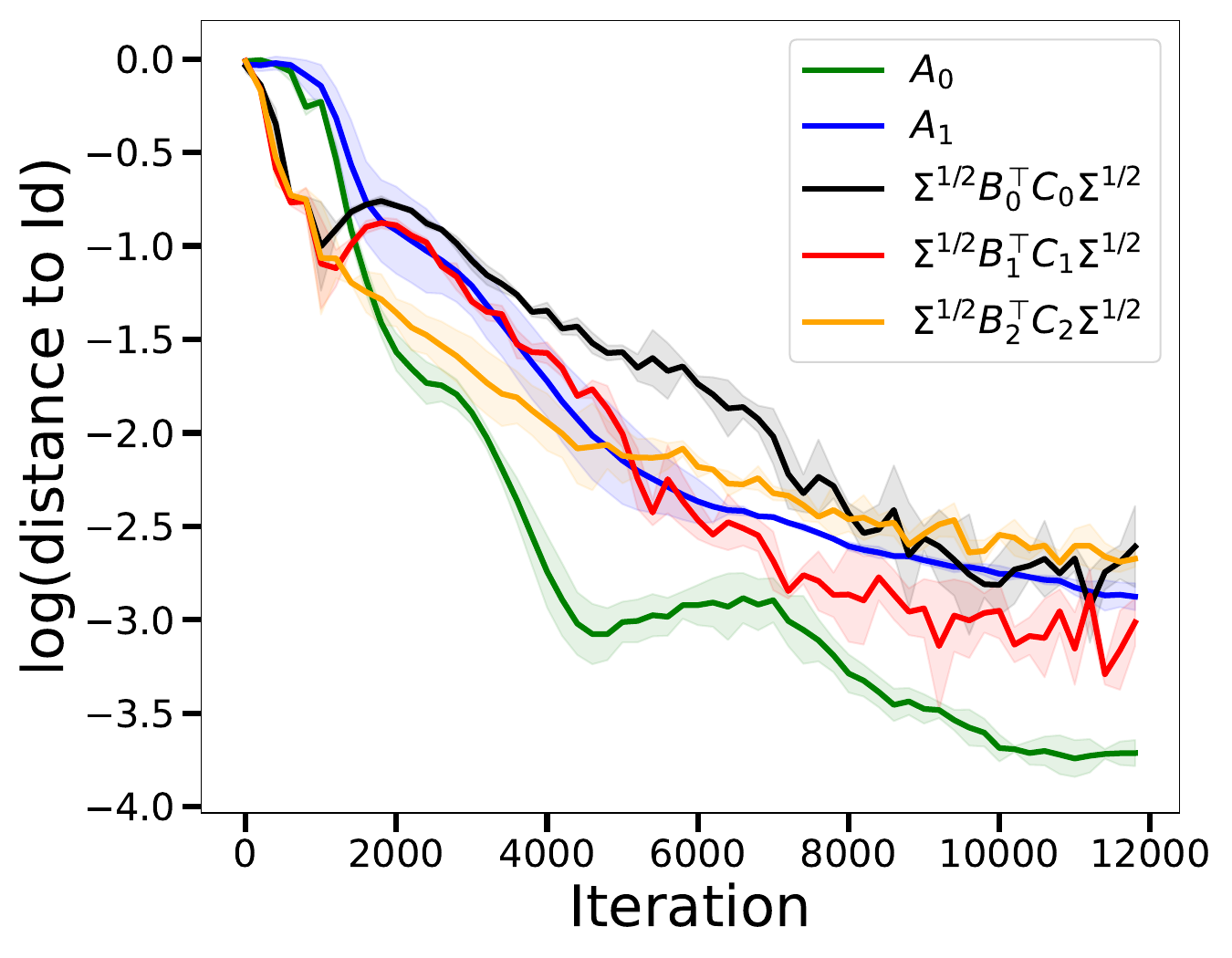}  
\caption{(Exp, Softmax)}
\end{subfigure}\\
\begin{subfigure}{0.32\textwidth}
\centering
\includegraphics[width=\textwidth]{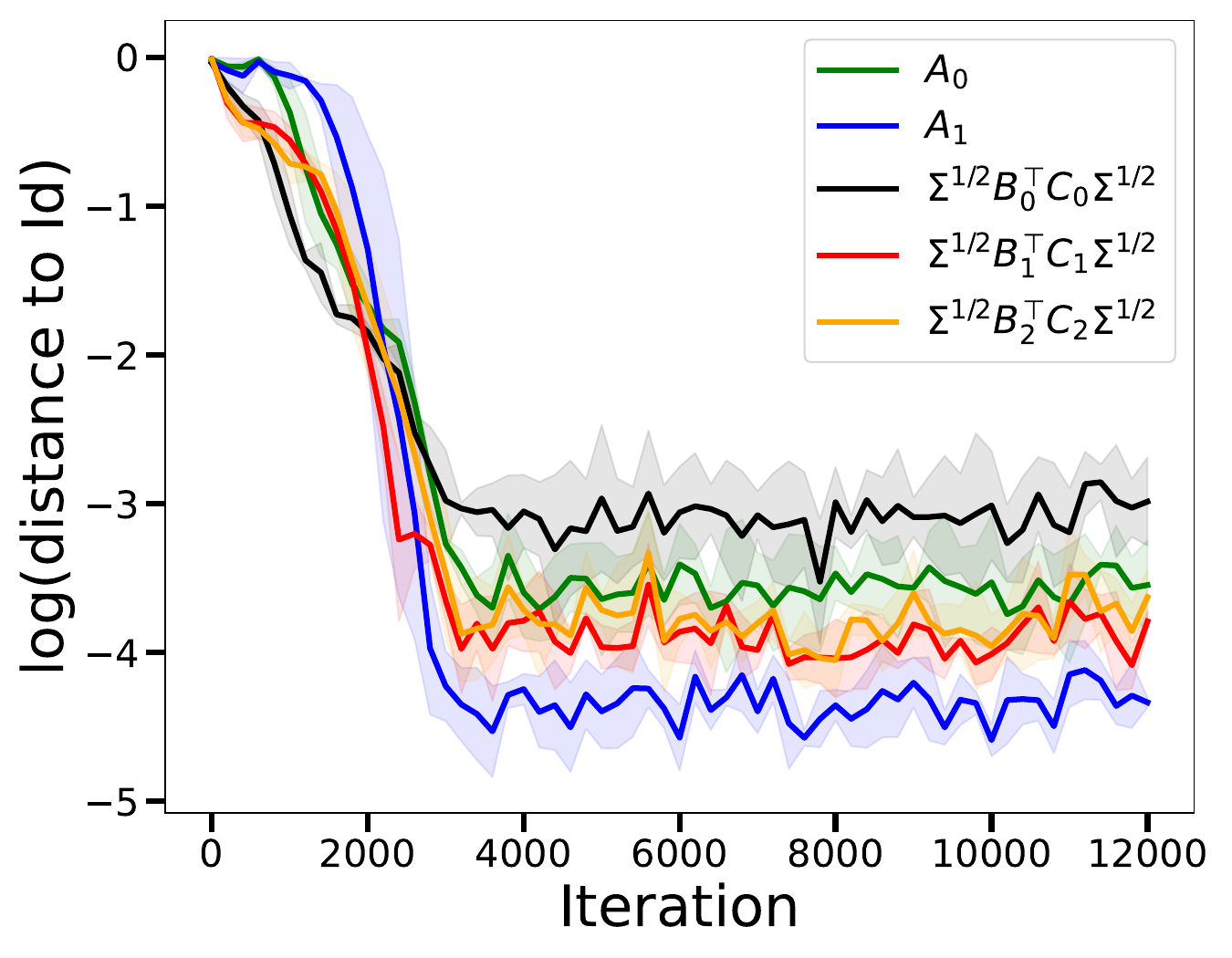} 
\caption{(Linear, Linear)}
\end{subfigure}\hfill
\begin{subfigure}{0.32\textwidth}
\centering
\includegraphics[width=\textwidth]{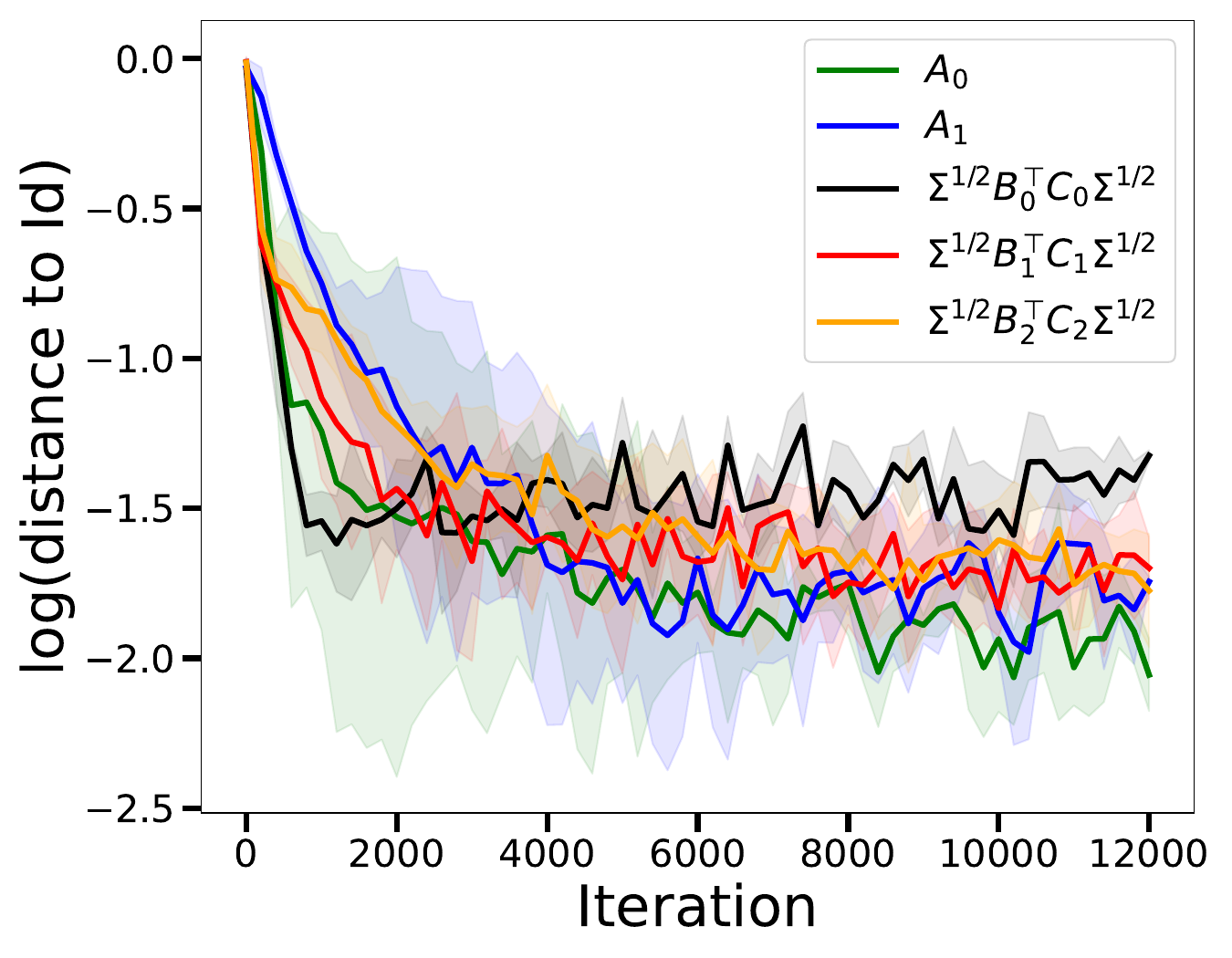}  
\caption{(ReLU, Linear)}
\label{f:theorem_full_h}
\end{subfigure}
\begin{subfigure}{0.32\textwidth}
\centering
\includegraphics[width=\textwidth]{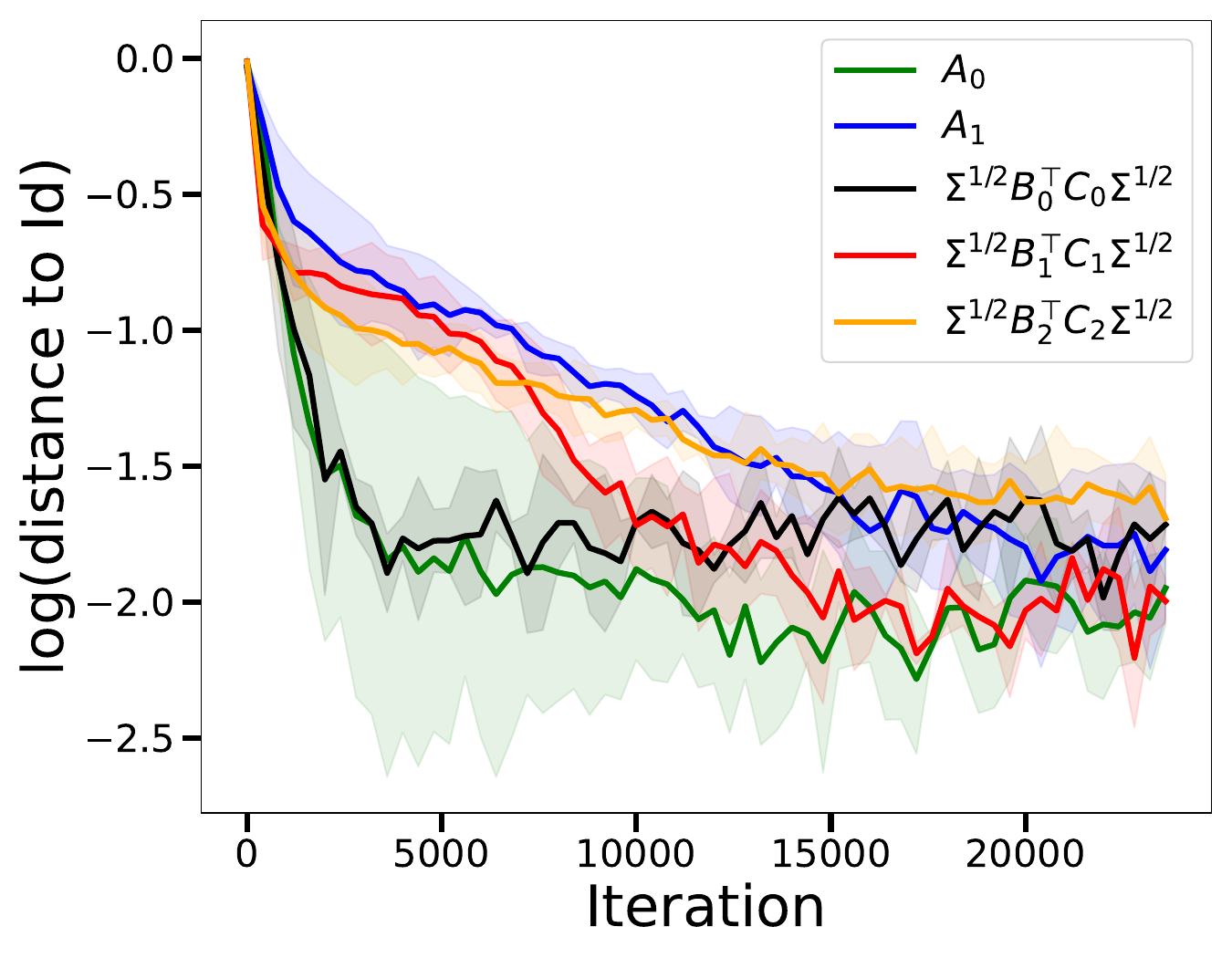}  
\caption{(Exp, Linear)}
\end{subfigure}
\caption{Plots of $\log(\dist(M,I))$ for $M \in \lrbb{A_0, A_1} \cup \lrbb{\Sigma^{1/2} B_i^\top C_i \Sigma^{1/2}}_{i=0,1,2}$ against number of training iterations. Each plot coincides with a different experiment setup, where we vary the generating distribution and the architecture. The subplot title is ($\K$, $\tilde{h}$), where $\K$ defines a Gaussian Process for labels, as described in Definition \ref{d:k_gaussian_process}, and $\th$ is the non-linear map in the Transformer's attention module. The definitions of each $\K$ and $\th$ can be found in \eqref{e:3_kernel_choices} and \eqref{e:4_kernel_choices} respectively. In all cases, the corresponding matrix appears to be converging to identity, which is the stationary point from Theorem \ref{t:informal_master_full}.}
\label{f:theorem_full}
\end{figure}

\section{Summary and Future Directions}
In this paper, we show that Transformers can learn to implement functional gradient descent in its forward pass, enabling them to learn non-linear functions in-context. The choice of non-linear attention module determines the \emph{metric function space} in which the functional gradient descent happens. When the labels are generated from a kernel Gaussian Process, the Transformer can produce the Bayes optimal predictor if the non-linear attention module matches the kernel. We show that our functional gradient descent construction is a stationary point of the in-context loss under certain sparsity constraints on the value matrix. In the absence of this sparsity constraint, we show the existence of stationary points which implement a more sophisticated functional-gradient-based learning algorithm. We provide empirical verification of all our theoretical claims.

We discuss a few potential directions for future research below:
\begin{itemize}
    \item \textbf{Representation Power via Composition:} Using the idea of kernel composition, we showed in Proposition \ref{c:multihead} that a multi-headed Transformer can have significantly greater representation power by combining \emph{parallel attention heads}. It will be interesting to also investigate the representation power of composition across layers, i.e .\emph{sequential attention heads}. It may also be worthwhile to investigate the benefits of diverse activations on practical tasks.

    \item \textbf{Optimal choice of non-linearity:} In Section \ref{ss:experiments_for_optimality}, we saw how the optimal choice of non-linear activation can depend on the function being learned. This may provide some intuition for how to select the right non-linear activation in practical settings.
    
    \item \textbf{Understanding functional gradient descent++} What is the interpretation of the algorithm implemented in Theorem \ref{t:informal_master_full}? Specifically, how does the transformation of covariates lead to a better learning algorithm?
    \item \textbf{Stronger Theoretical Guarantees: } can we show global optimality, or even establish convergence guarantees, for the stationary points in Theorem \ref{t:informal_master_sparse} and \ref{t:informal_master_full}, in light of our experimental observations?
\end{itemize}

\section*{Acknowledgement}
Xiang Cheng acknowledges NSF CCF-2112665 (TILOS AI Research Institute) for their generous support. Suvrit Sra acknowledges NSF CCF-2112665 (TILOS AI Research Institute) and the Humboldt foundation for their generous support.





\bibliographystyle{plainnat}
\setlength{\bibsep}{2pt}
\bibliography{references}

\appendix
\renewcommand{\appendixpagename}{\centering \LARGE Appendix}
\appendixpage
\startcontents[section]
\printcontents[section]{l}{1}{\setcounter{tocdepth}{2}}

\section{Reformulating the In-Context Loss}
\begin{lemma}
\label{l:icl_trace_form}
Let $Z_0\in \R^{(d+1)\times (n+1)}$ be the input to the Transformer (as defined in \eqref{d:Z_0}): 
\begin{align*}
Z_0 = \begin{bmatrix}
\tx{1} & \tx{2} & \cdots & \tx{n} &\tx{n+1} \\ 
\ty{1} & \ty{2} & \cdots &\ty{n}& 0
\end{bmatrix} \in \R^{(d+1) \times (n+1)}.
\end{align*}
Let $\overline{Z}_0$ be the input of the Transformer \textbf{without masking out $\ty{n+1}$}:
\begin{align*}
\overline{Z}_0 = \begin{bmatrix}
\tx{1} & \tx{2} & \cdots & \tx{n} &\tx{n+1} \\ 
\ty{1} & \ty{2} & \cdots &\ty{n}& \ty{n+1}
\end{bmatrix} \in \R^{(d+1) \times (n+1)},
\end{align*}
where $\ty{n+1} = \lin{\wstar, \tx{n+1}}$. 
Let $Z_\ell$ denote the output of the $(\ell-1)^{th}$ layer of the linear transformer \textbf{initialized at $Z_0$} (as defined in \eqref{e:dynamics_Z}). Let $\overline{Z}_\ell$ denote the output of the $(\ell-1)^{th}$ layer of the linear transformer \textbf{initialized at $\overline{Z}_0$} (as defined in \eqref{e:dynamics_Z}). Let $f\left(V,Q,K\right)$ denote the in-context loss defined in \eqref{d:ICL_loss}, i.e.
\begin{align*}
f\left(V,Q,K\right) = \E_{\overline{Z}_0} \Bigl[ \left( [Z_{k+1}]_{(d+1),(n+1)} + \ty{n+1} \right)^2\Bigr].
\numberthis \label{e:old_icl}
\end{align*}

Let $V_\ell$ satisfy Assumption \ref{ass:full_attention}. Then the in-context loss, defined in \eqref{d:ICL_loss}, has the equivalent form
\begin{align*}
f\left(A,B,C\right) := f\left(W\right) = \E_{\overline{Z}_0} \lrb{\tr\lrp{\lrp{I-M}\overline{Y}_{k+1}^\top \overline{Y}_{k+1}\lrp{I-M}}},
\end{align*}
where $\overline{Y}_{k+1}\in\R^{1\times n+1}$ is the $(d+1)^{th}$ row of $\overline{Z}_{k+1}$.
\end{lemma}

\begin{proof}
    Let $X_\ell \in \R^{d\times (n+1)}$ denote the first $d$ rows of $Z_\ell$ and let $Y_\ell \in \R^{1\times (n+1)}$ denote the last row of $Z_\ell$. Under Assumption \ref{ass:full_attention}, we can verify the following useful equivalent form of \eqref{e:dynamics_Z}:
    \begin{align*}
        \numberthis \label{e:dynamics_XY}
        & X_{\ell+1} = X_\ell + A_\ell X_\ell M \th\lrp{B_\ell X_\ell, C_\ell X_\ell}\\
        & Y_{\ell+1} = Y_\ell + r_\ell Y_\ell M \th\lrp{B_\ell X_\ell, C_\ell X_\ell}
    \end{align*}
    
    Let $c\in \R$ denote an arbitrary scalar. Let $\overline{X}_0 := X_0$ and let $\overline{Y}_0 = \begin{bmatrix}\ty{1} & \ty{2} & \cdots &\ty{n}& \ty{n+1} + c\end{bmatrix}$, i.e. $\overline{Y}_0$ is $\overline{Y}_0$ but with $c$ added to its last entry. Let $\overline{X}_\ell,\overline{Y}_\ell$ evolve under identical dynamics as \eqref{e:dynamics_XY} (but with initialization at $\overline{X}_0, \overline{Y}_0$):
    \begin{align*}
    & \overline{X}_{\ell+1} = \overline{X}_\ell + A_\ell \overline{X}_\ell M \th\lrp{B_\ell \overline{X}_\ell, C_\ell \overline{X}_\ell}\\
    & \overline{Y}_{\ell+1} = \overline{Y}_\ell + r_\ell \overline{Y}_\ell M \th\lrp{B_\ell \overline{X}_\ell, C_\ell \overline{X}_\ell}.
    \end{align*}

    Then for all $i$, (1) $\overline{X}_\ell = X_\ell$ and (2) $\overline{Y}_\ell - Y_\ell = \begin{bmatrix}0 & 0 & \cdots &0& c\end{bmatrix}$. 
    
    Statement (1) can be verified via simple induction. 
    
    Statement (2) follows from Statement (1) and induction: suppose (2) holds for some $i$. By definition of $M$, $\overline{Y}_\ell M$ has its $(n+1)^{th}$ entry zeroed out, thus by the inductive hypothesis, $\overline{Y}_\ell M = {Y}_\ell M$, and thus $\overline{Y}_{\ell+1} = Y_{\ell+1} + [0\ 0\cdots 0\ c]$.

    The lemma statement then follows by choosing $c = \ty{n+1}$: by (2) above, $\lrb{\overline{Z}_{k+1}}_{(d+1),(n+1)} =: \lrb{\overline{Y}_{k+1}}_{(n+1)} = \lrb{Y_{k+1}}_{(n+1)} + \ty{n+1}$. Plugging the above into the in-context loss defined in \eqref{d:ICL_loss} gives
    \begin{align*}
        f\left(V,Q,K\right) 
        =& \E_{\overline{Z}_0} \Bigl[ \left( \lrb{Z_{k+1}}_{(d+1),(n+1)} + \ty{n+1}  \right)^2\Bigr]\\
        =& \E_{\overline{Z}_0} \Bigl[ \left( \lrb{\overline{Z}_{k+1}}_{(d+1),(n+1)}\right)^2\Bigr]\\
        =& \E_{\overline{Z}_0} \Bigl[ \left( \lrn{\lrp{I-M} \overline{Y}_{k+1}^\top} \right)^2\Bigr]\\
        =& \E_{\overline{Z}_0} \Bigl[ \left( \tr\lrp{\lrp{I-M} \overline{Y}_{k+1}^\top \overline{Y}_{k+1}\lrp{I-M}} \right)^2\Bigr]
    \end{align*}
\end{proof}

\section{Theorem \ref{t:master_sparse}: Functional Gradient Descent is locally optimal under $A_\ell=0$ constraint.}
\label{s:main_theorem_sparse}
\begin{theorem}\label{t:master_sparse}
    Let $\th$ satisfy Assumption \ref{ass:th}, let $\tx{i}$'s satisfy Assumption \ref{ass:x_distribution} with matrix $\Sigma$, and $\ty{i}$'s satisfy Assumption \ref{ass:y_distribution}. With slight abuse of notation, let $f(r, B, C):= f\lrp{V = \lrbb{\bmat{0 & 0\\0 & r_\ell}}_{\ell=0\ldots k},B,C}$, where $f(V,B,C)$ is as defined in \eqref{d:ICL_loss}. Let $\S \subset \R^{(k+1) \times d \times d \times 2}$ denote a set of (Query, Key) matrices defined as follows: $(B,C) \in \S$ if and only if for all $\ell \in \lrbb{0\ldots k}$, there exist scalars $b_\ell, c_\ell \in \R$ such that $B_\ell = b_\ell \Sigma^{-1/2} $ and $C_\ell = c_\ell \Sigma^{-1/2}$. Then
    \begin{align*}
        \inf_{(r,B,C) \in \R^{k+1}\times \S} \sum_{\ell=0}^{k} \lrp{\partial_{r_\ell} f(r,B,C)}^2 + \lrn{\nabla_{B_\ell} f(r,B,C)}_F^2 + \lrn{\nabla_{C_\ell} f(r,B,C)}_F^2 = 0,
        \numberthis \label{e:sparse_theorem_condition}
    \end{align*}
    where $\nabla_{B_\ell} f$ denotes derivative wrt the Frobenius norm $\lrn{B_\ell}_F$ (same for $\nabla_{C_\ell}$).
\end{theorem}
\begin{remark}
    \label{r:overparameterized}
    By Assumption \ref{ass:th}, for any invertible $\Lambda \in \R^{d\times d}$, $f(r, B, C) = f(r, \Lambda^\top B, \Lambda^{-1} C)$. Thus the same result holds for $S_\Lambda = \lrbb{B_\ell = b_\ell \Lambda^\top \Sigma^{-1/2}, C_\ell = c_\ell \Lambda^{-1} \Sigma^{-1/2}}_{\ell=0\ldots k}$.
\end{remark}

\subsection{Proof of Proposition \ref{c:multihead}}
\label{ss:proof:c:multihead}
The proof of (A) is identical to the proof of Proposition \ref{p:rkhs_descent_transformer_construction} up to step \eqref{e:alskda:0}. We provide the remainder of the proof below:
\begin{align*}
    & \TF_{\ell+1}(x; (V,B,C)\vert \tz{1}\ldots \tz{n}) \\
    =& \TF_{\ell}(x; (V,B,C)\vert \tz{1}\ldots \tz{n}) - r_\ell \sum_{i=1}^n \textcolor{red}{\sum_{s=1}^H} \tY{i}_\ell \lrb{\th^{\textcolor{red}{s}}\lrp{\textcolor{red}{G^s}X_0, \textcolor{red}{G^s}X_0}}_{i,(n+1)}\\
    =& \TF_{\ell}(x; (V,B,C)\vert \tz{1}\ldots \tz{n}) - r_\ell \sum_{i=1}^n \textcolor{red}{\sum_{s=1}^H} \tY{i}_\ell \K^{\textcolor{red}{s}}(\textcolor{red}{G^s}\tx{i},\textcolor{red}{G^s}x)\\
    =& \TF_{\ell}(x; (V,B,C)\vert \tz{1}\ldots \tz{n}) - r_\ell \sum_{i=1}^n \tY{i}_\ell \K^{\textcolor{red}{\diamond}}(\tx{i},x)\\
    =& -f_{\ell} (x) - r_\ell \sum_{i=1}^n \lrp{\ty{i} - f_{\ell} (x)} \K^{\textcolor{red}{\diamond}}(\tx{i},x)\\
    =& -f_{\ell+1}(x).
\end{align*}
We highlight in red the differences from the proof of Proposition \ref{p:rkhs_descent_transformer_construction}. We use the definition of $\th^s$ and the definition that $\K^{\diamond}(u,v):= \sum_{s=1}^H \K^s(G^s u, G^s v)$.

The proof of (B) is entirely identical to proof of Proposition \ref{p:matching_h_k_optimality}, only replacing $\K$ by $\K^{\diamond}$, and using \eqref{e:dynamics_Z_multihead} instead of \eqref{e:dynamics_Z}.

\subsection{Proof of Theorem \ref{t:master_sparse}}
\label{ss:proof:t:master_sparse}
Let $r(0) \in \R, (B(0), C(0)) \in \S$. Let us define the $\S$-gradient-flow as
    \begin{align*}
        & \frac{d}{dt} r_\ell(t) = - \partial_{r_\ell} f(r(t), B(t), C(t))\\
        & \frac{d}{dt} B_\ell(t) = \tilde{U}_\ell(t)\\
        & \frac{d}{dt} C_\ell(t) = \tilde{W}_\ell(t),
        \numberthis \label{e:t:asdmqlwd:7}
    \end{align*}
    where for $\ell = 0\ldots k$, $\tilde{U}$ and $\tilde{W}$ are defined as
    \begin{alignat*}{2}
        & \tilde{u}_\ell(t) := -\frac{1}{d} \tr\lrp{\nabla_{B_\ell} f(r(t),B(t),C(t)) \Sigma^{1/2}}
        \qquad && \tilde{U}_\ell(t) := \tilde{u}_\ell(t) \Sigma^{-1/2}\\
        & \tilde{w}_\ell(t) := -\frac{1}{d} \tr\lrp{\nabla_{C_\ell} f(r(t),B(t),C(t)) \Sigma^{1/2}}
        \qquad && \tilde{W}_\ell(t) := \tilde{w}_\ell(t) \Sigma^{-1/2}.
    \end{alignat*}
    It follows by definition of $\tilde{U}$ and $\tilde{W}$ that $(B(t), C(t)) \in \S$ for all $t$. We will show that at any time $t$, 
    \begin{align*}
        \frac{d}{dt} f(r(t), B(t), C(t)) \leq& - \sum_{\ell=0}^{k} \lrp{\partial_{r_\ell} f(r(t),B(t),C(t))}^2 \\
        &\qquad - \sum_{\ell=0}^k \lrn{\nabla_{B_\ell} f(r(t),B(t),C(t))}_F^2 - \sum_{\ell=0}^k \lrn{\nabla_{C_\ell} f(r(t),B(t),C(t))}_F^2.
        \numberthis \label{e:t:asdmqlwd:0}
    \end{align*}
    Let $\lin{A,B}_{\tr} := \tr\lrp{A^\top B}$. By definition of the dynamics in \eqref{e:t:asdmqlwd:7},
    \begin{align}
        &\frac{d}{dt} f(r(t), B(t), C(t))\\
        \label{e:t:asdmqlwd:1}
        =& \sum_{\ell=0}^k \partial_{r_\ell} f(r(t),B(t),C(t)) \cdot \lrp{- \partial_{r_\ell} f(r(t), B(t), C(t))}\\
        \label{e:t:asdmqlwd:2}
        &\quad + \sum_{\ell=0}^k \lin{\nabla_{B_\ell} f(r(t),B(t),C(t)), \tilde{U}_\ell(t)}_{\tr}\\
        \label{e:t:asdmqlwd:3}
        &\quad + \sum_{\ell=0}^k \lin{\nabla_{C_\ell} f(r(t),B(t),C(t)), \tilde{W}_\ell(t)}_{\tr}.
    \end{align}
    We immediately verify that $\eqref{e:t:asdmqlwd:1} = - \sum_{\ell=0}^k \lrp{\partial{r_\ell} f(r(t),B(t),C(t))}^2$. By \eqref{e:t:unkqdwon:1} from Proposition \ref{p:master_sparse}, applied separately to each layer $\ell=0\ldots k$, 
    \begin{align*}
        \eqref{e:t:asdmqlwd:2}
        \leq& \sum_{\ell=0}^k \lin{\nabla_{B_\ell} f(r(t),B(t),C(t)), - \nabla_{B_\ell} f(r(t), B(t), C(t))}_{\tr}\\
        =& - \sum_{\ell=0}^k \lrn{\nabla_{B_\ell} f(r(t),B(t),C(t))}_F^2.
    \end{align*}
    Similarly, by \eqref{e:t:unkqdwon:11} from Proposition \ref{p:master_sparse}, applied separately to each layer $\ell=0\ldots k$,  
    \begin{align*}
        \eqref{e:t:asdmqlwd:3}
        \leq& \sum_{\ell=0}^k \lin{\nabla_{C_\ell} f(r(t),B(t),C(t)), - \nabla_{C_\ell} f(r(t), B(t), C(t))}_{\tr}\\
        =& - \sum_{\ell=0}^k \lrn{\nabla_{C_\ell} f(r(t),B(t),C(t))}_F^2.
    \end{align*}
    Combining the above bounds gives \eqref{e:t:asdmqlwd:0}. Suppose \eqref{e:sparse_theorem_condition} does not hold. Then there exists a positive constant $c>0$ such that for all $t$,
    \begin{align*}
        \sum_{\ell=0}^{k} \lrp{\partial_{r_\ell} f(r(t),B(t),C(t))}^2 + \lrn{\nabla_{B_\ell} f(r(t),B(t),C(t))}_F^2 + \lrn{\nabla_{C_\ell} f(r(t),B(t),C(t))}_F^2 \geq c.
    \end{align*}
    Then by \eqref{e:t:asdmqlwd:0}, $\frac{d}{dt} f(r(t), B(t), C(t)) \leq -c$ for all $t$. This contradicts the fact that $f(\cdot)$ is bounded below by $0$ (see \eqref{d:ICL_loss}). Thus we prove \eqref{e:sparse_theorem_condition}.
\subsection{Key Lemmas}
\begin{proposition}
    \label{p:master_sparse}
    Let $\th$ satisfy Assumption \ref{ass:th}, let $\tx{i}$'s satisfy Assumption \ref{ass:x_distribution} with matrix $\Sigma$, and $\ty{i}$'s satisfy Assumption \ref{ass:y_distribution}. Let $V\in \R^{(k+1) \times (d+1) \times (d+1)}$ satisfy, for all $\ell=0\ldots k$, $V_\ell = \bmat{0 & 0\\0 & r_\ell}$, where $r_\ell$ are arbitrary scalars. Let $(B,C)\in \R^{(k+1)\times d \times d\times 2}$ satisfy, for all $\ell=0\ldots k$, 
    \begin{align*}
        B_\ell = b_\ell \Sigma^{-1/2} \qquad C_\ell = c_\ell \Sigma^{-1/2},
        \numberthis \label{ass:sparse_ABC_shape}
    \end{align*}
    where $b_\ell, c_\ell \in \R$ are scalars. Let $j\in \lrbb{0\ldots k}$ be an arbitrary but fixed layer index. For $S\in \R^{d\times d}$, let $B_j(S) := B_j + S$, and let $B_\ell(S) := B_\ell$ for $\ell\neq j$. Let $B(S) := \lrbb{B_\ell(S)}_{\ell=0\ldots k}$. Recall $f(V,B,C)$ as defined in \eqref{d:ICL_loss}. Let $R\in \R^{d\times d}$ be an arbitrary matrix. Let 
    \begin{align*}
        \tilde{r} := \frac{1}{d} \tr\lrp{R\Sigma^{1/2}} \qquad \qquad \tilde{R} := \tilde{r} \Sigma^{-1/2}
        \numberthis \label{e:t:unkqdwon:8}
    \end{align*}
    Then
    \begin{align*}
        \numberthis \label{e:t:unkqdwon:1}
        \at{\frac{d}{dt} f(V, B(tR), C)}{t=0} \leq \at{\frac{d}{dt} f(V, B(t\tilde{R}), C)}{t=0}.
    \end{align*}

    Similarly, let $C_j(S) := C_j + S$, and $C_\ell(S) := C_\ell$ for $\ell\neq j$, and let $C(S) := \lrbb{C_\ell(S)}_{\ell=0\ldots k}$, then
    \begin{align*}
        \numberthis \label{e:t:unkqdwon:11}
        \at{\frac{d}{dt} f(V, B, C(tR))}{t=0} \leq \at{\frac{d}{dt} f(V, B, C(t\tilde{R}))}{t=0}.
    \end{align*}
\end{proposition}

\begin{proof}[Proof of Proposition \ref{p:master_sparse}]
   The proof of \eqref{e:t:unkqdwon:11} is identical to that of \eqref{e:t:unkqdwon:1}, so we only present the proof of \eqref{e:t:unkqdwon:1}. 

    \paragraph{Loss Reformulation:}
    Let us consider the reformulation of the in-context loss $f$ presented in Lemma \ref{d:ICL_loss}. Specifically, let $\overline{Z}_0$ be defined as
    \begin{align*}
    \overline{Z}_0 = \begin{bmatrix}
    \tx{1} & \tx{2} & \cdots & \tx{n} &\tx{n+1} \\
    \ty{1} & \ty{2} & \cdots &\ty{n}& \ty{n+1}
    \end{bmatrix} \in \R^{(d+1) \times (n+1)},
    \end{align*}
    Let $\overline{Z}_i$ denote the output of the $(i-1)^{th}$ layer of the linear transformer (as defined in \eqref{e:dynamics_Z}, initialized at $\overline{Z}_0$). For the rest of this proof, we will drop the bar, and simply denote $\overline{Z}_i$ by $Z_i$. Let $X_i\in \R^{d\times (n+1)}$ denote the first $d$ rows of $Z_i$ and let $Y_i\in \R^{1\times (n+1)}$ denote the $(d+1)^{th}$ row of $Z_k$. Under the assumption that $V_\ell = \bmat{0&0\\0& r_\ell}$ in the lemma statement, we verify that, for any $\ell \in \lrbb{0\ldots k}$,
    \begin{align*}
    & {X}_{\ell+1} = {X}_0\\
    & {Y}_{\ell+1} = {Y}_\ell + r_\ell {Y}_\ell M \th\lrp{B_\ell {X}_\ell, B_\ell {X}_\ell} = Y_0 \prod_{\ell = 0}^i \lrp{I + r_\ell M \th\lrp{B_\ell X_0, C_\ell X_0}}.
    \numberthis \label{e:t:unkqdwon:10}
    \end{align*}
    By Lemma \ref{l:icl_trace_form}, the in-context loss defined in \eqref{d:ICL_loss} is equivalent to 
    \begin{align*}
        f(V,B,C) = \E_{Z_0} \lrb{\tr\lrp{\lrp{I-M}{Y}_{k+1}^\top {Y}_{k+1}\lrp{I-M}}},
    \end{align*}
    where $Y_{k+1}$ is as defined in \eqref{e:t:unkqdwon:10}. We will now verify \eqref{e:t:unkqdwon:1}

    We will introduce one more piece of notation: for any $S\in \R^{d\times d}$, let 
    \begin{align*}
        G(X, S) := \prod_{\ell = 0}^k \lrp{I + r_\ell M \th\lrp{B_\ell(S) X, C_\ell X}},
    \end{align*}
    so that
    \begin{align*}
        f(V,B(S),C) =& \E_{Z_0} \lrb{\tr\lrp{\lrp{I-M}G(X, S)^\top {Y}_{0}^\top {Y}_{0} G(X, S) \lrp{I-M}}}\\
        =& \E_{X_0} \lrb{\tr\lrp{\lrp{I-M}G(X, S)^\top \Kmat(X_0) G(X, S) \lrp{I-M}}},
    \end{align*}
    where recall that $\Kmat(X_0)\in \R^{(n+1)\times (n+1)}$ is as defined in Assumption \ref{ass:y_distribution}. The second equality uses the assumption on distribution of $Y_0$ conditioned on $X_0$, as specified in Assumption \ref{ass:y_distribution}. Let $U$ denote a uniformly randomly sampled orthogonal matrix. Let $\US := \Sigma^{1/2} U \Sigma^{-1/2}$, so that $\US^{-1} = \Sigma^{1/2} U^\top \Sigma^{-1/2}$. Using the fact that $X_0 \overset{d}{=} \US X_0$, we can verify
    \begin{align*}
        \at{\frac{d}{dt} f(V,B(tR),C)}{t=0}
        =& \at{\frac{d}{dt} \E_{X_0} \lrb{\tr\lrp{\lrp{I-M}G(X_0, tR)^\top \Kmat(X_0) G(X_0, tR) \lrp{I-M}}}}{t=0}\\
        =& 2\E_{X_0} \lrb{\tr\lrp{\lrp{I-M}G(X_0, 0)^\top \Kmat(X_0) \at{\frac{d}{dt} G(X_0, tR)}{t=0} \lrp{I-M}}}\\
        =& 2\E_{X_0, U} \lrb{\tr\lrp{\lrp{I-M}G(\US X_0, 0)^\top \Kmat(X_0) \at{\frac{d}{dt} G(\US X_0, tR)}{t=0} \lrp{I-M}}}.
        \numberthis \label{e:t:unkqdwon:2}
    \end{align*}
    The last equality uses the assumption that $\K\lrp{\US X} = \K\lrp{X}$ from Assumption \ref{ass:y_distribution}.

    We will now show the following useful identities:
    \begin{align}
        \label{e:t:unkqdwon:3}
        & G(\US X_0, 0) = G(X_0, 0)\\
        \label{e:t:unkqdwon:4}
        & \at{\frac{d}{dt} G(\US X_0, tR)}{t=0} = \at{\frac{d}{dt} G(X_0, t \US^\top R \US)}{t=0} 
    \end{align}
    A useful intermediate identity is 
    \begin{align*}
        B_\ell \US = b_\ell \Sigma^{-1/2} \Sigma^{1/2} U \Sigma^{-1/2} = U B_\ell\\
        C_\ell \US = c_\ell \Sigma^{-1/2} \Sigma^{1/2} U \Sigma^{-1/2} = U C_\ell.
        \numberthis \label{e:t:unkqdwon:6}
    \end{align*}
    In the above, we crucially use the assumed form of $B_\ell, C_\ell$ from the \eqref{ass:sparse_ABC_shape}. We can now verify \eqref{e:t:unkqdwon:3}, which follows almost immediately from \eqref{e:t:unkqdwon:6}:
    \begin{align*}
        G(\US X_0, 0) 
        =& \prod_{\ell = 0}^k \lrp{I + r_\ell  M \th\lrp{B_\ell \US X_0, C_\ell \US X_0}} \\
        =& \prod_{\ell = 0}^k \lrp{I + r_\ell M \th\lrp{U B_\ell X_0, U C_\ell X_0}}\\
        =& \prod_{\ell = 0}^k \lrp{I + r_\ell M \th\lrp{B_\ell X_0, C_\ell X_0}} = G(X_0, 0),
    \end{align*}
    where the second equality uses \eqref{e:t:unkqdwon:6}, and the third equality uses the invariance of $\th$ from Assumption \ref{ass:th}.

    We now begin the verification of \eqref{e:t:unkqdwon:4}. To do so, let $\jA : \R^{d \times (n+1)} \to \R^{(n+1)\times (n+1)}$ denote the Jacobian of $\th$ wrt its first argument, evaluated at $\lrp{B_\ell X_0, C_\ell X_0}$. In more precise notation, for any $U,V,T\in \R^{d\times (n+1)}$, $\jA(U,V)\lrb{T} := \at{\frac{d}{dt} \tilde{h}\lrp{U + T, V}}{t=0}$. We verify the following useful identity: for any $S \in R^{d\times d}$, 
    \begin{align*}
        & \jA(B_\ell \US X_0, C_\ell \US X_0)\lrb{S \US X_0}\\
        =& \at{\frac{d}{dt} \th\lrp{ U B_\ell X_0 + tS\US X_0, U C_\ell X_0}}{t=0}\\
        =& \at{\frac{d}{dt} \th\lrp{ B_\ell X_0+ t  U^\top S \US X_0, C_\ell X_0}}{t=0}\\
        =& \jA(B_\ell X_0, C_\ell X_0)\lrb{U^\top S \US X_0},
        \numberthis \label{e:t:unkqdwon:7}
    \end{align*}
    where the first equality is by \eqref{e:t:unkqdwon:6}, the second equality is by Assumption \ref{ass:th}, the third equality is by definition of $\jA$. The identity \eqref{e:t:unkqdwon:4} then follows easily from chain rule and \eqref{e:t:unkqdwon:7}:
    \begin{align*}
        & \at{\frac{d}{dt} G(\US X_0, tR)}{t=0}\\
        =& \lrp{\prod_{\ell=0}^{j-1} \lrp{I + M\th\lrp{B_\ell \US X_0, C_\ell \US X_0}}} M \jA\lrp{B_j \US X_0, C_j \US X_0}\lrb{tR} \lrp{\prod_{\ell=j+1}^{k} \lrp{I + M\th\lrp{B_\ell \US X_0, C_\ell \US X_0}}}\\
        =& \lrp{\prod_{\ell=0}^{j-1} \lrp{I + M\th\lrp{B_\ell X_0, C_\ell X_0}}} M \jA\lrp{B_j \US X_0, C_j \US X_0}\lrb{tR} \lrp{\prod_{\ell=j+1}^{k} \lrp{I + M\th\lrp{B_\ell X_0, C_\ell X_0}}}\\
        =& \lrp{\prod_{\ell=0}^{j-1} \lrp{I + M\th\lrp{B_\ell X_0, C_\ell X_0}}} M \jA\lrp{B_j X_0, C_j X_0}\lrb{tU^\top R \US} \lrp{\prod_{\ell=j+1}^{k} \lrp{I + M\th\lrp{B_\ell X_0, C_\ell X_0}}}\\
        =& \at{\frac{d}{dt} G(\US X_0, tU^\top R \US)}{t=0}
    \end{align*}
    In the above, the second equality uses \eqref{e:t:unkqdwon:6} and Assumption \ref{ass:th}. The third equality uses \eqref{e:t:unkqdwon:7}. The fourth equality again uses chain rule. This concludes the proof of \eqref{e:t:unkqdwon:4}.

    We will now continue from \eqref{e:t:unkqdwon:2}:
    \begin{align*}
        \at{\frac{d}{dt} f(V,B(tR),C)}{t=0} 
        = \eqref{e:t:unkqdwon:2}
        =& 2\E_{X_0, U} \lrb{\tr\lrp{\lrp{I-M}G(\US X_0, 0)^\top \Kmat \at{\frac{d}{dt} G(\US X_0, tR)}{t=0} \lrp{I-M}}}\\
        =& 2\E_{X_0, U} \lrb{\tr\lrp{\lrp{I-M}G(X_0, 0)^\top \Kmat \at{\frac{d}{dt} G(X_0, t U^\top R \US)}{t=0} \lrp{I-M}}}\\
        =& 2\E_{X_0} \lrb{\tr\lrp{\lrp{I-M}G(X_0, 0)^\top \Kmat \at{\frac{d}{dt} G(X_0, t \E_{U}\lrb{U^\top R \US})}{t=0} \lrp{I-M}}}\\
        =& 2\E_{X_0} \lrb{\tr\lrp{\lrp{I-M}G(X_0, 0)^\top \Kmat \at{\frac{d}{dt} G(X_0, t \tilde{R})}{t=0} \lrp{I-M}}}\\
        =& \at{\frac{d}{dt} f(V,B(t\tilde{R}), C)}{t=0} 
    \end{align*}
    In the above, the second equality is by plugging in \eqref{e:t:unkqdwon:3} and \eqref{e:t:unkqdwon:4}. The third equality uses the fact that for any $S$, $\at{\frac{d}{dt} G(X_0, t S)}{t=0}$ is linear in $S$ (and jointly continuously differentiable in both $S$ and $t$). The fourth equality uses the definition of $\tilde{R}$ from \eqref{e:t:unkqdwon:8}. This concludes the proof of \eqref{e:t:unkqdwon:1}.
\end{proof}

\section{Theorem \ref{t:master_full}: locally optima when $A_\ell$ are unconstrained.}
\label{s:main_theorem_full}
\begin{theorem}\label{t:master_full}
    Let $\th$ satisfy Assumption \ref{ass:th}, let $\tx{i}$'s satisfy Assumption \ref{ass:x_distribution} with matrix $\Sigma$, and $\ty{i}$'s satisfy Assumption \ref{ass:y_distribution}. With abuse of notation, let $f(r, A, B, C):= f\lrp{V = \lrbb{\bmat{A_\ell & 0\\0 & r_\ell}}_{\ell=0\ldots k},B,C}$, where $f(V,B,C)$ is as defined in \eqref{d:ICL_loss}.

    Let $\S \subset \R^{(k+1) \times d \times d \times 3}$ denote a set of matrices defined as follows: $(A,B,C) \in \S$ if and only if for all $\ell \in \lrbb{0\ldots k}$, there exist scalars $a_\ell, b_\ell, c_\ell \in \R$ such that $A_\ell = a_\ell I, B_\ell = b_\ell \Sigma^{-1/2} $ and $C_\ell = c_\ell \Sigma^{-1/2}$. Then
    \begin{align*}
        \inf_{(r,A,B,C) \in \R^{k+1}\times \S} \sum_{\ell=0}^{k} \lrp{\partial_{r_\ell} f(r,A, B,C)}^2 + \lrn{\nabla_{A_\ell} f(r,A,B,C)}_F^2 + \lrn{\nabla_{B_\ell} f(r,A,B,C)}_F^2 + \lrn{\nabla_{C_\ell} f(r,A,B,C)}_F^2 = 0,\\
        \ \numberthis \label{e:full_theorem_condition}
    \end{align*}
    where $\nabla_{A_\ell} f$ denotes derivative wrt the Frobenius norm $\lrn{A_\ell}_F$ (same for $\nabla_{B_\ell}$ and $\nabla_{C_\ell}$).
\end{theorem}
\begin{remark}
    By Assumption \ref{ass:th}, for any invertible $\Lambda \in \R^{d\times d}$, $f(r, A, B, C) = f(r, A, \Lambda^\top B, \Lambda^{-1} C)$. Thus the same result holds for $S_\Lambda = \lrbb{A_\ell = a_\ell I, B_\ell = b_\ell \Lambda^\top \Sigma^{-1/2}, C_\ell = c_\ell \Lambda^{-1} \Sigma^{-1/2}}_{\ell=0\ldots k}$.
\end{remark}


\subsection{Proof of Theorem \ref{t:master_full}}
\label{ss:proof:t:master_full}
Let $r(0) \in \R, (A(0), B(0), C(0)) \in \S$. Let us define the $\S$-gradient-flow as
    \begin{align*}
        & \frac{d}{dt} r_\ell(t) = - \partial_{r_\ell} f(r(t), A(t), B(t), C(t))\\
        & \frac{d}{dt} A_\ell(t) = \tilde{P}_\ell(t)\\
        & \frac{d}{dt} B_\ell(t) = \tilde{U}_\ell(t)\\
        & \frac{d}{dt} C_\ell(t) = \tilde{W}_\ell(t),
        \numberthis \label{e:t:oqiwmdals:7}
    \end{align*}
    where for $\ell = 0\ldots k$, $\tilde{P}$, $\tilde{U}$, and $\tilde{W}$ are defined as
    \begin{alignat*}{2}
        & \tilde{p}_\ell(t) := -\frac{1}{d} \tr\lrp{\Sigma^{-1/2} \nabla_{P_\ell} f(r(t),A(t),B(t),C(t)) \Sigma^{1/2}}
        \qquad && \tilde{P}_\ell(t) := \tilde{p}_\ell(t) I\\
        & \tilde{u}_\ell(t) := -\frac{1}{d} \tr\lrp{\nabla_{B_\ell} f(r(t),A(t),B(t),C(t)) \Sigma^{1/2}}
        \qquad && \tilde{U}_\ell(t) := \tilde{u}_\ell(t) \Sigma^{-1/2}\\
        & \tilde{w}_\ell(t) := -\frac{1}{d} \tr\lrp{\nabla_{C_\ell} f(r(t),A(t),B(t),C(t)) \Sigma^{1/2}}
        \qquad && \tilde{W}_\ell(t) := \tilde{w}_\ell(t) \Sigma^{-1/2}.
    \end{alignat*}
    It follows by definition of $\tilde{P}$, $\tilde{U}$, and $\tilde{W}$ that $(A(t), B(t), C(t)) \in \S$ for all $t$. We will show that at any time $t$, 
    \begin{align*}
        &\frac{d}{dt} f(r(t), A(t), B(t), C(t)) \\
        \leq& - \sum_{\ell=0}^{k} \lrp{\partial_{r_\ell} f(r(t),A(t),B(t),C(t))}^2 \\
        &\qquad - \sum_{\ell=0}^k \lrn{\nabla_{A_\ell} f(r(t),A(t),B(t),C(t))}_F^2\\
        &\qquad - \sum_{\ell=0}^k \lrn{\nabla_{B_\ell} f(r(t),A(t),B(t),C(t))}_F^2 - \sum_{\ell=0}^k \lrn{\nabla_{C_\ell} f(r(t),A(t),B(t),C(t))}_F^2.
        \numberthis \label{e:t:oqiwmdals:0}
    \end{align*}
    Let $\lin{A,B}_{\tr} := \tr\lrp{A^\top B}$. By definition of the dynamics in \eqref{e:t:oqiwmdals:7},
    \begin{align}
        &\frac{d}{dt} f(r(t), A(t), B(t), C(t))\\
        \label{e:t:oqiwmdals:1}
        =& \sum_{\ell=0}^k \partial_{r_\ell} f(r(t),A(t),B(t),C(t)) \cdot \lrp{- \partial_{r_\ell} f(r(t), A(t), B(t), C(t))}\\
        \label{e:t:oqiwmdals:4}
        &\quad + \sum_{\ell=0}^k \lin{\nabla_{A_\ell} f(r(t),A(t),B(t),C(t)), \tilde{P}_\ell(t)}_{\tr}\\
        \label{e:t:oqiwmdals:2}
        &\quad + \sum_{\ell=0}^k \lin{\nabla_{B_\ell} f(r(t),A(t),B(t),C(t)), \tilde{U}_\ell(t)}_{\tr}\\
        \label{e:t:oqiwmdals:3}
        &\quad + \sum_{\ell=0}^k \lin{\nabla_{C_\ell} f(r(t),A(t),B(t),C(t)), \tilde{W}_\ell(t)}_{\tr}.
    \end{align}
    We immediately verify that $\eqref{e:t:oqiwmdals:1} = - \sum_{\ell=0}^k \lrp{\partial{r_\ell} f(r(t),A(t),B(t),C(t))}^2$. By \eqref{e:t:bjkjrdnads_A:3} from Proposition \ref{p:master_full_A}, applied separately to each layer $\ell=0\ldots k$, 
    \begin{align*}
        \eqref{e:t:oqiwmdals:4}
        \leq& \sum_{\ell=0}^k \lin{\nabla_{A_\ell} f(r(t), A(t),B(t),C(t)), - \nabla_{A_\ell} f(r(t), A(t), B(t),C(t))}_{\tr}\\
        =& - \sum_{\ell=0}^k \lrn{\nabla_{B_\ell} f(r(t),A(t),B(t),C(t))}_F^2.
    \end{align*}

    By \eqref{e:t:bjkjrdnads:3} from Proposition \ref{p:master_full_BC}, applied separately to each layer $\ell=0\ldots k$, 
    \begin{align*}
        \eqref{e:t:oqiwmdals:2}
        \leq& \sum_{\ell=0}^k \lin{\nabla_{B_\ell} f(r(t),A(t),B(t),C(t)), - \nabla_{B_\ell} f(r(t), A(t), B(t), C(t))}_{\tr}\\
        =& - \sum_{\ell=0}^k \lrn{\nabla_{B_\ell} f(r(t),A(t),B(t),C(t))}_F^2.
    \end{align*}
    Similarly, by \eqref{e:t:bjkjrdnads:33} from Proposition \ref{p:master_full_BC}, applied separately to each layer $\ell=0\ldots k$,  
    \begin{align*}
        \eqref{e:t:oqiwmdals:3}
        \leq& \sum_{\ell=0}^k \lin{\nabla_{C_\ell} f(r(t),A(t),B(t),C(t)), - \nabla_{C_\ell} f(r(t), A(t), B(t), C(t))}_{\tr}\\
        =& - \sum_{\ell=0}^k \lrn{\nabla_{C_\ell} f(r(t),A(t),B(t),C(t))}_F^2.
    \end{align*}
    Combining the above bounds gives \eqref{e:t:oqiwmdals:0}. Suppose \eqref{e:full_theorem_condition} does not hold. Then there exists a positive constant $c>0$ such that for all $t$,
    \begin{align*}
        & \sum_{\ell=0}^{k} \lrp{\partial_{r_\ell} f(r(t),A(t),B(t),C(t))}^2 + \lrn{\nabla_{A_\ell} f(r(t),A(t),B(t),C(t))}_F^2
        \\& + \lrn{\nabla_{B_\ell} f(r(t),A(t),B(t),C(t))}_F^2 + \lrn{\nabla_{C_\ell} f(r(t),A(t),B(t),C(t))}_F^2 \geq c.
    \end{align*}
    Then by \eqref{e:t:oqiwmdals:0}, $\frac{d}{dt} f(r(t), A(t), B(t), C(t)) \leq -c$ for all $t$. This contradicts the fact that $f(\cdot)$ is bounded below by $0$ (see \eqref{d:ICL_loss}). Thus we prove \eqref{e:full_theorem_condition}.

\subsection{Key Lemmas}

\begin{proposition}
    \label{p:master_full_BC}
    Let $\th$ satisfy Assumption \ref{ass:th}, let $\tx{i}$'s satisfy Assumption \ref{ass:x_distribution} with matrix $\Sigma$, and $\ty{i}$'s satisfy Assumption \ref{ass:y_distribution}. Let $(A,B,C)\in \R^{(k+1)\times d \times d\times 3}$ satisfy, for all $\ell=0\ldots k$, 
    \begin{align*}
        A_\ell = a_\ell I \qquad B_\ell = b_\ell \Sigma^{-1/2} \qquad C_\ell = c_\ell \Sigma^{-1/2},
        \numberthis \label{ass:full_ABC_shape}
    \end{align*}
    where $a_\ell, b_\ell, c_\ell \in \R$ are scalars. Let $V\in \R^{(k+1) \times (d+1) \times (d+1)}$ satisfy, for all $\ell=0\ldots k$, $V_\ell = \bmat{A_\ell & 0\\0 & r_\ell}$, where $r_\ell$ are arbitrary scalars. Let $j\in \lrbb{0\ldots k}$ be an arbitrary but fixed layer index. For $S\in \R^{d\times d}$, let $B_j(S) := B_j + S$, and let $B_\ell(S) := B_\ell$ for $\ell\neq j$. Let $B(S) := \lrbb{B_\ell(S)}_{\ell=0\ldots k}$. Recall $f(V,B,C)$ as defined in \eqref{d:ICL_loss}. Let $R\in \R^{d\times d}$ be an arbitrary matrix. Let 
    \begin{align*}
        \tilde{r} := \frac{1}{d} \tr\lrp{R\Sigma^{1/2}} \qquad \qquad \tilde{R} := \tilde{r} \Sigma^{-1/2}
        \numberthis \label{e:t:bjkjrdnads:2}
    \end{align*}
    Then
    \begin{align*}
        \numberthis \label{e:t:bjkjrdnads:3}
        \at{\frac{d}{dt} f(V, B(tR), C)}{t=0} \leq \at{\frac{d}{dt} f(V, B(t\tilde{R}), C)}{t=0}.
    \end{align*}

    Similarly, let $C_j(S) := C_j + S$, and $C_\ell(S) := C_\ell$ for $\ell\neq j$, and let $C(S) := \lrbb{C_\ell(S)}_{\ell=0\ldots k}$, then
    \begin{align*}
        \numberthis \label{e:t:bjkjrdnads:33}
        \at{\frac{d}{dt} f(V, B, C(tR))}{t=0} \leq \at{\frac{d}{dt} f(V, B, C(t\tilde{R}))}{t=0}.
    \end{align*}
\end{proposition}
\begin{proof}[Proof of Proposition \ref{p:master_full_BC}]
    The proof of \eqref{e:t:bjkjrdnads:33} is identical to that of \eqref{e:t:bjkjrdnads:3}, so we only present the proof of \eqref{e:t:bjkjrdnads:3}.
    \paragraph{Loss Reformulation:} 
    Let us consider the reformulation of the in-context loss $f$ presented in Lemma \ref{d:ICL_loss}. Specifically, let $\overline{Z}_0$ be defined as
    \begin{align*}
    \overline{Z}_0 = \begin{bmatrix}
    \tx{1} & \tx{2} & \cdots & \tx{n} &\tx{n+1} \\
    \ty{1} & \ty{2} & \cdots &\ty{n}& \ty{n+1}
    \end{bmatrix} \in \R^{(d+1) \times (n+1)},
    \end{align*}
    Let $\overline{Z}_\ell$ denote the output of the $(i-1)^{th}$ layer of the linear transformer (as defined in \eqref{e:dynamics_Z}, initialized at $\overline{Z}_0$). For the rest of this proof, we will drop the bar, and simply denote $\overline{Z}_\ell$ by $Z_\ell$. Let $X_\ell\in \R^{d\times (n+1)}$ denote the first $d$ rows of $Z_\ell$ and let $Y_\ell\in \R^{1\times (n+1)}$ denote the $(d+1)^{th}$ row of $Z_k$. Under the theorem's assumption that $V_\ell = \bmat{A_\ell & 0\\0 & r_\ell}$, we verify that, for any $\ell \in \lrbb{0\ldots k}$,
    \begin{align*}
    & {X}_{\ell+1} = X_\ell + A_\ell X_\ell M \th\lrp{B_\ell X_\ell, C_\ell X_\ell}\\
    & {Y}_{\ell+1} = {Y}_\ell + r_\ell {Y}_\ell M \th\lrp{B_\ell {X}_\ell, B_\ell {X}_\ell} = Y_0 \prod_{\ell = 0}^i \lrp{I + r_\ell M \th\lrp{B_\ell X_0, C_\ell X_0}}.
    \numberthis \label{e:dynamic_XY_full_proof}
    \end{align*}
    By Lemma \ref{d:ICL_loss}, the in-context loss defined in \eqref{d:ICL_loss} is equivalent to 
    \begin{align*}
        f(V,B,C) = \E_{Z_0} \lrb{\tr\lrp{\lrp{I-M}{Y}_{k+1}^\top {Y}_{k+1}\lrp{I-M}}}
    \end{align*}
    
    We will introduce one more piece of notation: Following \eqref{e:dynamic_XY_full_proof}, notice that for any layer $i$, $X_i$ is a function of $A, B, C, X_0$. Since for this part of the proof, only $B$ is variable (function of $S$), and $A,C$ are fixed, we define $X_i(X, S)$ to be "the result of evolving as \eqref{e:dynamic_XY_full_proof}, initialized at $X_0=X$, where $B_i$ is replaced by $B_i(S)$", i.e.
    \begin{align*}
        \numberthis \label{e:Xi(X,S)_dynamics_full}
        {X}_{i+1}(X,S) = X_i(X,S) + A_i X_i(X,S) M \th\lrp{B_i(S) X_i(X,S), C_i X_i(X,S)}
    \end{align*}
    Let us also define
    \begin{align*}
        G_i(X, S) := \prod_{\ell = 0}^i \lrp{I + r_\ell M \th\lrp{B_\ell(S) X_\ell(X,S), C_\ell X_\ell(X,S)}},
    \end{align*}
    so that
    \begin{align*}
        f(V, B(S), C) =& \E_{Z_0} \lrb{\tr\lrp{\lrp{I-M}G_k(X, S)^\top {Y}_{0}^\top {Y}_{0} G_k(X, S) \lrp{I-M}}}\\
        =& \E_{X_0} \lrb{\tr\lrp{\lrp{I-M}G_k(X, S)^\top \Kmat G_k(X, S) \lrp{I-M}}},
    \end{align*}
    where recall that $\Kmat\in \R^{(n+1)\times (n+1)}$ and $\Kmat_{ij} = \K(\Sigma^{-1/2} \tx{i}, \Sigma^{-1/2}\tx{j})$ as defined in Assumption \ref{ass:y_distribution}. The second equality uses the assumption on distribution of $Y_0$ conditioned on $X_0$, as specified in Assumption \ref{ass:y_distribution}. Let $U$ denote a uniformly randomly sampled orthogonal matrix. Let $\US := \Sigma^{1/2} U \Sigma^{-1/2}$, so that $\US^{-1} = \Sigma^{1/2} U^\top \Sigma^{-1/2}$. We will repeatedly use the following identities:
    \begin{align*}
        & B_i \US = b_i \Sigma^{-1/2} \Sigma^{1/2} U \Sigma^{-1/2} = U B_i\\
        & C_i \US = c_i \Sigma^{-1/2} \Sigma^{1/2} U \Sigma^{-1/2} = U C_i
        \numberthis\label{e:B_US_commute_full}
    \end{align*}
    
    Using the fact that $X_0 \overset{d}{=} \US X_0$, we can verify
    \begin{align*}
        \at{\frac{d}{dt} f(V, B(tR), C)}{t=0}
        =& \at{\frac{d}{dt} \E_{X_0} \lrb{\tr\lrp{\lrp{I-M}G_k(X_0, tR)^\top \Kmat(X_0) G_k(X_0, tR) \lrp{I-M}}}}{t=0}\\
        =& 2\E_{X_0} \lrb{\tr\lrp{\lrp{I-M}G_k(X_0, 0)^\top \Kmat(X_0) \at{\frac{d}{dt} G_k(X_0, tR)}{t=0} \lrp{I-M}}}\\
        =& 2\E_{X_0, U} \lrb{\tr\lrp{\lrp{I-M}G_k(\US X_0, 0)^\top \Kmat(X_0) \at{\frac{d}{dt} G_k(\US X_0, tR)}{t=0} \lrp{I-M}}}.
        \numberthis \label{e:t:bjkjrdnads:9}
    \end{align*}
    The last equality uses the fact that $\Kmat(\US X_0) = \Kmat(X_0)$ by Assumption \ref{ass:y_distribution}.

    Henceforth, assume all $\frac{d}{dt}$ occurs at $t=0$, and we somtimes drop the explicit $\at{}{t=0}$ notation to save space.

    \paragraph{$X_i$ and $\frac{d}{dt} X_i$ under random transformation of $X_0$}\\
    In this part of the proof, we establish two important identities about the evolution of $X_i$ under random rotation of its arguments:
    \begin{align*}
        \numberthis \label{e:Xi_rotation_full}
        & X_i\lrp{\US X_0, 0} = \US X_i \lrp{X_0, 0},\\
        \numberthis \label{e:ddt_Xi_rotation_full}
        & \at{\frac{d}{dt} X_i\lrp{\US X_0, tR }}{t=0} = \US \at{\frac{d}{dt} X_i\lrp{X_0, tU^{\top} R \US}}{t=0}.
    \end{align*}
    
    We first verify \eqref{e:Xi_rotation_full} by induction. For $i=0$, this identity holds by definition. Assume the identity holds for some $i$. Then following \eqref{e:Xi(X,S)_dynamics_full}, 
    \begin{align*}
        X_{i+1}(\US X_0,0) 
        =& X_i(\US X_0,0) + A_i X_i(\US X_0,0) M \th\lrp{B_i X_i(\US X_0,0), C_i X_i(\US X_0,0)}\\
        =& \US X_i(X_0,0) + \US A_i X_i(X_0,0) M \th\lrp{B_i \US  X_i(X_0,0), C_i \US X_i(X_0,0)}\\
        =& \US X_i(X_0,0) + \US A_i X_i(X_0,0) M \th\lrp{B_i X_i(X_0,0), C_i X_i(X_0,0)}\\
        =& \US X_{i+1}(X_0,0).
    \end{align*}
    The second equality is by the inductive hypothesis, and the fact that $A_i = a_i I$. The third equality uses \eqref{e:B_US_commute_full} and Assumption \ref{ass:th}.

    Next, we verify \eqref{e:ddt_Xi_rotation_full}. By definition of $X_i(X_0,S)$, the case for $i\leq j$ is simple:
    \begin{align*}
        \at{\frac{d}{dt} X_i(\US X_0, tR)}{t=0} = 0 = \US \at{\frac{d}{dt} X_i(X_0, tU^{\top} R \US)}{t=0}.
        \numberthis \label{e:t:bjkjrdnads:0}
    \end{align*}
    
    For $i=j+1$, it follows from \eqref{e:Xi(X,S)_dynamics_full} and chain rule that
    \begin{align*}
        & \frac{d}{dt} X_{j+1}(\US X_0, tR)\\
        =& \frac{d}{dt} X_j (\US X_0, tR)  + A_j \lrp{\frac{d}{dt} X_j (\US X_0, tR)}M \th\lrp{B_j X_j(\US X_0, 0), C_j X_j (\US X_0,0))}\\
        &\qquad + A_j X_j(\US X_0, 0) M \frac{d}{dt} \th\lrp{ \underbrace{\lrp{B_j + tR} X_j(\US X_0, tR)}_{S(t)}, \underbrace{C_j X_j (\US X_0, tR)}_{T(t)}}.
        \numberthis \label{e:t:bjkjrdnads:1}
    \end{align*}
    We will now apply Lemma \ref{l:th_ddt_invariance}. Let $S(t):= \lrp{B_j + tR} \US X_j(\US X_0, tR)$ and $T(t) :=  C_j \US X_j (\US X_0, tR)$. By \eqref{e:t:bjkjrdnads:0}, we know that $\at{\frac{d}{dt} X_j(\US X_0, tR)}{t=0} = 0$. Thus, we can define $\tS(t):= \lrp{B_j + tR} X_j(\US X_0, 0)$ and $\tT(t):= C_j X_j(\US X_0, 0)$. Using \eqref{e:Xi_rotation_full} and \eqref{e:B_US_commute_full}, we verify that 
    \begin{align*}
        \tS(t)
        =& \lrp{B_j + tR} \US X_j(X_0, 0)= U \lrp{B_j + t U^\top R \US} X_j(X_0, 0)\\
        \tT(t)
        =& \lrp{C_j + tR} \US X_j(X_0, 0)= U \lrp{C_j + t U^\top R \US} X_j(X_0, 0).
    \end{align*}
    Let us therefore pick $\Gamma := U$. Applying Lemma \ref{l:th_ddt_invariance} and plugging into \eqref{e:t:bjkjrdnads:1} gives
    \begin{align*}
        & \frac{d}{dt} X_{j+1}(\US X_0, tR)\\
        =& \US \frac{d}{dt} X_j \lrp{X_0, t U^\top R \US} + A_j \US \lrp{\frac{d}{dt} X_j \lrp{X_0, t U^\top R \US}} M \th\lrp{B_j X_j(X_0, 0), C_j X_j (X_0,0))}\\
        &\qquad + A_j \US X_j(X_0, 0) M \frac{d}{dt} \th\lrp{ \underbrace{\lrp{B_j + tU^\top R \US} X_j( X_0, 0)}_{\Gamma^\top \tS(t)}, \underbrace{C_j X_j (X_0, 0)}_{\Gamma^{-1} \tT(t)}}\\
        =& \US \frac{d}{dt} X_{j+1}(X_0, tU^\top R \US),
    \end{align*}
    where the first equality also uses \eqref{e:B_US_commute_full} and \eqref{e:Xi_rotation_full} and \eqref{e:t:bjkjrdnads:0}. 

    Finally, we need to prove \eqref{e:ddt_Xi_rotation_full} for the $i>j+1$ case. We will prove this by induction over $i$. The proof is very similar to the $i=j+1$ case:
    \begin{align*}
        & \frac{d}{dt} X_{i+1}(\US X_0, tR)\\
        =& \frac{d}{dt} X_i (\US X_0, tR)  + A_i \lrp{\frac{d}{dt} X_i (\US X_0, tR)}M \th\lrp{B_i X_i(\US X_0, 0), C_i X_i (\US X_0,0))}\\
        &\qquad + A_i X_i(\US X_0, 0) M \frac{d}{dt} \th\lrp{ \underbrace{B_i X_i(\US X_0, tR)}_{S(t)}, \underbrace{C_i X_i (\US X_0, tR)}_{T(t)}}\\
        =& \frac{d}{dt} X_i (\US X_0, tR)  + A_i \lrp{\frac{d}{dt} X_i (\US X_0, tR)}M \th\lrp{B_i X_i(\US X_0, 0), C_i X_i (\US X_0,0))}\\
        &\qquad + A_i X_i(\US X_0, 0) M \frac{d}{dt} \th\lrp{ U \underbrace{B_i X_i(X_0, tU^\top R\US)}_{\tS(t)}, U \underbrace{C_i X_i (X_0, tU^\top R \US)}_{\tT(t)}}\\
        =& \US \frac{d}{dt} X_{i+1}(X_0, tU^\top R \US).
    \end{align*}
    In the second equality, we apply Lemma \ref{l:th_ddt_invariance} with $\Gamma = U$. We use the inductive hypothesis to verify that $\tS'(0) = S'(0)$ and $\tT'(0) = T'(0)$. This concludes the proof of \eqref{e:ddt_Xi_rotation_full}.

    \paragraph{$G$ and $\frac{d}{dt} G$ under random transformation of $X_0$}\\
    In this part of the proof, we establish two important identities about the evolution of $G$ under random rotation of its arguments:
    \begin{align*}
        \numberthis \label{e:Gi_rotation_full}
        & G_i(\US X_0,0) = G_i(X_0,0),\\
        \numberthis \label{e:ddt_Gi_rotation_full}
        & \at{\frac{d}{dt} G_i\lrp{\US X_0, tR }}{t=0} = \at{\frac{d}{dt} G_i\lrp{X_0, tU^{\top} R \US}}{t=0}.
    \end{align*}
    \eqref{e:Gi_rotation_full} is an immediate consequence of \eqref{e:Xi_rotation_full}:
    \begin{align*}
        G_i(\US X_0, 0) 
        :=& \prod_{\ell = 0}^i \lrp{I + r_\ell M \th\lrp{B_\ell X_\ell(\US X_0,0), C_\ell X_\ell(\US X_0,0)}}\\
        =& \prod_{\ell = 0}^i \lrp{I + r_\ell M \th\lrp{B_\ell X_\ell(X_0,0), C_\ell X_\ell(X_0,0)}}\\
        =& G_i(X_0, 0) ,
    \end{align*}
    where  the second equality uses \eqref{e:Xi_rotation_full}, \eqref{e:B_US_commute_full} and Assumption \ref{ass:th}.

    To verify \eqref{e:ddt_Gi_rotation_full}, we first verify the following recursive relationship:
    \begin{align*}
        & G_{i}\lrp{\US X_0,S} \\
        =& \lrp{I + r_{i} M \th\lrp{B_{i}(S) X_{i}(\US X_0,S), C_{i} X_{i}(\US X_0,S)}} G_{i-1}\lrp{\US X_0,S}\\
        \Rightarrow \qquad 
        & \at{\frac{d}{dt} G_{i}\lrp{\US X_0,tR}}{t=0}\\
        =& \lrp{\frac{d}{dt} \lrp{I + r_{i} M \th\lrp{B_{i}(tR) X_{i}(\US X_0,tR), C_{i} X_{i}(\US X_0,tR)}} } G_{i-1}\lrp{\US X_0,tR}\\
        & + \lrp{I + r_{i} M \th\lrp{B_{i} X_{i}(\US X_0,0), C_{i} X_{i}(\US X_0,0)}} \frac{d}{dt} G_{i-1}\lrp{\US X_0,tR}.
        \numberthis \label{e:ddt_Gi_recursion_full}
    \end{align*}
    We will analyze the two terms in \eqref{e:ddt_Gi_recursion_full} separately:
    \begin{align*}
        & \frac{d}{dt} \lrp{I + r_{i} M \th\lrp{B_{i}(tR) X_{i}(\US X_0,tR), C_{i} X_{i}(\US X_0,tR)}} \\
        =& I + r_{i} M \frac{d}{dt} \th\lrp{B_{i}(tR) X_{i}(\US X_0,tR), C_{i} X_{i}(\US X_0,tR)}.
    \end{align*}
    Let $S(t) := B_{i}(tR) X_{i}(\US X_0,tR)$ and $T(t) := C_{i} X_{i}(\US X_0,tR)$. Let $\tS(t) := U B_{i}(tU^\top R \US) X_{i}(X_0,t U^\top R \US)$ and $\tT(t) := U C_{i} X_{i}(X_0,tU^\top R\US)$. We verify that
    \begin{align*}
        S(0) =& B_{i} X_{i}(\US X_0,0) = U B_{i} \lrp{X_0, 0} = \tS(0)\\
        T(0) =& C_{i} X_{i}(\US X_0,0) = U C_{i} \lrp{X_0, 0} = \tT(0).
    \end{align*}
    By chain rule,
    \begin{align*}
        S'(0) 
        =& \lrp{\frac{d}{dt} B_{i} (tR)}X_{i}(\US X_0,0) + B_i \frac{d}{dt} X_{i}(\US X_0,tR)\\
        =& \lrp{\frac{d}{dt} B_{i} (tR)} \US X_{i}(X_0,0) + B_i \US \frac{d}{dt} X_{i}(X_0,tU^\top R \US)\\
        =& U \lrp{\frac{d}{dt} B_{i} (t U^\top R\US)}X_{i}(X_0,0) + U B_i \frac{d}{dt} X_{i}(X_0,tU^\top R \US)\\
        =& \tS'(0).
    \end{align*}
    The second equality follows from \eqref{e:Xi_rotation_full} and \eqref{e:ddt_Xi_rotation_full}. The third equality uses \eqref{e:B_US_commute_full}, as well as the fact that $U^\top \frac{d}{dt} B_i (tR) \US = \frac{d}{dt} B_i (tU^\top R \US)$; this is because for $i\neq j$, both sides are $0$, and for $i=j$, $\frac{d}{dt} B_j(tR) = R$.

    Similarly, we verify that
    \begin{align*}
        T'(0)
        =& C_i \frac{d}{dt} X_i (\US X_0, tR)\\
        =& U C_i \frac{d}{dt} X_i (X_0, t U^\top R \US)\\
        =& \tT'(0).
    \end{align*}
    Applying Lemma \ref{l:th_ddt_invariance} with $\Gamma = U$ gives
    \begin{align*}
        & \frac{d}{dt} \lrp{I + r_{i} M \th\lrp{B_{i}(tR) X_{i}(\US X_0,tR), C_{i} X_{i}(\US X_0,tR)}}\\
        =& \frac{d}{dt} \lrp{I + r_{i} M \th\lrp{B_{i}(tU^\top R \US) X_{i}(X_0,tU^\top R\US ), C_{i} X_{i}(X_0,tU^\top R\US)}}.
    \end{align*}
    Using \eqref{e:B_US_commute_full} and Assumption \ref{ass:th} and the inductive hypothesis, the second term of \eqref{e:ddt_Gi_recursion_full} satisfies
    \begin{align*}
        & \lrp{I + r_{i} M \th\lrp{B_{i} X_{i}(\US X_0,0), C_{i} X_{i}(\US X_0,0)}} \frac{d}{dt} G_{i-1}\lrp{\US X_0,tR}\\
        =& \lrp{I + r_{i} M \th\lrp{B_{i} X_{i}(X_0,0), C_{i} X_{i}(X_0,0)}} \frac{d}{dt} G_{i-1}\lrp{X_0,tU^\top R \US}.
    \end{align*}
    Combining the above identities for each term of \eqref{e:ddt_Gi_recursion_full}, we conclude that
    \begin{align*}
        \at{\frac{d}{dt} G_{i}\lrp{\US X_0,tR}}{t=0} = \at{\frac{d}{dt} G_{i}\lrp{X_0,tU^\top R\US}}{t=0}.
    \end{align*}
    This concludes the proof of \eqref{e:ddt_Gi_rotation_full}.

    \paragraph{Putting everything together:}\\
    We will now conclude the proof of \eqref{e:t:bjkjrdnads:3}. Plugging in \eqref{e:Gi_rotation_full} and \eqref{e:ddt_Gi_rotation_full} into \eqref{e:t:bjkjrdnads:9} gives
    \begin{align*}
        \at{\frac{d}{dt} f(V, B(tR), C)}{t=0}
        =& 2\E_{X_0, U} \lrb{\tr\lrp{\lrp{I-M}G_k(\US X_0, 0)^\top \Kmat \at{\frac{d}{dt} G_k(\US X_0, tR)}{t=0} \lrp{I-M}}}\\
        =& 2\E_{X_0, U} \lrb{\tr\lrp{\lrp{I-M}G_k(X_0, 0)^\top \Kmat \at{\frac{d}{dt} G_k(X_0, tU^\top R\US)}{t=0} \lrp{I-M}}}\\
        =& 2\E_{X_0} \lrb{\tr\lrp{\lrp{I-M}G_k(X_0, 0)^\top \Kmat \at{\frac{d}{dt} G_k(X_0, t\E_{U}\lrb{U^\top R\US})}{t=0} \lrp{I-M}}}\\
        =& 2\E_{X_0} \lrb{\tr\lrp{\lrp{I-M}G_k(X_0, 0)^\top \Kmat \at{\frac{d}{dt} G_k(X_0, t\tilde{R})}{t=0} \lrp{I-M}}}\\
        =& \at{\frac{d}{dt} f(V, B(t\tilde{R}), C)}{t=0}
    \end{align*}
    The third equality uses the fact that $\at{\frac{d}{dt} G_k(X_0, tS)}{t=0}$ is linear in $S$ for any $S$. The fourth equality is by \eqref{e:t:bjkjrdnads:2}. This concludes the proof of \eqref{e:t:bjkjrdnads:3}.

\end{proof}

\begin{proposition}
    \label{p:master_full_A}
    Let $\th$ satisfy Assumption \ref{ass:th}, let $\tx{i}$'s satisfy Assumption \ref{ass:x_distribution} with matrix $\Sigma$, and $\ty{i}$'s satisfy Assumption \ref{ass:y_distribution}. Let $(A,B,C)\in \R^{(k+1)\times d \times d\times 3}$ satisfy, for all $\ell=0\ldots k$, 
    \begin{align*}
        A_\ell = a_\ell I \qquad B_\ell = b_\ell \Sigma^{-1/2} \qquad C_\ell = c_\ell \Sigma^{-1/2},
        \numberthis \label{ass:full_ABC_shape_A}
    \end{align*}
    where $a_\ell, b_\ell, c_\ell \in \R$ are scalars. Let $V\in \R^{(k+1) \times (d+1) \times (d+1)}$ satisfy, for all $\ell=0\ldots k$, $V_\ell = \bmat{A_\ell & 0\\0 & r_\ell}$, where $r_\ell$ are arbitrary scalars. Let $j\in \lrbb{0\ldots k}$ be an arbitrary but fixed layer index. For $S\in \R^{d\times d}$, let $A_j(S) := A_j + S$, and let $A_\ell(S) := A_\ell$ for $\ell\neq j$. Let $A(S) := \lrbb{A_\ell(S)}_{\ell=0\ldots k}$. Let $V_\ell(S) := \lrbb{\bmat{A_\ell(S)&0\\0&r_\ell}}$ and $V(S) = \lrbb{V_\ell(S)}_{\ell=0\ldots k}$. Let $f(V,B,C)$ e as defined in \eqref{d:ICL_loss}. Let $R\in \R^{d\times d}$ be an arbitrary matrix. Let 
    \begin{align*}
        \tilde{r} := \frac{1}{d} \tr\lrp{\Sigma^{-1/2} R\Sigma^{1/2}} \qquad \qquad \tilde{R} := \tilde{r} I
        \numberthis \label{e:t:bjkjrdnads_A:2}
    \end{align*}
    Then
    \begin{align*}
        \numberthis \label{e:t:bjkjrdnads_A:3}
        \at{\frac{d}{dt} f(V(tR), B, C)}{t=0} \leq \at{\frac{d}{dt} f(V(t\tilde{R}), B, C)}{t=0}.
    \end{align*}
\end{proposition}

\begin{proof}
    \begin{proof}[Proof of Proposition \ref{p:master_full_A}]
    \paragraph{Loss Reformulation:} 
    Let us consider the reformulation of the in-context loss $f$ presented in Lemma \ref{d:ICL_loss}. Specifically, let $\overline{Z}_0$ be defined as
    \begin{align*}
    \overline{Z}_0 = \begin{bmatrix}
    \tx{1} & \tx{2} & \cdots & \tx{n} &\tx{n+1} \\
    \ty{1} & \ty{2} & \cdots &\ty{n}& \ty{n+1}
    \end{bmatrix} \in \R^{(d+1) \times (n+1)},
    \end{align*}
    Let $\overline{Z}_i$ denote the output of the $(i-1)^{th}$ layer of the linear transformer (as defined in \eqref{e:dynamics_Z}, initialized at $\overline{Z}_0$). For the rest of this proof, we will drop the bar, and simply denote $\overline{Z}_i$ by $Z_i$. Let $X_i\in \R^{d\times (n+1)}$ denote the first $d$ rows of $Z_i$ and let $Y_i\in \R^{1\times (n+1)}$ denote the $(d+1)^{th}$ row of $Z_k$. Under the assumption that $V_\ell = \bmat{A_\ell &0\\0&r_\ell}$, we verify that for all $i\in\lrbb{0\ldots k}$:
    \begin{align*}
    & {X}_{i+1} = X_i + A_i X_i M \th\lrp{B_i X_i, C_i X_i}\\
    & {Y}_{i+1} = {Y}_i + r_i {Y}_i M \th\lrp{B_i {X}_i, B_i {X}_i} = Y_0 \prod_{\ell = 0}^i \lrp{I + r_\ell M \th\lrp{B_\ell X_0, C_\ell X_0}}.
    \numberthis \label{e:dynamic_XY_full_proof_A}
    \end{align*}
    By Lemma \ref{d:ICL_loss}, the in-context loss defined in \eqref{d:ICL_loss} is equivalent to 
    \begin{align*}
        f(V,B,C) = \E_{Z_0} \lrb{\tr\lrp{\lrp{I-M}{Y}_{k+1}^\top {Y}_{k+1}\lrp{I-M}}}
    \end{align*}
    We will introduce one more piece of notation: Following \eqref{e:dynamic_XY_full_proof_A}, notice that for any layer $i$, $X_i$ is a function of $A, B, C, X_0$. Since for this part of the proof, only $A$ is variable (function of $S$), and $B,C$ are fixed, we define $X_i(X, S)$ to be "the result of evolving as \eqref{e:dynamic_XY_full_proof_A}, initialized at $X_0=X$, where $A_i$ is replaced by $A_i(S)$", i.e.
    \begin{align*}
        \numberthis \label{e:Xi(X,S)_dynamics_full_A}
        {X}_{i+1}(X,S) = X_i(X,S) + A_i(S) X_i(X,S) M \th\lrp{B_i X_i(X,S), C_i X_i(X,S)}
    \end{align*}
    Let us also define
    \begin{align*}
        G_i(X, S) := \prod_{\ell = 0}^i \lrp{I + r_\ell M \th\lrp{B_\ell X_\ell(X,S), C_\ell X_\ell(X,S)}},
    \end{align*}
    so that
    \begin{align*}
        f(V(S),B,C) =& \E_{Z_0} \lrb{\tr\lrp{\lrp{I-M}G_k(X, S)^\top {Y}_{0}^\top {Y}_{0} G_k(X, S) \lrp{I-M}}}\\
        =& \E_{X_0} \lrb{\tr\lrp{\lrp{I-M}G_k(X, S)^\top \Kmat G_k(X, S) \lrp{I-M}}},
    \end{align*}
    where recall that $\Kmat\in \R^{(n+1)\times (n+1)}$ and $\Kmat_{ij} = \K(\Sigma^{-1/2} \tx{i}, \Sigma^{-1/2}\tx{j})$ as defined in Assumption \ref{ass:y_distribution}. The second equality uses the assumption on distribution of $Y_0$ conditioned on $X_0$, as specified in Assumption \ref{ass:y_distribution}. Let $U$ denote a uniformly randomly sampled orthogonal matrix. Let $\US := \Sigma^{1/2} U \Sigma^{-1/2}$, so that $\US^{-1} = \Sigma^{1/2} U^\top \Sigma^{-1/2}$. We will repeatedly use the following identities:
    \begin{align*}
        & B_i \US = b_i \Sigma^{-1/2} \Sigma^{1/2} U \Sigma^{-1/2} = U B_i\\
        & C_i \US = c_i \Sigma^{-1/2} \Sigma^{1/2} U \Sigma^{-1/2} = U B_i
        \numberthis\label{e:B_US_commute_full_A}
    \end{align*}
    
    Using the fact that $X_0 \overset{d}{=} \US X_0$, we can verify
    \begin{align*}
        \at{\frac{d}{dt} f(V(tR),B,C)}{t=0}
        =& \at{\frac{d}{dt} \E_{X_0} \lrb{\tr\lrp{\lrp{I-M}G_k(X_0, tR)^\top \Kmat(X_0) G_k(X_0, tR) \lrp{I-M}}}}{t=0}\\
        =& 2\E_{X_0} \lrb{\tr\lrp{\lrp{I-M}G_k(X_0, 0)^\top \Kmat(X_0) \at{\frac{d}{dt} G_k(X_0, tR)}{t=0} \lrp{I-M}}}\\
        =& 2\E_{X_0, U} \lrb{\tr\lrp{\lrp{I-M}G_k(\US X_0, 0)^\top \Kmat(X_0) \at{\frac{d}{dt} G_k(\US X_0, tR)}{t=0} \lrp{I-M}}}.
        \numberthis \label{e:t:bjkjrdnads_A:9}
    \end{align*}
    The last equality uses the fact that $\Kmat(\US X_0) = \Kmat(X_0)$ by Assumption \ref{ass:y_distribution}.

    Henceforth, assume all $\frac{d}{dt}$ occurs at $t=0$, and we somtimes drop the explicit $\at{}{t=0}$ notation to save space.

    \paragraph{$X_i$ and $\frac{d}{dt} X_i$ under random transformation of $X_0$}\\
    In this part of the proof, we establish two important identities about the evolution of $X_i$ under random rotation of its arguments:
    \begin{align*}
        \numberthis \label{e:Xi_rotation_full_A}
        & X_i\lrp{\US X_0, 0} = \US X_i \lrp{X_0, 0},\\
        \numberthis \label{e:ddt_Xi_rotation_full_A}
        & \at{\frac{d}{dt} X_i\lrp{\US X_0, tR }}{t=0} = \US \at{\frac{d}{dt} X_i\lrp{X_0, t\US^{-1} R \US}}{t=0}.
    \end{align*}

    We first verify \eqref{e:Xi_rotation_full_A} by induction. For $i=0$, this identity holds by definition. Assume the identity holds for some $i$. Then following \eqref{e:Xi(X,S)_dynamics_full}, 
    \begin{align*}
        X_{i+1}(\US X_0,0) 
        =& X_i(\US X_0,0) + A_i(0) X_i(\US X_0,S) M \th\lrp{B_i X_i(\US X_0,0), C_i X_i(\US X_0,0)}\\
        =& \US X_i(X_0,0) + \US A_i(0) X_i(X_0,S) M \th\lrp{B_i \US  X_i(X_0,0), C_i \US X_i(X_0,0)}\\
        =& \US X_i(X_0,0) + \US A_i(0) X_i(X_0,S) M \th\lrp{B_i X_i(X_0,0), C_i X_i(X_0,0)}\\
        =& \US X_{i+1}(X_0,0).
    \end{align*}
    The second equality is by the inductive hypothesis, and the fact that $A_i = a_i I$. The third equality uses \eqref{e:B_US_commute_full} and Assumption \ref{ass:th}.

    Next, we verify \eqref{e:ddt_Xi_rotation_full}. By definition of $X_i(X_0,S)$, the case for $i\leq j$ is simple:
    \begin{align*}
        \at{\frac{d}{dt} X_i(\US X_0, tR)}{t=0} = 0 = \US \at{\frac{d}{dt} X_i(X_0, t\US^{-1} R \US)}{t=0}.
        \numberthis \label{e:t:bjkjrdnads_A:0}
    \end{align*}
    
    For $i=j+1$, it follows from \eqref{e:Xi(X,S)_dynamics_full} and chain rule that
    \begin{align*}
        & \frac{d}{dt} X_{j+1}(\US X_0, tR)\\
        =& \frac{d}{dt} X_j (\US X_0, tR)  + \lrp{\frac{d}{dt} A_j(tR)} X_j (\US X_0, 0) M \th\lrp{B_j X_j(\US X_0, 0), C_j X_j (\US X_0,0))}\\
        &\qquad + A_j(0) \lrp{\frac{d}{dt} X_j (\US X_0, tR)}M \th\lrp{B_j X_j(\US X_0, 0), C_j X_j (\US X_0,0))}\\
        &\qquad + A_j(0) X_j(\US X_0, 0) M \frac{d}{dt} \th\lrp{ B_j X_j(\US X_0, 0), C_j X_j (\US X_0, 0)}.
        \numberthis \label{e:t:bjkjrdnads_A:1}
    \end{align*}
    We can simplify each term on the RHS above separately:

    By \eqref{e:B_US_commute_full_A}, the first, third and fourth terms are $0$. By definition of $A_j$, $\frac{d}{dt} A_j (tR) = R$. Furthermore, using \eqref{e:t:bjkjrdnads_A:0}, \eqref{e:B_US_commute_full_A} and Assumption \ref{ass:th}, the second term simplifies to $\US \US^{-1} R \US X_j (X_0, 0) M \th\lrp{B_j X_j(X_0, 0), C_j X_j (X_0,0))}$. Therefore,
    \begin{align*}
        \frac{d}{dt} X_{j+1}(\US X_0, tR)
        =& \US \US^{-1} R \US X_j (X_0, 0) M \th\lrp{B_j X_j(X_0, 0), C_j X_j (X_0,0))}\\
        =& \US \frac{d}{dt} X_{j+1}(X_0, t\US^{-1} R \US ).
    \end{align*}
    We have thus verified \eqref{e:ddt_Xi_rotation_full_A} for $i\leq j+1$. For $i>j+1$, we will use proof by induction. Assume \eqref{e:ddt_Xi_rotation_full_A} holds for all $\ell \leq i$ for some $i\geq j+1$. Then for $i+1$,
    \begin{align*}
        & \frac{d}{dt} X_{i+1}(\US X_0, tR)\\
        =& \frac{d}{dt} X_i (\US X_0, tR) + \lrp{\frac{d}{dt}A_i(tR) } X_i (\US X_0, 0)M \th\lrp{B_i X_i(\US X_0, 0), C_i X_i (\US X_0,0))}\\
        &\qquad + A_i(0) \lrp{\frac{d}{dt} X_i (\US X_0, tR)}M \th\lrp{B_i X_i(\US X_0, 0), C_i X_i (\US X_0,0))}\\
        &\qquad + A_i(0) X_i(\US X_0, 0) M \frac{d}{dt} \th\lrp{ \underbrace{B_i X_i(\US X_0, tR)}_{S(t)}, \underbrace{C_i X_i (\US X_0, tR)}_{T(t)}}.
    \end{align*}
    Since $i\geq j+1$, we know that $\frac{d}{dt} A_i(tR)=0$, so the second term on RHS is 0. By the inductive hypothesis and \eqref{e:ddt_Xi_rotation_full_A}, and $A_i(0) =a_i I$, and Assumption \ref{ass:th}, the third RHS term can be simplified to be \\
    $\US A_i(0) \lrp{\frac{d}{dt} X_i\lrp{X_0, t \US^{-1} R\US}} M \th\lrp{B_i X_i(X_0, 0), C_i X_i (X_0,0))}$. Finally, to simplify the last RHS term, we apply Lemma \ref{l:th_ddt_invariance}. Let $S(t):= B_i X_i(\US X_0, tR)$ and $T(t) :=  C_i X_i (\US X_0, tR)$. Let $\tS(t) := U B_i X_i \lrp{X_0, t \US^{-1} R \US}$ and $\tT(t) := U C_i X_i \lrp{X_0, t \US^{-1} R \US}$. Let $\Gamma := U$. Then $\frac{d}{dt} \th\lrp{ B_i X_i(\US X_0, tR), C_i X_i (\US X_0, tR)} =$ \\$\frac{d}{dt} \th\lrp{ B_i X_i(X_0, t \US^{-1} R \US), C_i X_i (X_0, t\US^{-1} R \US)}$. Put together, we conclude that
    \begin{align*}
        \frac{d}{dt} X_{i+1}(\US X_0, tR) = \US \frac{d}{dt} X_{i+1}(X_0, t\US^{-1} R \US).
    \end{align*}
    We thus complete the proof of \eqref{e:ddt_Xi_rotation_full_A}.

    \paragraph{$G$ and $\frac{d}{dt} G$ under random transformation of $X_0$}\\
    In this part of the proof, we establish two important identities about the evolution of $G$ under random rotation of its arguments:
    \begin{align*}
        \numberthis \label{e:Gi_rotation_full_A}
        & G_i(\US X_0,0) = G_i(X_0,0),\\
        \numberthis \label{e:ddt_Gi_rotation_full_A}
        & \at{\frac{d}{dt} G_i\lrp{\US X_0, tR }}{t=0} = \at{\frac{d}{dt} G_i\lrp{X_0, t\US^{-1} R \US}}{t=0}.
    \end{align*}
    \eqref{e:Gi_rotation_full_A} is an immediate consequence of \eqref{e:Xi_rotation_full_A}:
    \begin{align*}
        G_i(\US X_0, 0) 
        :=& \prod_{\ell = 0}^i \lrp{I + r_\ell M \th\lrp{B_\ell X_\ell(\US X_0,0), C_\ell X_\ell(\US X_0,0)}}\\
        =& \prod_{\ell = 0}^i \lrp{I + r_\ell M \th\lrp{B_\ell X_\ell(X_0,0), C_\ell X_\ell(X_0,0)}}\\
        =& G_i(X_0, 0) ,
    \end{align*}
    where  the second equality uses \eqref{e:Xi_rotation_full_A}, \eqref{e:B_US_commute_full_A} and Assumption \ref{ass:th}.

    To verify \eqref{e:ddt_Gi_rotation_full_A}, we first verify the following recursive relationship:
    \begin{align*}
        & G_{i}\lrp{\US X_0,S} \\
        =& \lrp{I + r_{i} M \th\lrp{B_{i}(S) X_{i}(\US X_0,S), C_{i} X_{i}(\US X_0,S)}} G_{i-1}\lrp{\US X_0,S}\\
        \Rightarrow \qquad 
        & \at{\frac{d}{dt} G_{i}\lrp{\US X_0,tR}}{t=0}\\
        =& \lrp{\frac{d}{dt} \lrp{I + r_{i} M \th\lrp{B_{i} X_{i}(\US X_0,tR), C_{i} X_{i}(\US X_0,tR)}} } G_{i-1}\lrp{\US X_0,tR}\\
        & + \lrp{I + r_{i} M \th\lrp{B_{i} X_{i}(\US X_0,0), C_{i} X_{i}(\US X_0,0)}} \frac{d}{dt} G_{i-1}\lrp{\US X_0,tR}.
        \numberthis \label{e:ddt_Gi_recursion_full_A}
    \end{align*}
    We will analyze the two terms in \eqref{e:ddt_Gi_recursion_full_A} separately:
    \begin{align*}
        & \frac{d}{dt} \lrp{I + r_{i} M \th\lrp{B_{i} X_{i}(\US X_0,tR), C_{i} X_{i}(\US X_0,tR)}} \\
        =& I + r_{i} M \frac{d}{dt} \th\lrp{B_{i} X_{i}(\US X_0,tR), C_{i} X_{i}(\US X_0,tR)}.
    \end{align*}
    Let $S(t) := B_{i} X_{i}(\US X_0,tR)$ and $T(t) := C_{i} X_{i}(\US X_0,tR)$. Let $\tS(t) := U B_{i} X_{i}(X_0,t \US^{-1} R \US)$ and $\tT(t) := U C_{i} X_{i}(X_0,t\US^{-1} R\US)$. We verify that
    \begin{align*}
        S(0) =& B_{i} X_{i}(\US X_0,0) = U B_{i} \lrp{X_0, 0} = \tS(0)\\
        T(0) =& C_{i} X_{i}(\US X_0,0) = U C_{i} \lrp{X_0, 0} = \tT(0)\\
        S'(0) =& B_i \frac{d}{dt} X_{i}(\US X_0,tR)
        = U B_i \frac{d}{dt} X_{i}(X_0,t\US^{-1} R \US) = \tS'(0)\\
        T'(0) =& B_i \frac{d}{dt} X_{i}(\US X_0,tR) = U B_i \frac{d}{dt} X_{i}(X_0,t\US^{-1} R \US) = \tT'(0),
    \end{align*}
    where the last two equalities use \eqref{e:ddt_Xi_rotation_full_A} and \eqref{e:B_US_commute_full_A}. Applying Lemma \ref{l:th_ddt_invariance} with $\Gamma = U$ gives
    \begin{align*}
        & \frac{d}{dt} \lrp{I + r_{i} M \th\lrp{B_{i}(tR) X_{i}(\US X_0,tR), C_{i} X_{i}(\US X_0,tR)}}\\
        =& \frac{d}{dt} \lrp{I + r_{i} M \th\lrp{B_{i}X_{i}(X_0,t\US^{-1} R\US ), C_{i} X_{i}(X_0,t\US^{-1} R\US)}}.
    \end{align*}
    Using \eqref{e:B_US_commute_full_A} and Assumption \ref{ass:th} and the inductive hypothesis, the second term of \eqref{e:ddt_Gi_recursion_full_A} satisfies
    \begin{align*}
        & \lrp{I + r_{i} M \th\lrp{B_{i} X_{i}(\US X_0,0), C_{i} X_{i}(\US X_0,0)}} \frac{d}{dt} G_{i-1}\lrp{\US X_0,tR}\\
        =& \lrp{I + r_{i} M \th\lrp{B_{i} X_{i}(X_0,0), C_{i} X_{i}(X_0,0)}} \frac{d}{dt} G_{i-1}\lrp{X_0,t\US^{-1} R \US}.
    \end{align*}
    Combining the above identities for each term of \eqref{e:ddt_Gi_recursion_full_A}, we conclude that
    \begin{align*}
        \at{\frac{d}{dt} G_{i}\lrp{\US X_0,tR}}{t=0} = \at{\frac{d}{dt} G_{i}\lrp{X_0,t\US^{-1} R\US}}{t=0}.
    \end{align*}
    This concludes the proof of \eqref{e:ddt_Gi_rotation_full}.

    \paragraph{Putting everything together:}\\
    We will now conclude the proof of \eqref{e:t:bjkjrdnads_A:3}. Plugging in \eqref{e:Gi_rotation_full_A} and \eqref{e:ddt_Gi_rotation_full_A} into \eqref{e:t:bjkjrdnads_A:9} gives
    \begin{align*}
        \at{\frac{d}{dt} f(V(tR),B,C)}{t=0}
        =& 2\E_{X_0, U} \lrb{\tr\lrp{\lrp{I-M}G_k(\US X_0, 0)^\top \Kmat \at{\frac{d}{dt} G_k(\US X_0, tR)}{t=0} \lrp{I-M}}}\\
        =& 2\E_{X_0, U} \lrb{\tr\lrp{\lrp{I-M}G_k(X_0, 0)^\top \Kmat \at{\frac{d}{dt} G_k(X_0, tU^\top R\US)}{t=0} \lrp{I-M}}}\\
        =& 2\E_{X_0} \lrb{\tr\lrp{\lrp{I-M}G_k(X_0, 0)^\top \Kmat \at{\frac{d}{dt} G_k(X_0, t\E_{U}\lrb{U^\top R\US})}{t=0} \lrp{I-M}}}\\
        =& 2\E_{X_0} \lrb{\tr\lrp{\lrp{I-M}G_k(X_0, 0)^\top \Kmat \at{\frac{d}{dt} G_k(X_0, t\tilde{R})}{t=0} \lrp{I-M}}}\\
        =& \at{\frac{d}{dt} f(V(t\tilde{R}),B,C)}{t=0}
    \end{align*}
    The third equality uses the fact that $\at{\frac{d}{dt} G_k(X_0, tS)}{t=0}$ is linear in $S$ for any $S$. The fourth equality is by \eqref{e:t:bjkjrdnads_A:2}. This concludes the proof of \eqref{e:t:bjkjrdnads_A:3}.
\end{proof}

\end{proof}

\begin{lemma}
    \label{l:th_ddt_invariance}
    Let $S(t),T(t), \tilde{S}, \tilde{T}: \R \to \Gamma \in \R^{d\times d}$, denote arbitrary continuously  differentiable, matrix-valued, functions of time. Assume that $S(0) = \tilde{S}(0)$, $T(0) = \tilde{T}(0)$, $S'(0) =  \tilde{S}'(0)$ and $T'(0) = \tilde{T}'(0)$ (i.e. have the same time derivative at $t=0$). Let $\Gamma\in \R^{d\times d}$ be an arbitrary invertible matrix. Then for any $\th$ satisfying Assumption \ref{ass:th},
    \begin{align*}
        \at{\frac{d}{dt} \th\lrp{S(t), T(t)}}{t=0} = \at{\frac{d}{dt} \th\lrp{\Gamma^{\top} \tilde{S}(t), \Gamma^{-1} \tilde{T}(t)}}{t=0}
    \end{align*}
\end{lemma}
\begin{proof}
    Let $\jA^1$ and $\jA^2$ denote the Jacobians of $\th(A,B)[\cdot]$ with respect to $A$ and $B$ respectively. Then
    \begin{align*}
        & \at{\frac{d}{dt} \th\lrp{S(t), T(t)}}{t=0}\\
        =& \jA^1\lrp{S(0), T(0)}[S'(0)] + \jA^2\lrp{S(0), T(0)}[T'(0)]\\
        =& \jA^1\lrp{\tilde{S}(0), \tilde{T}(0)}[\tilde{S}'(0)] + \jA^2\lrp{\tilde{S}(0), \tilde{T}(0)}[\tilde{T}'(0)]\\
        =& \jA^1\lrp{\tilde{S}(0), \tilde{T}(0)}[\Gamma^\top \tilde{S}'(0)] + \jA^2\lrp{\tilde{S}(0), \tilde{T}(0)}[\Gamma^{-1}\tilde{T}'(0)]\\
        =& \at{\frac{d}{dt} \th\lrp{\Gamma^\top \tilde{S}(t), \Gamma^{-1} \tilde{T}(t)}}{t=0}.
    \end{align*}
    The third equality follows from Assumption \ref{ass:th}.
\end{proof}

\section{Experiment Details}
\label{ss:common_experiment_details}
The following are common to all experiments in this paper:

We train the Transformer to minimize the in-context loss given in \eqref{d:ICL_loss}. 

\textbf{Covariate Distribution}\\
The covariates $\tx{i} = \Sigma^{1/2} \xi^{(i)}$, where $\xi^{(i)}$ are sampled iid from the unit sphere. The dimension is $d=5$. The covariance matrix $\Sigma = U^T D U$, where $U$ is a uniformly random orthogonal matrix that changes across seeds, and $D$ is a fixed diagonal matrix with entries $(1,1,0.25,2.25,1)$. 

\textbf{Label Distribution}\\
Conditioned on $\tx{i}$'s, the labels $\ty{i}$ are jointly sampled from the $\K$ Gaussian Process (see Definition \ref{d:k_gaussian_process}). We consider three choices of kernels: $\K^{linear}(u,v)=\lin{u,v}$, ${\K}^{relu}(u,v)=\relu\lrp{\lin{u,v}}$, and $\K^{exp}(u,v)=\exp(\lin{u,v})$ (as defined \eqref{e:3_kernel_choices}).

\textbf{Transformer Architecture}\\
Unless otherwise stated, we train a three-layer linear Transformer (see \eqref{e:dynamics_Z}), where the matrices are initialized by i.i.d. Gaussian matrices. We consider three different choices of nonlinearity $\th$: linear, ReLU and softmax, defined in \eqref{e:4_kernel_choices} (see also Examples \ref{ex:linear}, \ref{ex:relu} and \ref{ex:softmax}). The Transformer is parameterized by $(r_\ell, A_\ell, B_\ell, C_\ell)_{\ell=0,1,2}$. (the value matrix $V_\ell$ is parameterized by the $A_\ell, r_\ell$, see Assumption \ref{ass:full_attention}). 

\textbf{Training Algorithm}\\
We train the Transformer using ADAM with gradient clipping. Each gradient step is computed from a minibatch of size 30000, and we resample the minibatch every 10 steps. All plots are averaged over $3$ runs with different $U$ (i.e. $\Sigma$) sampled each time, and different seeds for sampling training data.

\section{Background on RKHS}
\label{s:rkhs_basics}
Here we introduce a number of results from RKHS literature, which we will use later.

\begin{theorem}[\cite{wainwright2019high} Theorem 12.11, Kernel Reproducing Property]
    \label{t:kernel_reproducing_property}
    Given any positive semidefinite kernel function $\K$, defined on the cartesian product space $\X \times \X$, there is a unique Hilbert space $\HH$ in which the kernel satisfies the reproducing property: for any $x\in \X$, the function $\K(\cdot, x)$ belongs to $\HH$, and satisfies the relation
    \begin{align*}
        \lin{f, \K(\cdot, x)}_{\HH} = f(x)
    \end{align*}
    \label{t:rkhs_existence_wainwr}
\end{theorem}

\begin{theorem}[\cite{scholkopf2001generalized} Theorem 1, Nonparametric Representer Theorem]
    \label{t:representer_theorem}
    Given $\X$, positive semi-definite kernel $\K$ over $\X\times\X$, a set of $m$ training samples $(x_1,y_1)\ldots (x_m,y_m) \in \X \times \R$, a strictly monotonically increasing real-valued function $g$ on $[0,\infty]$, an arbitrary cost function $c : (\X \times \R^2)^m \to \R$. Then any $f\in \HH$ minimizing the regularized risk functional
    \begin{align*}
        c\lrp{(x_1,y_1,f(x_1))\ldots (x_m,y_m,f(x_m))} + g (\lrn{f}_{\HH})
    \end{align*}
    admits a representation of the form
    \begin{align*}
        f(\cdot) = \sum_{i=1}^m \alpha_i \K(\cdot, x_i)
    \end{align*}
\end{theorem}
Finally, we establish the explicit form of steepest descent in Hilbert space. The proof is quite standard (see e.g. Martin's 241B lecture 6), and we include it for completeness.
\begin{lemma}[Steepest Descent in Hilbert Space]
    \label{l:rkhs_descent}
    Given any $f\in \HH$, let $g^*$ denote the steepest descent direction of the weighted empirical least-squares loss wrt $\lrn{\cdot}_{\HH}$, i.e.
    \begin{align*}
        g^* := \argmin_{g\in \HH, \lrn{g}_{\HH}=1} \at{\frac{d}{dt} \sum_{i=1}^n \lrp{\ty{i} - (f+tg)(\tx{i})}^2}{t=0}.
    \end{align*}
    Then $g^*(\cdot) = c \sum_{i=1}^n\lrp{\ty{i} - f(\tx{i})} \K(\cdot, \tx{i})$
    for some scalar $c\in \R^+$ (we give explicit expression for $c$ in the proof). 
\end{lemma}
\begin{proof}
    Using the method of Lagrangian multipliers, there exists some $\lambda$ for which the above is equivalent to
    \begin{align*}
        g^* = & \argmin_{g\in \HH} \at{\frac{d}{dt} \sum_{i=1}^n \lrp{\ty{i} - (f+tg)(\tx{i})}^2 + \lambda \lrn{g}_{\HH}^2 }{t=0}\\
        =& \argmin_{g\in \HH} \sum_{i=1}^n -2\ty{i} g(\tx{i}) + \lambda \lrn{g}_{\HH}^2.
    \end{align*}
    The second line is by simple algebra.

    Applying Theorem \ref{t:representer_theorem} with $f:= g$ and $g(r):=\frac{\lambda}{2}r^2$, we know that $g^*(\cdot) = \sum_{i=1}^n \alpha_i \K(\cdot, \tx{i})$, for some $\alpha \in \R^{n}$. Using this together with Theorem \ref{t:kernel_reproducing_property}, we can write 
    \begin{align*}
        \lrn{g}_{\HH}^2
        =& \sum_{i,j=1}^n \alpha_i \alpha_j \lin{\K(\cdot,\tx{i}), \K(\cdot, \tx{j})}_{\HH}
        = \sum_{i,j=1}^n \alpha_i \alpha_j \K(\tx{i},\tx{j}).
    \end{align*}
    Thus $g^*(\cdot) = \sum_{i=1}^n \alpha^* \K(\cdot, \tx{i})$. Let $Y,F\in \R^n$ be defined such that $Y_i = \ty{i}$ and $F_i = f(\tx{i})$. Then
    \begin{align*}
        \alpha^* 
        =& \argmin_{\alpha \in \R^n} \sum_{i,j=1}^n  - 2 (\ty{i}-f(\tx{i})) \alpha_j \K(\tx{i},\tx{j}) + \lambda \alpha_i \alpha_j \K(\tx{i},\tx{j})\\
        =& \argmin_{\alpha \in \R^n} -2\lrp{\lrp{Y - F}}^\top \Kmat \alpha + \lambda \alpha^\top \Kmat \alpha.
    \end{align*}
    Taking $\nabla_\alpha = 0$, we get $\alpha^* \propto \lrp{Y-F}$. Recall that our original constraint is $\alpha^\top \Kmat \alpha = 1$, it follows that
    \begin{align*}
        \alpha^* = \frac{1}{(Y-F)^\top \Kmat (Y-F)} (Y-F).
    \end{align*}

    This implies that $g^*(\cdot) = \frac{1}{(Y-F)^\top \Kmat (Y-F)}\sum_{i=1}^n \lrp{\ty{i} - f(\tx{i})} \K\lrp{\cdot, \tx{i}}$, which concludes our proof.

    Note that the choice of $\alpha^*$ is in fact not unique when $rank(\Kmat)<n$. To see this, let $b\in \R^n$ be such that $b\neq 0, \Kmat b =0$. Then by the same argument above, 
    \begin{align*}
        \lrn{\sum_{i=1}^n b_i \K(\cdot, \tx{i})}_{\HH}^2 = b^\top \Kmat b = 0,
    \end{align*}
    so that $\sum_{i=1}^n \lrb{\alpha^* + b}_i \K(\cdot, \tx{i}) = \sum_{i=1}^n \lrb{\alpha^*}_i \K(\cdot, \tx{i})$, where equality is in the sense of $\lrn{\cdot}_{\HH}$.
\end{proof}

For intuition, let us consider the reduction of Lemma \ref{l:rkhs_descent} to the linear regression setting. Let $\theta(t) : \R \to \R^d$ denote the gradient flow of the parameter $\theta$ with respect to the empirical least-squares loss $\sum_{i=1}^n \lrp{\ty{i} - \lin{\theta(t), \tx{i}}}^2$ in Euclidean norm. It follows that 
\begin{align*}
    \frac{d}{dt} \theta(t) 
    =& 2 \sum_{i=1}^n (\ty{i} - \lin{\theta(t), \tx{i}}) \tx{i}\\
    =& 2 \sum_{i=1}^n X (Y - X^\top \theta(t)).
\end{align*}
Now let $f_t(x) := \lin{x, \theta(t)}$, and let $F$ denote the vector with $F_i = f_t(\tx{i}) = \lrb{X^\top \theta(t)}_i$. Then $\frac{d}{dt} f_t(x) = -2 \sum_{i=1}^n \lrb{Y-F}_i \K(\cdot, \tx{i})$, which is equal to the direction in Lemma \ref{l:rkhs_descent}.

\end{document}